%% file: main.tex
\documentclass[10pt,twocolumn,letterpaper]{article}
\usepackage[T1]{fontenc}
\usepackage[utf8]{inputenc}
\usepackage[letterpaper,top=0.85in,bottom=0.9in,left=0.75in,right=0.75in,columnsep=0.25in]{geometry}
\usepackage{amsmath,amsthm}
\usepackage{mathptmx}
\usepackage[varqu]{inconsolata}
\usepackage{textcomp}
\usepackage{microtype}
\usepackage{booktabs,graphicx,xspace}
\usepackage{array,tabularx}
\usepackage[table]{xcolor}
\usepackage[font=small,labelfont=bf,labelsep=period,skip=4pt]{caption}
\usepackage[numbers,sort&compress]{natbib}
\usepackage{titlesec}
\usepackage{balance}
\usepackage{needspace}
\PassOptionsToPackage{hyphens}{url}
\usepackage[colorlinks,linkcolor=linkblue,citecolor=linkblue,urlcolor=linkblue]{hyperref}
\definecolor{linkblue}{RGB}{20,50,140}
\hypersetup{pdftitle={Delta Tokenizer: Reference-Identical Tokenization of Session Deltas and Segment-Parallelism for Long Prompts on Multicore CPUs},pdfauthor={Zhenyu Zhang, Zhichao Cao}}
\definecolor{ctrlrow}{RGB}{255,236,204}
\titleformat*{\section}{\large\bfseries}
\titleformat*{\subsection}{\normalsize\bfseries}
\titleformat*{\subsubsection}{\normalsize\bfseries\itshape}
\titleformat*{\paragraph}{\normalsize\bfseries}
\titlespacing*{\section}{0pt}{2.2ex plus 0.6ex minus 0.3ex}{1ex plus 0.2ex}
\titlespacing*{\subsection}{0pt}{1.6ex plus 0.5ex minus 0.2ex}{0.6ex plus 0.2ex}
\titlespacing*{\subsubsection}{0pt}{1.2ex plus 0.4ex minus 0.2ex}{0.4ex plus 0.1ex}
\titlespacing*{\paragraph}{0pt}{0.9ex plus 0.3ex minus 0.1ex}{1em}
\newtheorem{proposition}{Proposition}[section]
\newcommand{\Description}[1]{}
\newcommand{\sys}{Delta Tokenizer\xspace}
\newcommand{\cat}{\mathbin{\Vert}}
\newcommand{\code}[1]{\texttt{#1}}
\graphicspath{{figures/}}
\title{\sys: Reference-Identical Tokenization of Session Deltas and Segment-Parallelism for Long Prompts on Multicore CPUs}
\author{Zhenyu Zhang\\ Arizona State University\\ \texttt{zzhan641@asu.edu} \and Zhichao Cao\\ Arizona State University\\ \texttt{Zhichao.Cao@asu.edu}}
\date{}
\begin{document}
\twocolumn[
\begin{@twocolumnfalse}
\maketitle
\begin{abstract}
An agent that drives an LLM sends the model the same conversation again and again, and each turn adds only a little text at the end. Inference engines already avoid redoing the model's work for the part that has not changed: the KV cache holds the state of the shared prefix, and only the new tokens pass through the model. However, tokenization has no such shortcut. The CPU front end of the serving engine converts every request back into token IDs from its first character, so this stateless step is paid in full on every turn. As a result, it lengthens the time to first token and consumes more CPU resources. Even worse, longer contexts and longer sessions widen the gap.

To address those issues, we present \sys, a CPU tokenization front end that reads each request as a change to the request before it (i.e., tokenizing the delta) in the same session. \sys retains the token sequence of the earlier turn, encodes again only the appended text and the short tail whose tokens the append could affect, and joins the two only at a position the tokenizer's splitting rules themselves cannot move, which keeps the result exact. A request with no state to reuse must still be encoded from scratch. For long requests of this kind, \sys cuts the text where practical tokenizers already keep neighboring pieces independent, encodes the pieces on several cores at once, and concatenates the outputs into the reference encoding, exactly. \sys leverages a shared CPU pool to schedule these chunks together with the short incremental updates, so an update runs between two chunks rather than behind an entire long job.

Across our experiments, the incremental path shortens the tokenization of a session turn by roughly 40$\times$ relative to the Hugging Face reference tokenizer, and on 16 cores the chunk-parallel path encodes a request of four million characters 8.2$\times$ faster than a serial encode and 1.6--2.1$\times$ faster than LoPT. Inside SGLang, \sys reduces the median time to first token, or TTFT, through the chat interface by 59\% at 100K tokens, which is the median context of the agents we traced, and brings the P99 TTFT of agent sessions down 9.5$\times$ under load and 23$\times$ when long requests tie up the front end without ever reaching the model. When such requests are admitted and the model serves them from its prefix cache, the P99 TTFT of agent turns stays about 7$\times$ lower, and about 20$\times$ lower on a server with one core or a smaller model. Every output we checked matches the reference, among them 41.8 billion documents encoded whole by the fast encoder inside \sys.
\end{abstract}
\vspace{1.2em}
\end{@twocolumnfalse}
]
\input{sections/01-introduction}

\input{sections/02-motivation}

\input{sections/03-overview}

\input{sections/04-incremental}

\input{sections/05-parallel}

\input{sections/06-runtime}

\input{sections/07-evaluation}

\input{sections/08-discussion-related}

\input{sections/09-conclusion}

\label{sec:body-end}
\clearpage
\balance
\bibliographystyle{abbrvnat}
\bibliography{refs}
\clearpage
\newgeometry{margin=1in,twocolumn=false}
\onecolumn
\appendix
\renewcommand{\thesection}{S\arabic{section}}
\setcounter{table}{0}\renewcommand{\thetable}{S\arabic{table}}\renewcommand{\theHtable}{S\arabic{table}}
\setcounter{figure}{0}\renewcommand{\thefigure}{S\arabic{figure}}\renewcommand{\theHfigure}{S\arabic{figure}}
\renewcommand{\topfraction}{0.9}
\renewcommand{\bottomfraction}{0.8}
\renewcommand{\textfraction}{0.07}
\renewcommand{\floatpagefraction}{0.8}
\setcounter{topnumber}{3}
\setcounter{bottomnumber}{2}
\setcounter{totalnumber}{5}
\begin{center}
{\LARGE\bfseries Supplementary Material}
\end{center}
\vspace{0.5em}
\input{sections/supplement-body}

\end{document}

%% file: sections/01-introduction.tex
\section{Introduction}
\label{sec:intro}
A coding agent works in rounds. It sends the model a request, receives an answer, executes a tool, and comes back with the outcome so the model can take the next step in the same session. Every one of these requests repeats the conversation so far, although what is new is often no more than a tool's output or a one-line instruction~\cite{react2023,sweagent2024}. On the model side, this repetition is already handled. Serving engines keep the state of the shared prefix in the KV cache and run the model only over the text that was added~\cite{vllm2023,sglang2024}. However, current tokenizers do not have such an equivalent mechanism. They treat each request as a fresh string and turn the full history back into token IDs on the CPU, turn after turn.

That repeated encoding has a real price, which also scales with the length of the context. Contexts are growing, and because the KV cache removes most of the model's own repeated work, the tokenizer's slice of each turn keeps widening. In a vLLM sweep of our own, the tokenizer's share of the time to first token (TTFT) rose from 10\% to 64\% as the prefix-cache hit rate approached 0.99. In our SGLang measurements, a Qwen3-8B chat turn whose 100K-token prefix is already cached reaches its first token after 145\,ms, and 93--99\,ms of that is the reference tokenizer running on a single CPU core (\S\ref{sec:workload}). Making the tokenizer faster shrinks this step but leaves its shape intact: the work still repeats, and it still grows with the history.

\begin{figure*}[t]
\centering
\includegraphics[width=\textwidth]{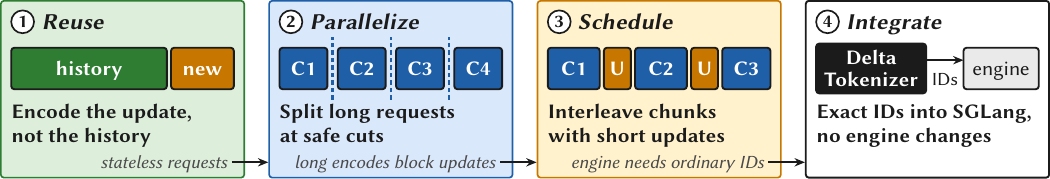}
\caption{The four steps of the design: reuse across turns, parallel full encodes, scheduling at chunk granularity, and an adapter that plugs into an unmodified serving engine. C stands for a chunk and U for an update.}
\label{fig:story}
\Description{Four connected design decisions: keep the previous prefix, run eligible full encodes in parallel, schedule chunks alongside updates, and hand exact IDs back to the serving engine.}
\end{figure*}

Our core argument is that \textbf{\emph{tokenization ought to be incremental, as the rest of the serving stack already is through its KV cache.}} A front end can hold on to the token sequence it produced for the previous turn and encode only the appended text plus the old suffix whose tokens the append might alter. What it must remember is modest, about one megabyte beyond the text the session stores anyway when it holds 100K tokens, four orders of magnitude and more below the KV cache the model keeps for those tokens. The model's cache cannot supply this reuse, since it is keyed by exactly the token IDs the tokenizer still has to compute.

However, incremental tokenization is more than concatenating the old IDs with the encoding of the new text, because adding characters to the last word can rewrite tokens that were already produced. A request ending in \code{pipe} that receives \code{line} ends up with \code{pipeline} as its last token (Figure~\ref{fig:append}). Before byte-pair encoding (BPE) runs, the tokenizer also normalizes characters, matches special-token strings, and splits the text into pieces, and every one of these stages can shift a token boundary close to the append~\cite{hftokenizers} (\S\ref{sec:challenges}). An incorrect token sequence changes what the model receives and what its KV cache can hit, even when it decodes back to text that reads fine (\S\ref{sec:incremental-results}). \textbf{\emph{Reuse therefore needs an exact, guaranteed rule}} for where the old tokens cease to be valid, while a fixed overlap will not achieve it.

To address those fundamental challenges, this paper presents \sys, a CPU tokenization front end built on three ideas: exact incremental reuse, chunk-parallel encoding for the full encodes that remain, and a scheduler that fits short updates between chunks. When a session continues, \sys encodes only the new text along with a short stretch taken from the tail of the preceding request, finds where the new token sequence falls into agreement with the old one, and splices them at a boundary that the splitting rules of the tokenizer itself cannot move, and nowhere else. Should the window hold no such boundary, \sys enlarges it, and as a final fallback it encodes the append together with everything after the last boundary it has sealed.

Reuse does nothing for the full encodes that are still required. A session that has just started, or one whose history was edited, has to be encoded from the beginning. A long request of this kind, carrying no session state, can keep a tokenization worker busy for hundreds of milliseconds. \sys divides such a request into text chunks and encodes them in parallel across CPU cores. For BPE in general, this is unsafe. Where a token ends may hinge on text far away, so no cut position can be trusted without reading the entire input. Existing parallel tokenizers therefore cut at arbitrary positions and mend the seams using overlap plus retries, and LoPT, the most recent, is not exact in its published form~\cite{lopt2026}. What makes the problem tractable is that tokenizers in practice have more structure. A pre-tokenizer first breaks the text into short pieces, words and numbers for instance, and BPE never merges across two of them. \sys cuts only at positions where this split cannot move, which the characters on either side of the cut reveal, and it leaves special-token strings and Unicode normalization undisturbed across the cut. Each chunk is then encoded by an existing exact encoder on a core of its own, and concatenating the outputs yields the whole request's encoding, with nothing to reconcile at the seams and nothing to retry. Chunks have a target size, 64K characters unless configured otherwise (\S\ref{sec:runtime}).

An incremental job is normally far smaller than one chunk. We propose tail-aware scheduling for those two types of jobs, such that an incremental job does not need to wait while long requests are being encoded chunk by chunk. \sys runs the tokenization cores as a worker pool, assigns every session a home worker that holds its reuse state, feeds chunks into a shared queue, and lets an incremental job slip in between chunks. A short update thus never queues behind an entire long stateless request. An adapter connects \sys to a serving engine like SGLang with no change to the engine itself. Figure~\ref{fig:story} shows the four steps in sequence.

We evaluate \sys with six byte-level BPE tokenizers, among them the tokenizers of Qwen3~\cite{qwen3}, gpt-oss~\cite{gptoss}, Llama-3.1~\cite{llama3}, and DeepSeek-V3~\cite{deepseekv3}, across three CPU platforms, and with SGLang~\cite{sglang2024} serving Qwen3-8B and gpt-oss-120b, on web text, code, and multilingual text as well as public coding-agent transcripts~\cite{swebenchpro_traces} and agent trajectories~\cite{swesmith2025} (\S\ref{sec:evaluation}). Measured against the Hugging Face reference tokenizer~\cite{hftokenizers}, which SGLang uses by default, incremental reuse brings a session turn from 93--99\,ms down to about 2\,ms, and the gain comes from the reuse rather than from the encoder. For chat requests with 100K-token contexts, the fast encoder in \sys, used without session state, is no quicker than the reference, whereas \sys cuts Qwen3-8B's median TTFT by 59\%; fastokens~\cite{fastokens}, a backend that reuses earlier work as well, is as fast, but for 78 of the 279 raw-text requests its token IDs are wrong. On 16 cores, chunk-parallel encoding of a request of four million characters runs 8.2$\times$ faster than a serial encode and 1.6--2.1$\times$ faster than LoPT~\cite{lopt2026}. With the same FIFO queue and the same workers, scheduling chunks rather than whole requests takes the P99 latency for million-character requests down from 374\,ms to 103\,ms. Inside SGLang, agent sessions see their P99 TTFT fall from 995\,ms to 105\,ms when they produce 24 turns per second, and when long stateless requests that exceed the model's context tie up the front end, it goes from 1,441\,ms to 63\,ms, 23$\times$ lower. When documents of about 100K tokens are admitted instead and the model serves them from its prefix cache, the tail is still set by a stateless front end. The P99 TTFT of agent turns then drops from about half a second to under 70\,ms, roughly 7$\times$, and by about 20$\times$ once the server is limited to one core or runs a smaller model. Only when every long document is new does GPU prefill take over the tail, and the three front ends then land within 5\% of one another.

The contributions of this paper are the following.
\begin{itemize}
\item \textbf{Exact incremental tokenization (\S\ref{sec:incremental}).} A front end that keeps the previous turn's tokens and repairs only the tail, under a check, so the encoding work of a turn follows the size of the append instead of the size of the context. The check, rather than any fixed overlap, is what makes the reuse exact.
\item \textbf{Chunk-parallel full encoding at safe cuts (\S\ref{sec:parallel}).} Practical tokenizers expose cut points that the surrounding characters identify. We use them to divide a long stateless request into independently encoded chunks whose concatenation is the reference output exactly, without the overlap, the seam matching, and the retries that earlier parallel tokenizers need.
\item \textbf{Tail-aware scheduling of updates and chunks (\S\ref{sec:runtime}).} Chunk boundaries act as scheduling points inside a CPU pool that pins each session to a home worker and draws chunks from a shared queue, so a short incremental job can run in between one long request's chunks.
\item \textbf{Integration and evaluation (\S\ref{sec:evaluation}).} We connect \sys to SGLang without modifying the engine and measure each mechanism on its own and in combination, including long documents that the model serves from its cache. Every comparison against the reference agrees, among them 41.8 billion documents encoded whole by \sys's fast encoder.
\end{itemize}

%% file: sections/02-motivation.tex
\section{Background and Motivation}
\label{sec:motivation}

\subsection{LLM Inference and Tokenization}
\label{sec:background}
An LLM serving engine such as vLLM or SGLang handles a request in three stages~\cite{vllm2023,sglang2024}. The front end renders the conversation into one text string, for chat requests through the model's chat template, which wraps each message in role markers, and tokenizes it into vocabulary IDs on the CPU. The engine then runs prefill over the prompt tokens on the GPU, which fills the key-value (KV) cache, and decodes output tokens one at a time. The time to first token (TTFT), from a request's arrival to its first output token, includes the front end, queueing, and prefill. Engines keep the KV cache across the requests of a session and reuse it for the shared token prefix, as in vLLM's prefix caching~\cite{vllm2023} and SGLang's RadixAttention~\cite{sglang2024}, so a continuation pays prefill only for its appended tokens.

Tokenization maps text to IDs deterministically, and a serving system must return exactly the IDs the model was trained and evaluated with. The LLMs we study use byte-level BPE~\cite{gpt2,llama3,qwen3}, which runs as a pipeline of stages. Special-token strings, such as the role markers of the chat template, are recognized first and kept whole. A normalizer may rewrite characters, for example into Unicode NFC form~\cite{uax15}. A pre-tokenizer then splits the text into short pieces with regular expressions, separating words, digit runs, punctuation, and whitespace, and BPE merges the bytes of each piece according to a learned merge order. The IDs of the pieces concatenate to the output. The reference implementation for these models is the Hugging Face \code{tokenizers} library~\cite{hftokenizers}, which encodes one request on one core.

\subsection[The Critical Need for Incremental and Parallel Tokenization]{The Critical Need for Incremental and\\ Parallel Tokenization}%
\label{sec:case}

\subsubsection[LLM agent workloads: long contexts, but small appends]{LLM agent workloads: long contexts,\\ but small appends}%
\label{sec:workload}
In 153,951 calls recorded from Claude Code and Codex sessions on developer machines, the prompt cache serves 88--94\% of the input tokens depending on the agent, and the median call appends about 1.5K characters to a context of 86--123K tokens (Figure~\ref{fig:motivation}a). The front end still re-encodes the whole history on every call, and in a vLLM sweep, as the hit rate approaches 0.99, tokenization grows from 10\% to 64\% of TTFT (Figure~\ref{fig:motivation}b). The reference encode itself grows with the context, from 7\,ms at 8K tokens to 93--99\,ms at 100K (Figure~\ref{fig:motivation}c), so the cost sits where the contexts are long. At these single-core encode times, the 49\% of recorded calls with contexts of 100K tokens or more account for 78\% of the reference tokenizer's CPU time, the 37\% between 32K and 100K tokens for 20\%, and the 14\% under 32K tokens for only 2\%. Reuse therefore matters for the long, cache-friendly sessions that carry most of this cost and changes little for short ones. Consecutive calls of a session arrive 12\,s apart at the median, so at the median call a session keeps a reference tokenizer core 0.7\% busy. At that rate, one core saturates at about 140 active sessions, against about 6,400 with the incremental path. Even below saturation, with one event loop a single 94\,ms encode delays every request queued behind it.

\begin{figure*}[t]
\centering
\includegraphics[width=\textwidth]{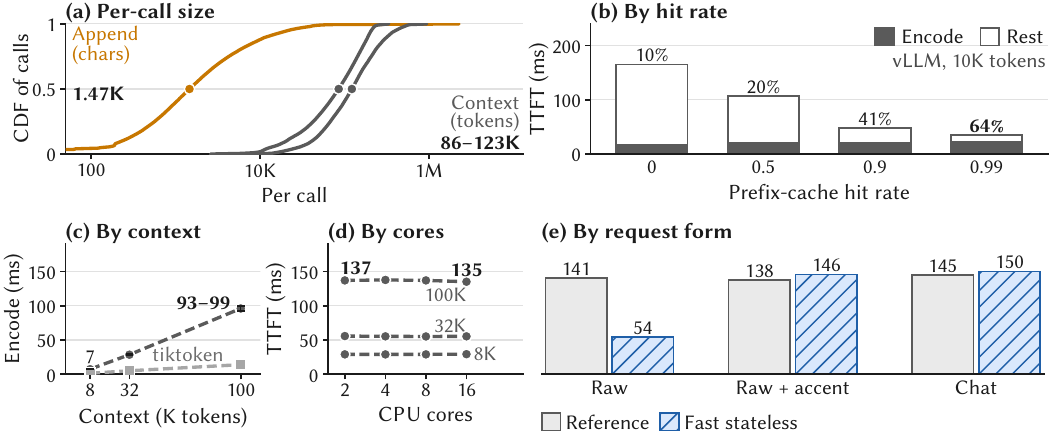}
\caption{Agent requests and what a stateless front end pays for them. (a)~Append (characters, both agents pooled) and context (tokens, Claude Code and Codex) per call in the recorded traces, dots at the medians. (b)~Mean TTFT in a vLLM sweep (Qwen3-1.7B, 10K-token prompts), reference encode timed separately. (c)~Reference encode of Qwen3-8B's tokenizer on one core, clean and accented transcripts, two runs each, and tiktoken on gpt-oss-120b's tokenizer, where it is native, on the same texts. (d,~e)~Median TTFT of Qwen3-8B in SGLang with a cached prefix, raw text by cores, and at 100K tokens by request form against the fast encoder without session state.}
\label{fig:motivation}
\Description{Five small panels in two rows, (a) and (b) above, (c), (d) and (e) below. (a) Two cumulative distributions on a log axis: appends have a median of 1.47 thousand characters, while contexts have medians of 86 and 123 thousand tokens. (b) Stacked bars of time to first token at prefix-cache hit rates 0, 0.5, 0.9 and 0.99; the tokenization part stays near 20 ms while the rest shrinks, so its share rises from 10 to 64 percent. (c) The reference encode takes 7, 28 to 29, and 93 to 99 ms at 8K, 32K and 100K tokens, and tiktoken on gpt-oss-120b's tokenizer about 1.4, 5 and 14 ms. (d) Time to first token stays at 137 and 135 ms at 100K tokens, near 55 ms at 32K and near 29 ms at 8K from 2 to 16 cores. (e) Bars at 100K tokens: the fast stateless encoder cuts the reference's 141 ms to 54 ms on raw text, but takes 146 against 138 ms with one accent and 150 against 145 ms on chat requests.}
\end{figure*}

Adding cores to a stateless serving process does not change this cost, because each encode runs on one core. In SGLang, moving from two to sixteen cores leaves the TTFT of a Qwen3 turn with a cached 100K-token prefix at 137 and 135\,ms (Figure~\ref{fig:motivation}d). Four tokenizer workers, like SGLang's dynamic-batching tokenizer, overlap requests rather than shortening one, and improve the Hugging Face tokenizer's median TTFT by only 6--10\% under a burst and not at all under paced load.

\subsubsection{A faster tokenizer is necessary but not sufficient}
\label{sec:fast-floor}
A faster serial encoder is the first step. \sys uses a corrected build of Gigatoken~\cite{gigatoken2026} that encodes 54--132\,MB/s on one core, an order of magnitude faster than the 3--4\,MB/s of the Hugging Face reference~\cite{hftokenizers}. Where it is native, on gpt-oss's tokenizer, tiktoken~\cite{tiktoken} encodes 100K tokens in 13--15\,ms (Figure~\ref{fig:motivation}c), against 85--86\,ms for the Hugging Face port, and as a whole encode its cost still grows with the context. Gigatoken needs the correction because the released version does not always return the reference's IDs. Over the 3.8 billion web documents of \S\ref{sec:correctness}, each encoded under eleven tokenizers, version 0.9.0 differs from the reference in 558 of the 41.8 billion encodes. Every difference we archived, 554 of the 558, traces to a rare character that its normalizer or character-class tables treat differently, since both follow a newer Unicode version than the reference. Our build uses the reference's Unicode data for both and agrees with the reference on all 41.8 billion encodes. \sys's encoder stays exact by guarding its fast path. When a call contains text its fast path does not handle, such as a special-token string or a character the normalizer would change, it encodes the whole call on an exact slow path instead, so a small exceptional region sets the cost of a large input. Even on English web text, four of the twenty documents of a million characters take this slow path and run about ten times slower than the rest (\S\ref{sec:parallel-results}). On mixed-language text, the fast encoder with Qwen3's tokenizer falls to 3.9\,MB/s, the speed of the reference.

Agent traffic hits these guards on every chat request, because the chat template renders role markers as special-token strings. Rendered with Qwen3's chat template, all 610 public Codex transcripts on SWE-Bench Pro tasks~\cite{swebenchpro_traces,swebenchpro2025} fell back to the encoder's exact slow path and took 94\,ms at the median on one core, against 124\,ms for the reference tokenizer. A stateless front end that encodes the whole request pays this price on every turn, because the offending text stays in the history (Figure~\ref{fig:motivation}e). An incremental front end pays it in full only at the session's first encode. Later, a character the normalizer would change sends just the unsealed tail and the append to the reference (\S\ref{sec:splice-correct}).

\subsubsection{The remaining full encodes set the tail}
After reuse, most continuations need a small repair, but initializations (new sessions and changed histories) and stateless requests still need a full encode, and one long encode can delay many updates that need only a few milliseconds of CPU time. In the traces, calls without a reusable prefix are only 1.0--3.6\% of the total, but they carry complete contexts, often two orders of magnitude larger than an append.

\subsubsection{Challenges of exact reuse and exact splitting}
\label{sec:challenges}
For both reuse and splitting, the difficulty is to return exactly the IDs of a full encode for every request while doing a small fraction of its work, and tokenization resists this because its output is not local.

\emph{An append can change tokens that were already emitted.} If a request ends with \code{pipe} and the next one appends \code{line}, the tokenizer sees \code{pipeline}, whose tokens need not equal the old tokens followed by an encoding of \code{line} (Figure~\ref{fig:append}). The new text can also regroup the whitespace or digits before it, combine with the preceding character under Unicode normalization, or complete a special-token string. How far back the change reaches depends on the tokenizer and on the text, so no fixed overlap is safe. Reuse therefore needs a check, cheap enough to run on every update, of where the old tokens stop changing. fastokens~\cite{fastokens} reuses earlier tokens without a check at the seam and returns different token IDs for 78 of 279 replayed requests (\S\ref{sec:incremental-results}).

\emph{A long input has no cut that is safe in general.} In byte-pair encoding, where a token ends can depend on text far from the cut, by a distance that depends on the tokenizer, so chunks encoded alone need not concatenate to the encoding of the whole. Earlier parallel tokenizers encode overlapping chunks and reconcile the seams afterwards, at the cost of sequential matching and retries. LoPT~\cite{lopt2026}, the most recent of them, matches overlapping tokens by character position alone and, as published, can accept a seam at which the two chunks disagree. The supplement (\S S5) shows a document from LoPT's own benchmark on which the published rule returns wrong token IDs although the seam meets the condition of its correctness theorem, and counts how often this happens. An exact splitter has to find cuts that every stage of the tokenizer respects, without running the tokenizer over the whole input first.

\emph{A running encode cannot be interrupted.} A scheduler can act only between calls into the encoder, so splitting must also yield units of work small enough to schedule.

%% file: sections/03-overview.tex
\section{\texorpdfstring{\sys}{Delta Tokenizer} Design}
\label{sec:design}
\subsection{Overview}
\label{sec:overview}
\sys keeps tokenization state per session and handles three kinds of request. A \emph{continuation} carries a session key, and the text stored for that session is a prefix of the new text; \sys repairs the tail of the stored tokens (\S\ref{sec:incremental}). An \emph{initialization} carries a session key but has no usable state, because the session is new, its history changed, its state was evicted, or the tokenizer changed; \sys encodes it whole on the session's home worker and stores the result. A \emph{stateless request} carries no session; \sys divides it at safe cuts and encodes the chunks on a shared pool (\S\ref{sec:parallel}).

\begin{figure*}[t]
\centering
\includegraphics[width=\textwidth]{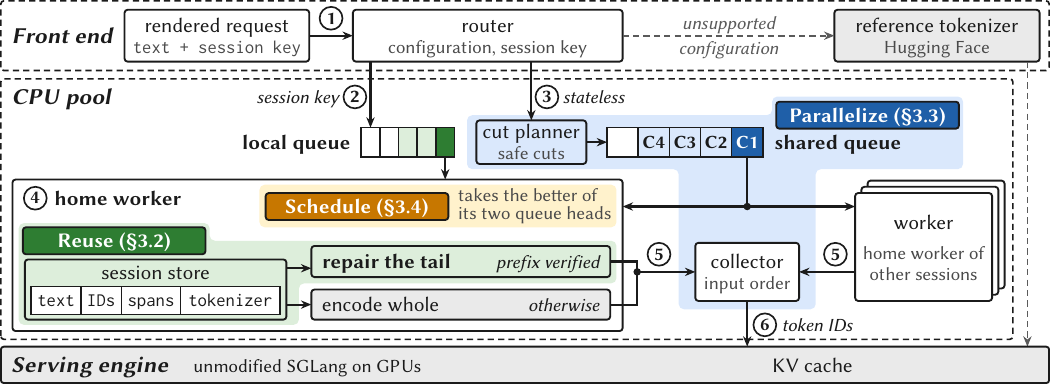}
\caption{Runtime architecture of \sys, with the request path numbered as in \S\ref{sec:overview}. Unsupported configurations take the reference path. The three labels mark where the designs of Figure~\ref{fig:story} live and the sections that describe them.}
\label{fig:architecture}
\Description{A router sends keyed requests to the home worker's local queue. The home worker holds the session store and either repairs the tail after verifying the prefix or encodes the request whole. Stateless requests pass through a cut planner into a shared queue served by every worker. A collector assembles the results in input order for the model engine, and unsupported configurations go to the reference tokenizer.}
\end{figure*}

Figure~\ref{fig:architecture} shows how a request moves through \sys, which sits between request rendering and the serving engine and uses only CPUs. (1)~The router receives the rendered text, the tokenizer configuration, and a session key when the request carries one. A request without an explicit key, such as a chat completion, gets a key derived from its content (\S\ref{sec:integration}), unless it is long enough to be chunked, in which case it is treated as a stateless request. (2)~A keyed request enters the local queue of its home worker, which holds the session store. After a successful repair, the worker updates the stored text and token sequence. (3)~The cut planner divides a stateless request at safe cuts into chunks on a shared queue that every worker serves, and a short stateless request enters that queue as one task. (4)~An idle worker takes the better of its two queue heads (\S\ref{sec:runtime}) and never interrupts an encode it has started. (5)~Workers return integer buffers to the collector, which assembles them in input order. (6)~The adapter returns the ordinary token-ID sequence to the unmodified serving engine. A request whose tokenizer configuration or API options are not admitted bypasses the pool and goes to the reference tokenizer, and the fast encoder's own guards can also send a request straight to the reference.

All paths share one contract. For text $x$ and a frozen tokenizer configuration $c$, the returned IDs must satisfy
\begin{equation}
 \operatorname{IDs}_{\sys}(x,c)=\operatorname{IDs}_{\mathrm{ref}}(x,c).
 \label{eq:contract}
\end{equation}
Sequence-level processing, such as adding a beginning-of-sequence token, is applied once rather than once per chunk. A failed reuse or cut check costs extra work, not exactness.

%% file: sections/04-incremental.tex
\subsection{Exact Incremental Tokenization}
\label{sec:incremental}
\subsubsection[Windowed re-encoding and the stable-boundary splice]{Windowed re-encoding\\ and the stable-boundary splice}%
\label{sec:repair}
\sys takes a suffix of the old text, 512 characters by default, appends the new text, and re-encodes this window (Figure~\ref{fig:repair}a). A reusable prefix must end where the old and new runs of the tokenizer agree on how the text divides into pieces (Figure~\ref{fig:append}), which equal token IDs alone do not show. \sys therefore looks for a run of tokens on which the cached result and the new encoding agree. Tokens are compared as records of ID, source position, and order among tokens with the same position. Window positions are first shifted into request coordinates, so repeated text elsewhere cannot produce a false match. A run counts only while the token indices stay consecutive in both sequences, so equal tokens at scattered positions never justify dropping the tokens between them. The first tokens of the window, which its artificial left edge may have changed, are ignored.

\begin{figure}[t]
\centering
\includegraphics[width=\columnwidth]{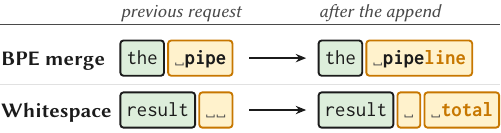}
\caption{An append can change tokens that were already emitted. Amber marks the tokens that change, and \textvisiblespace{} marks a space. Both examples hold for Qwen3 and Llama-3.1.}
\label{fig:append}
\Description{Two examples. Appending "line" to "the pipe" turns the token " pipe" into " pipeline". Appending "total" to "result" followed by two spaces splits the two-space token into a single space and the token " total".}
\end{figure}

A matching run is accepted only if it covers more characters than the longest token the vocabulary can produce and contains a character transition that the tokenizer's splitting rules always treat as a piece boundary, such as a letter followed by a digit. The pieces on each side of such a transition do not depend on the text on the other side, which is what makes the splice exact (\S\ref{sec:splice-correct}); the length condition is an extra margin that the argument does not need.

When the check succeeds, let $i$ and $j$ be the positions right after the matching run in the cached sequence $T_{\mathrm{old}}$ and in the window sequence $T_{\mathrm{win}}$. The new sequence is
\begin{equation}
 T_{\mathrm{new}}=T_{\mathrm{old}}[0{:}i]\cat T_{\mathrm{win}}[j{:}],
 \label{eq:splice}
\end{equation}
so the matching run is kept once, from the cache, and the start of the window, which may have been grouped differently, never reaches the result (Figure~\ref{fig:repair}b). When the check fails, \sys doubles the window, up to five times, and then lets the reference encode the text that the session has not sealed, as described below. The initial window only sets how much work to try first, and the check alone decides whether the result is exact. In our replay of 56,052 recorded appends with this rule, all but 3 splice on the first window, those 3 after one widening, and none needs the reference.

\begin{figure*}[t]
\centering
\includegraphics[width=\textwidth]{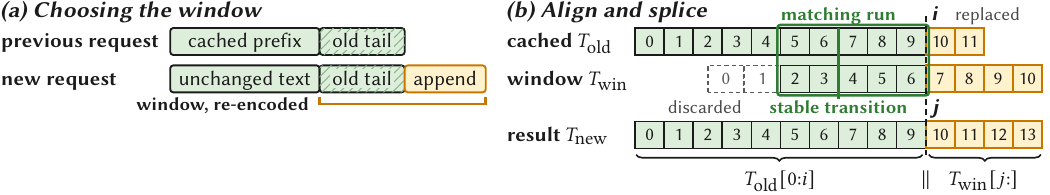}
\caption{Incremental repair. (a)~Only the old tail and the append are re-encoded. (b)~A matching run that contains a stable transition lets the result keep the cached IDs before index $i$ and take the window IDs from index $j$ on. A failed check widens the window or hands the unsealed tail and the append to the reference.}
\label{fig:repair}
\Description{An old prefix and tail become a new request with an append. Only an overlapping tail window is encoded. A checked matching run selects a splice between retained cached IDs and a newly encoded suffix.}
\end{figure*}

\subsubsection{Why the splice is exact}
\label{sec:splice-correct}
The argument covers the complete tokenizer, not only the BPE merges, and must rule out two changes. The append can change the end of the cached encoding, and starting the window inside a piece can change the start of the new encoding. The stable transition in the matching run separates the two regions. Before it, the cached pieces are those of a full encode, because no piece crosses the transition and the append lies beyond it. After it, the window's pieces are those of a full encode, because the window's left edge lies before it. The tokens of the matching run belong to both encodings, so keeping their cached copy changes nothing, and since BPE is deterministic within each piece, the result is exactly the full encode. Written with the notation of Equation~\ref{eq:splice}, the argument is the following statement, which Section~S3 of the supplement proves with its boundary conditions.

\Needspace{9\baselineskip}%
\begin{proposition}[Splice through an agreed region]
\label{prop:splice}
Let the cached encoding $T_{\mathrm{old}}$ of a text $A$ be exact, let $B$ be an append, and let $T_{\mathrm{win}}$ encode a window $(A\cat B)[s{:}]$. Suppose a matching run contains a boundary $p$ such that the pieces of $A$ before $p$ do not change when $B$ is appended and the pieces of the window after $p$ are those of $A\cat B$. If $i$ and $j$ are the positions just after the last entry of the run in $T_{\mathrm{old}}$ and $T_{\mathrm{win}}$, then $T_{\mathrm{old}}[0{:}i]\cat T_{\mathrm{win}}[j{:}]$ is the reference encoding of $A\cat B$.
\end{proposition}

The stable transition is what supplies $p$. Starting from an exact first encode, the proposition applies again after every accepted update, so exactness carries through a whole session.

A chat turn meets two checks on this path, with different costs. The first concerns special-token strings, such as the role markers that the chat template inserts. The window runs through the complete pipeline, which matches special-token strings before the fast path, so a role marker inside the window is recognized as a unit, as in the reference pipeline, and the repair goes ahead (\S\ref{sec:incremental-results}). The second is the normalizer check that a full encode also applies. If the normalizer would change any character of the window, \sys does not splice. The session keeps its IDs fixed up to a sealed boundary, a stable transition a short distance before the end of its text, and the reference encodes only the retained tail after it together with the append (supplement, \S S3), so the cost follows the append rather than the context. In agent traffic the two checks fire at very different rates. In a turn-by-turn replay of the 610 chat transcripts of \S\ref{sec:fast-floor}, every append carries role markers, and 19,609 of the 19,620 appends are repaired on the fast path. The other 11 bring in a character that the normalizer changes. For them the reference encodes fewer than 5K characters beyond the append, and the call takes 4--98\,ms as the append grows from 2.9K to 106K characters, whatever the length of the history. The 610 first turns have no session state yet and are encoded whole, and since each contains special-token strings, the fast encoder's guard sends the whole call to the reference (\S\ref{sec:fast-floor}). The argument also relies on the tokenizer's configuration, which \sys checks before admitting a tokenizer (\S\ref{sec:cuts}).

\subsubsection{Session state, verification, and misses}
The session key only locates a candidate record, which is reused only after its text prefix and tokenizer identity are verified. If a template changes an earlier part of the conversation, the request is not an append even when the application considers it the next turn, and \sys encodes it in full, as it does for evicted state, a lost home worker, or a changed tokenizer, so stale IDs are never reused.

A session record holds the text, a four-byte ID per token, source positions for the tail that a repair may touch, and a few prefix digests, on the session's home worker, with a revision that the next request must extend. Beyond the text, which a session holds in any case, this adds about 4 bytes per token of history, plus prefix blocks shared through the store (Table~\ref{tab:memory}). The store keeps up to 128\,MiB of session text resident, evicting the least recently used session first (supplement, \S S6.3). The store is independent of the model's KV cache, so tokenizer state survives the model's cache eviction, and evicting a session's tokenizer state loses a CPU optimization, not correctness.

%% file: sections/05-parallel.tex
\subsection{Chunk-Parallel Full Tokenization}
\label{sec:parallel}
\subsubsection{Safe cuts from pre-tokenizer transitions}
Initializations and stateless requests still need a full encode, and \sys parallelizes the stateless ones. Since no position in general BPE is safe to cut without reading the whole input (\S\ref{sec:challenges}), earlier parallel tokenizers cut blindly, encode overlapping chunks, and repair the seams with retries (Figure~\ref{fig:contrast}a).

Practical tokenizers are more structured. The pre-tokenizer first splits the text into pieces that BPE never merges across, and most piece boundaries are decided by the two characters that meet there, such as a letter followed by a digit. If a cut falls on such a boundary, both sides consist of the same pieces as in the whole input, so any exact encoder can encode each chunk on its own and the outputs concatenate (Figure~\ref{fig:contrast}b). Parallel BPE thus reduces to finding such boundaries cheaply and making sure that the stages before the pre-tokenizer do not move them.

\begin{figure}[t]
\centering
\includegraphics[width=\columnwidth]{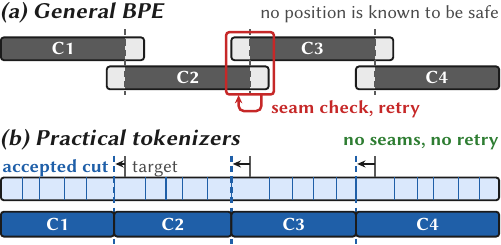}
\caption{Splitting a request for parallel encoding. (a)~Earlier systems cut blindly, encode overlapping chunks, and check the seams, retrying when a check fails. (b)~\sys moves each target cut back to a transition that ends a piece in every context, so the chunks are disjoint and nothing is retried.}
\label{fig:contrast}
\Description{Two panels over the same text bar. In the first, three blind cuts produce overlapping chunks whose seams are checked and retried. In the second, the text is divided into pieces, each target cut moves left to the nearest piece boundary, and the resulting chunks are disjoint.}
\end{figure}

\subsubsection{Preserving special tokens and normalization across cuts}
\label{sec:cuts}
Figure~\ref{fig:cuts} shows the three conditions \sys checks. First, the cut must fall on a pre-tokenizer boundary that holds in every context. \sys accepts a fixed set of character-class transitions for each admitted tokenizer, such as letter to digit or digit to punctuation, and leaves out the cases where the splitting rules look further. Whitespace can attach to the following word, an apostrophe can begin a contraction, and newlines follow rules of their own, so none of these is a cut by itself. The boundary must also survive the cut. At a letter-to-digit transition, each chunk keeps the letter run or digit run that the whole text has on its side, so the pieces are unchanged. A whitespace run, whose grouping depends on the next character, changes when it ends a chunk, which is why no transition out of whitespace is accepted.

Second, the cut must not fall inside any occurrence of a special-token string. These strings are recognized before anything else and map to their own IDs, so splitting one would turn it into ordinary text. The same stage yields a second kind of cut. When a tokenizer's special tokens neither strip surrounding whitespace nor depend on word boundaries, as holds for all 26 of Qwen3's, the text between two special tokens is encoded as a separate segment anyway, so the edges of every special-token string are safe cuts. This cut mode is evaluated on its own and is not used by the SGLang adapter. Cutting a chat prompt at its role markers keeps every chunk on the fast path (\S\ref{sec:parallel-results}).

Third, Unicode normalization must not reach across the cut. Normalization can combine a character with a following accent or rewrite one character as two~\cite{uax15}, so a boundary that exists in the raw text can disappear after normalization. The Devanagari letter U+0958, for example, normalizes to a base letter followed by a combining mark, so a slash that followed the letter now follows the mark and joins it in one piece. For tokenizers that normalize, such as Qwen3, \sys therefore checks two conditions within a short window around the cut. Normalizing the two sides separately must give the same text as normalizing them together, and the transition must still be a boundary in the normalized text. The window runs between two anchor characters at which the normalizer is known to factor, found within 256 code points on each side of the cut. Equality inside such a window is equality for the whole text, and a cut with no anchors in reach is rejected (supplement, \S S4). The U+0958 example gives the same text either way but loses the boundary. Without this check, cuts change the output on 2.1\% of a set of adversarial strings (\S\ref{sec:correctness}).

\begin{figure}[t]
\centering
\includegraphics[width=\columnwidth]{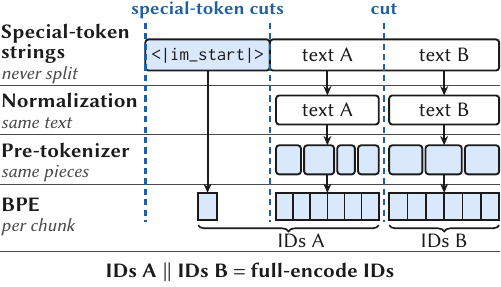}
\caption{A cut must preserve every stage that precedes BPE. Special-token strings stay intact, normalization gives the same text on both sides, and the pre-tokenizer the same pieces, so the outputs concatenate to the full encode.}
\label{fig:cuts}
\Description{A special-token string bypasses normalization and pre-tokenization. Ordinary text on the two sides of a verified cut normalizes independently, retains the same pieces, and produces concatenable ID subsequences.}
\end{figure}

Let $F$ denote tokenization without sequence-level processing. The three checks combine into the following statement, which Section~S4 of the supplement proves.

\Needspace{6\baselineskip}
\begin{proposition}[Concatenation at accepted cuts]
\label{prop:chunks}
For a frozen, admitted tokenizer configuration, let $x_1,\ldots,x_k$ divide a text at positions that satisfy the special-token, normalization, and pre-tokenizer conditions above. Then
\begin{equation}
 F(x_1\cat\cdots\cat x_k)=F(x_1)\cat\cdots\cat F(x_k).
 \label{eq:chunks}
\end{equation}
\end{proposition}

The proof follows the stages of the pipeline. The special-token condition keeps the special-token scan the same on each chunk, the normalization condition keeps the text that reaches the pre-tokenizer the same, and the transition condition keeps the ordered pieces the same. BPE is deterministic within each piece, so every chunk emits exactly its part of the whole-input IDs, and sequence-level processing is applied once to the combined result. The statement holds for any division at accepted cuts, so it needs no bound on the gap between cuts, no equal chunk sizes, and no particular completion order.

A tokenizer is admitted only after \sys checks its configuration, not its family name, uses character-class tables that match the reference, and probes the transition set against the reference pre-tokenizer. Eleven byte-level BPE tokenizers are admitted (supplement, \S S1), and other tokenizers, even new versions of admitted ones, are encoded whole.

\subsubsection{Chunk planning and execution}
The planner places tentative cuts at evenly spaced positions and scans backward from each, within a bounded distance, for an accepted cut. It never runs the pre-tokenizer over the whole request, which keeps planning a million characters into 32 chunks under 0.16\,ms. A stretch without an accepted cut yields fewer chunks, and a request without any runs whole, so the planner can lose parallelism but never exactness. Of 102.4 million web documents, only 35 or 36 per tokenizer had no accepted cut.

Each worker owns its own encoder instance, and a worker error fails the request in the pool rather than returning a partial sequence. Chunks of one request may use different encoders, the fast encoder for one and the reference for another, because both satisfy the same contract. On some inputs this matters as much as parallelism. A fast encoder's whole-call guard turns one exceptional character into a slow encode of the entire request, whereas with chunks the guard fires only for the chunk that contains the character (\S\ref{sec:parallel-results}). Chunks also give the pool a unit of work short enough to run between two updates, which the next section builds on.

%% file: sections/06-runtime.tex
\subsection{Tail-Aware Scheduling and Serving Integration}
\label{sec:runtime}
\subsubsection{Chunk-level scheduling in the CPU pool}
Once reuse makes most requests short, giving a short request priority helps only while the competing work is still queued, because an encoder cannot be interrupted in the middle of an encode. Reordering reduces waiting but cannot shorten service that has already started. With chunks, the scheduler can pick a waiting update whenever a chunk finishes, and several workers can advance different chunks of the same request at once, so a request behind it no longer waits for the whole encode on one worker (Figure~\ref{fig:scheduling}).

\begin{figure}[t]
\centering
\includegraphics[width=\columnwidth]{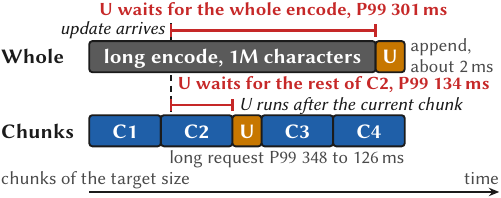}
\caption{Chunk boundaries are scheduling points. A queued update cannot interrupt a whole encode, but with chunks it can start once the current chunk completes. Each row shows one worker, and the P99 values are from the standard mix of \S\ref{sec:queue-results} with chunks in FIFO order.}
\label{fig:scheduling}
\Description{On one worker, an update arriving during a long encode waits for it to finish. With chunks, the scheduler can run the update after the current chunk.}
\end{figure}

Each worker keeps a local queue for the jobs of its sessions, and stateless tasks and chunks go to one shared queue. An idle worker takes the better of the two heads. In arrival order, the older task wins. Under shortest-first order, a task ranks by its estimated service time minus a credit that grows while it waits, which is meant to keep long requests from starving. For a task of $N$ characters the estimate is $a+b\,N(1+5\sigma)$, where $\sigma$ is the share of non-ASCII characters, $a$ is 0.2\,ms, and the per-character cost $b$ depends on the kind of task, since a session task mostly encodes an append. After each task, $b$ moves toward the value implied by the measured service time, as an exponentially weighted average with weight 0.05. The non-ASCII share is free for ASCII text and is sampled from a few short windows of a longer text, so the estimate stays cheap and rough. A wrong estimate changes only the order, never the tokens. The scheduler switches only at chunk boundaries, so the chunk size, more than the queue order, sets how long an update can wait behind a long request. The FIFO experiment of \S\ref{sec:queue-results} keeps the queue in arrival order and changes only the unit of work, which isolates this effect.

Not every job goes to the shared queue. Each session's home worker is chosen by a hash of its key. Session state stays on its home worker to avoid copying it between processes, so continuations and initializations run there and initializations remain whole jobs, which matters for workloads dominated by new sessions~(\S\ref{sec:queue-results}). The first error from any task fails its request, a worker that exits is retired and its queued session tasks move to the new home of their keys, and the adapter then encodes a failed request in its own process (\S\ref{sec:integration}). In 103 pool runs with 389,240 requests, none failed or took that path.

The main size parameter is the target chunk size. Larger chunks delay waiting updates, and smaller chunks need more dispatches. The pool splits a stateless request into chunks of about the target size, 64K characters by default, with at most one chunk per worker, and leaves requests shorter than 128K characters whole, 64K in the SGLang runs. On four workers the tail is flat over an eightfold range of targets. On 64 workers 32K-character chunks work best~(\S\ref{sec:capacity}).

\subsubsection{Integration behind the serving engine's interface}
\label{sec:integration}
\sys integrates with SGLang through a tokenizer adapter loaded at process startup, without changes to the serving engine's source. The adapter offers the same interface as the Hugging Face tokenizer it replaces, and calls with options it does not model, such as paired inputs and some truncation modes, go to the reference implementation unchanged. It takes the session identifier or cache salt when one is present and otherwise, as for a chat completion that SGLang renders before encoding, finds the session by content, from the first few thousand characters that stay the same from turn to turn (supplement, \S S6.3).

The adapter compares the tokenizer configuration with the admitted one at construction and detects later changes before reusing accelerated state, and a request whose configuration differs goes to the reference implementation. A sampled shadow check re-encodes served text with a separate copy of the reference tokenizer and compares the IDs to catch regressions in deployment.

The TTFT experiments of \S\ref{sec:incremental-results} run the adapter in process, without the pool, on the event-loop thread where SGLang encodes by default, so one long encode holds up every request behind it. The experiments of \S\ref{sec:e2e} hand requests to the worker pool instead, which answers asynchronously, and if a worker fails the adapter encodes the request in process, so the IDs never depend on the pool. Sending a session's call through the pool adds about 0.8\,ms at 100K characters and 5.7\,ms at a million, so small requests are cheaper in process.

On either path the adapter returns the complete ID list that the serving engine expects, so a repaired turn costs more as the context grows even though the encoding work does not. Through \sys's Python session interface, appending 1,300 characters to a session just created from a million characters takes 7.6\,ms, although the repair replaces only about 290 tokens. The same call takes 0.3\,ms at 16K characters, and the difference grows by about 34\,ns per token of the result, the cost of handing back the IDs as Python integers. Through the SGLang adapter, appending 29 characters to a million-character session takes 5.3\,ms, about the same as an adapter call that finds its text stored and only returns the IDs (supplement, Table~S15). An engine that accepted token updates would avoid this copy.

%% file: sections/07-evaluation.tex
\section{Evaluation}
\label{sec:evaluation}
We evaluate the three mechanisms of the design one at a time and then together in SGLang. The first experiments ask what incremental reuse adds beyond a fast stateless encoder, the second how fast the remaining full encodes run, and the third whether chunks work as scheduling units. The remaining experiments cover served requests in SGLang, including long requests that the model serves, many-core scaling, and exactness.

\subsection{Experimental Setup}
The main CPU experiments use sixteen pinned cores on one socket of a two-socket Zen 5 server (two 16-core AMD EPYC 9115, up to 4.1\,GHz, SMT disabled). Single-core measurements and repeats also run on two 16-core Zen 5 desktops (Ryzen 9 9950X, up to 5.7\,GHz). The many-core measurements, the capacity study of \S\ref{sec:capacity} and the Rust prototype of \S\ref{sec:parallel-results}, come from a rented 128-core Zen 3 node with two 64-core EPYC 7763 (up to 3.5\,GHz), on which CPU use outside our container stayed below 0.7 cores during the capacity runs (supplement, \S S6.6).
SGLang v0.5.19 runs the models on the server's 96\,GB RTX PRO 6000 or on a desktop's 32\,GB RTX 5090, while tokenization stays on the CPU. The reference is Hugging Face \code{tokenizers} 0.22.2 with frozen Qwen3, Llama-3.1, \mbox{gpt-oss}, and \mbox{DeepSeek-V3} tokenizer configurations.

The CPU inputs include English web text, code, rendered chat, and multilingual text. The SGLang sessions take their text from the instructions and tool outputs of the public Codex transcripts of \S\ref{sec:fast-floor}. Sizes are in characters, with 1M meaning one million characters, unless a size is given in tokens. Isolated encodes use distinct documents so that no document cache is measured, and worker pools are prewarmed on other text.

The scheduling testbed performs real encodes under controlled arrivals, and its standard mix has 64 sessions with at most one request in flight each and a few stateless requests. The sessions replay agent trajectories from SWE-smith~\cite{swesmith2025}, with each first turn extended by web text to the stated context, and each run measures 20 seconds after a four-second warmup. Latency is completion time minus scheduled arrival, so waiting caused by a session's own earlier request is included, and compared designs replay the same arrivals. Sweeps use one run per point, and the main comparisons are repeated with three arrival seeds. A warm turn is any turn after a session's first, when the model has cached its prefix, and TTFT intervals are bootstrap intervals over warm turns.

\paragraph{Baselines.} We refer to the configurations by fixed names. Reference is the Hugging Face tokenizer, also SGLang's default. Fast stateless is \sys's encoder without session state, so every call is encoded whole by the fast path or, when guarded, by the reference. \sys serial is the adapter in process with session reuse and without the pool, and \sys pool is the chunked pool with its shared queue in shortest-first order, four workers unless stated. fastokens 0.1.2~\cite{fastokens}, the optional backend of this SGLang version, reuses work across calls without a boundary check. LoPT~\cite{lopt2026}, whose code is not public, is reimplemented as published with the same encoder and chunk count. The parallel path is also compared with serial encoding by the same encoder and with a one-worker pool, whose own overhead changes the apparent speedup. For scheduling we compare the chunked pool with three whole-request pool designs (\S\ref{sec:queue-results}) and, in SGLang, with its multi-worker and dynamic-batching settings (\S\ref{sec:workload}). Table~\ref{tab:baselines} summarizes the baselines by whether they return the reference IDs and whether they reuse work across calls.

\begin{table}[t]
\centering
\footnotesize
\setlength{\tabcolsep}{3pt}%
\begin{tabularx}{\columnwidth}{@{}lll>{\raggedright\arraybackslash}X@{}}
\toprule
Baseline & Wrong IDs & Reuse & Result \\
\midrule
Reference & none & no & 93--99\,ms per encode (Qwen3) \\
SGLang workers & none & no & median TTFT 6--10\% below reference \\
tiktoken & none & no & 14.0\,ms vs.\ 85.8 for Hugging Face \\
Gigatoken 0.9.0 & 558 of 41.8B & no & compared in corrected form \\
Fast stateless & none & no & no faster than reference on chat \\
fastokens 0.1.2 & 78 of 279 & unchecked & as fast as \sys \\
LoPT & 30 of 21,789 & no & 1.6--2.1$\times$ slower than \sys \\
\bottomrule
\end{tabularx}
\caption{Baselines. Encode times are at 100K tokens. SGLang workers means four tokenizer workers or dynamic batching, measured under a burst. tiktoken runs on gpt-oss-120b's tokenizer. Fast stateless is our corrected Gigatoken without session state, and LoPT is our reimplementation. Wrong IDs counts encodes that differ from the reference, on the web corpus of \S\ref{sec:correctness} for Gigatoken, on replayed SGLang requests for fastokens, and on CLongEval for LoPT (supplement, \S S5). Reuse means carrying a session's tokens from one call to the next.}
\label{tab:baselines}
\end{table}

Two fast encoders appear in only part of the evaluation. tiktoken~\cite{tiktoken} is exact and fast where it can load the vocabulary, but it applies no normalizer, so it cannot serve Qwen3's tokenizer, which normalizes to NFC. On gpt-oss-120b's tokenizer, where it is native, it encodes the 100K-token turns of Figure~\ref{fig:motivation}c in 14.0\,ms against 85.8\,ms for the Hugging Face port, and the 8K and 32K turns in 1.4 and 5.1\,ms against 6.6 and 25.4\,ms. It encodes a whole Codex transcript in 11.0\,ms against 113\,ms, with equal IDs on all 610. On the same transcripts Fast stateless takes 2.0\,ms as plain text but 86\,ms in chat form, where every call is guarded, so for chat requests on this tokenizer tiktoken is the faster stateless encoder. It still encodes every call whole, so its cost grows with the context, while a repaired turn of \sys at 100K tokens takes about 2\,ms (\S\ref{sec:incremental-results}). Gigatoken as released returns different IDs from the reference on rare characters (\S\ref{sec:fast-floor}). Since a serving front end must return the reference IDs, we compare against our corrected build, the Fast stateless configuration, rather than the released one. Neither keeps a session's token sequence across calls. Gigatoken's cache of pre-tokenized pieces speeds up text it has seen before, but every call is still encoded whole.

\paragraph{Metrics.} For served requests we report the median and 99th-percentile (P99) TTFT of warm turns, measured from each turn's scheduled arrival. For the front end alone we report the median and P99 latency of the tokenizer stage against the arrival rate, and the highest arrival rate that meets a P99 objective of 50\,ms. Isolated encodes are reported as single-core throughput in MB/s and as speedup over serial encoding or over the one-worker pool. Exactness is the equality of the complete token-ID sequence with the reference, reported as the number of comparisons and the number that differ, plus the prompt-token counts that SGLang reports.

\subsection{Incremental Tokenization}
\label{sec:incremental-results}
We first measure how much of the token sequence a repair actually touches. With a fixed 1.3K-character append, the full result grows from about 3.7K to 662K tokens as the previous context grows from 16K to three million characters, while the part that the append replaces after the splice stays near 300 tokens.

Table~\ref{tab:memory} gives the state a session keeps for its next append. Beyond the transcript, \sys adds the IDs, tail positions, and shared prefix blocks of \S\ref{sec:incremental}, computed from the record layout and checked against the served store at 100K tokens, within 0.2\%. At 100K tokens this adds 1.1\,MB to a 0.4\,MB transcript, more than ten thousand times less than the KV cache in Qwen3-8B.

\begin{table}[t]
\centering
\small
\setlength{\tabcolsep}{4pt}%
\begin{tabular}{@{}rrrrrr@{}}
\toprule
 & Session & \multicolumn{3}{c}{Added by \sys} & Model \\
\cmidrule(lr){2-2}\cmidrule(lr){3-5}\cmidrule(lr){6-6}
Context & Text & IDs & Blocks & Added & KV cache \\
\midrule
8K & 31\,KB & 37\,KB & 117\,KB & 155\,KB & 1.18\,GB \\
32K & 125\,KB & 133\,KB & 562\,KB & 696\,KB & 4.72\,GB \\
100K & 391\,KB & 526\,KB & 562\,KB & 1,089\,KB & 14.75\,GB \\
250K & 977\,KB & 1,051\,KB & 562\,KB & 1,614\,KB & 36.86\,GB \\
500K & 1,954\,KB & 2,099\,KB & 562\,KB & 2,662\,KB & 73.73\,GB \\
1M & 3,908\,KB & 4,196\,KB & 562\,KB & 4,759\,KB & 147.46\,GB \\
2M & 7,816\,KB & 8,391\,KB & 562\,KB & 8,954\,KB & 294.91\,GB \\
\bottomrule
\end{tabular}
\caption{State per session by context in tokens (1\,KB = 1,000\,B), against the KV cache of Qwen3-8B for the same tokens. Text is the session's transcript. IDs include the tail's positions, blocks are prefix blocks shared through the store, counted as if no session shared them, and Added is IDs plus blocks plus digests under 1\,KB. The rows from 500K on follow the same accounting.}
\label{tab:memory}
\end{table}

Figure~\ref{fig:sglttft} shows warm-turn TTFT in SGLang with \sys serial, each turn appending 400 characters. On Qwen3-8B, median TTFT falls by 13\% at 16K tokens, 33\% at 32K, 49\% at 64K, and 61\% at 100K (Figure~\ref{fig:sglttft}a). The median goes from 55.7 to 37.5\,ms at 32K and from 136.8 to 53.7\,ms at 100K, where the P99 falls from 158.5 to 75.0\,ms. At 8K it goes from 29.1 to 29.2\,ms, since a 7\,ms reference encode does not show in TTFT. On the desktop with the RTX 5090 the median falls by 0.3\% at 8K tokens and by 28\% at 32K. On gpt-oss-120b it falls by 24, 42, and 45\% at 100K, 250K, and 350K characters, about 26K, 64K, and 90K of its tokens (Figure~\ref{fig:sglttft}c), and the P99 by 20--40\%.

The gain follows the context, as the trace analysis of \S\ref{sec:workload} leads one to expect. Below about 32K tokens, where 14\% of the recorded calls lie, the reference encode is a few milliseconds of a TTFT that the GPU sets, and reuse changes little. The 59--61\% at 100K tokens is what the long, cache-friendly sessions see. They are half of the recorded calls and carry most of the tokenizer's CPU time, and the prompt cache already serves 88--94\% of all input tokens, depending on the agent. An agent that keeps its contexts short would see much less.

\begin{figure*}[t]
\centering
\includegraphics[width=\textwidth]{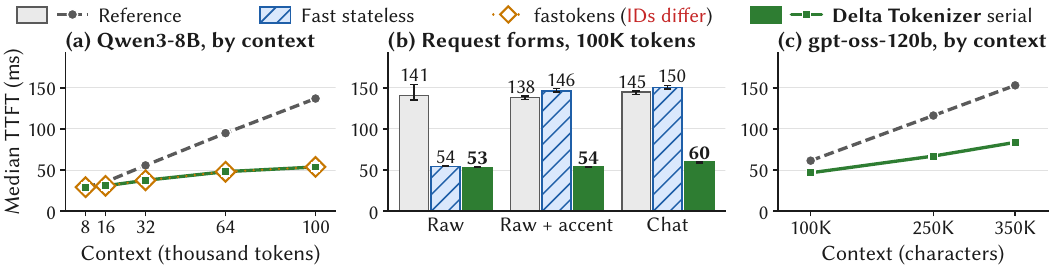}
\caption{Warm-turn median TTFT in SGLang. (a)~Qwen3-8B by context, where fastokens returns token IDs that differ from the reference. (b)~Qwen3-8B at 100K tokens by request form, extending Figure~\ref{fig:motivation}e with \sys serial, with 95\% bootstrap intervals. (c)~gpt-oss-120b by context.}
\label{fig:sglttft}
\Description{Three panels. Left: median TTFT against context for the reference, fastokens, and Delta Tokenizer on Qwen3-8B. Middle: bars for raw, accented, and chat requests at 100K tokens for the reference, Fast stateless, and Delta Tokenizer. Right: reference and Delta Tokenizer on gpt-oss-120b against context length.}
\end{figure*}

Which part of \sys carries the gain depends on the request form (Figure~\ref{fig:sglttft}b). We test clean raw text, raw text with one decomposed accent, which stands for any history with one character the normalizer would change, and requests through the chat interface, the usual way agents reach a model. On clean raw text, Fast stateless alone accounts for the gain, and median TTFT falls from 141\,ms with the reference, in a separate set of runs, to 54\,ms without session state and 53\,ms with it. The special tokens of the chat template, or the one accent, send every stateless call back to the reference path, and Fast stateless is then no faster than the reference (150\,ms and 146\,ms against 145\,ms and 138\,ms). \sys serial re-encodes only the appended turn and a short suffix before it, in about 2\,ms against 93--99\,ms for a reference encode, and reaches 60\,ms through chat and 54\,ms with the accent, 59\% below the reference through chat. The 2\,ms is the adapter's whole call and grows with the context because the returned ID list does (\S\ref{sec:integration}). The gain repeats on the desktop with the RTX 5090 and under Poisson arrivals (supplement, \S S1.3 and \S S6.7). We pace arrivals because in a closed loop a faster front end also changes how requests overlap in GPU prefill.

fastokens is as fast as \sys at every context and through the chat interface (54\,ms at 100K tokens and 60\,ms through chat), but it does not check where the reused tokens end. Replaying the 279 requests of the runs in Figure~\ref{fig:sglttft}a in serving order through one fastokens process changes the tokens of 78 of them, each time by splitting a whitespace run that the reference keeps whole, while the same text encoded in a fresh process is correct. In the raw-text serving runs, the prompt-token counts that SGLang reports with fastokens differ from the reference in 15 of 24 runs, although all six chat runs match.

We also repeat the core-count test of \S\ref{sec:workload} under open-loop load, where a turn also waits for its session's previous request, with six sessions replaying 250K-character transcripts. At eight turns per second a turn with the reference waits about 116\,ms for that request. \sys cuts the wait to 0.2\,ms and the median TTFT by 59\% at four turns per second and 71\% at eight, whether the front end has one core or eight.

\subsection{Chunk-Parallel Full Tokenization}
\label{sec:parallel-results}
Figure~\ref{fig:parallel-performance}a shows how Qwen3 encoding scales with the number of chunks, set directly for this sweep, on sixteen cores. Small inputs stop benefiting early because dispatch and assembly dominate. A four-million-character input takes 18\,ms with 32 chunks, against 150\,ms for direct serial encoding and 170\,ms on a one-worker pool, 8.2$\times$ and 9.3$\times$ faster. Tokenizers without a normalizer start faster and end lower, with Llama-3.1 at 2.6\,ms for a million characters, 4.4$\times$ faster than its one-worker pool. In all 838 configurations of the chunking sweep the planner found every requested cut, and a document of one repeated letter, with none, is encoded whole.

\begin{figure}[t]
\centering
\includegraphics[width=\columnwidth]{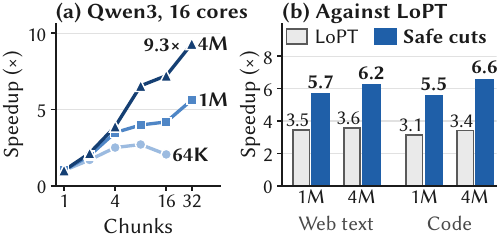}
\caption{Full-encode speedup from chunking. (a)~Qwen3 on 16 cores, relative to a one-worker pool. (b)~Speedup over serial encoding with the same Hugging Face encoder, Qwen3 with eight chunks on eight cores, for LoPT as published, whose seam matching is not exact, and for \sys's safe cuts. The 1.6--2.1$\times$ range in the text spans eight and sixteen chunks and two tokenizers (supplement, \S S5).}
\label{fig:parallel-performance}
\Description{Two panels. Left: pool-relative Qwen3 speedup versus chunk count for 64K, one-million, and four-million-character inputs on sixteen cores. Right: bars of speedup over serial encoding for LoPT and for Delta Tokenizer's safe cuts on web text and code at one and four million characters; the safe cuts are faster in every case.}
\end{figure}

\begin{figure*}[t]
\centering
\includegraphics[width=\textwidth]{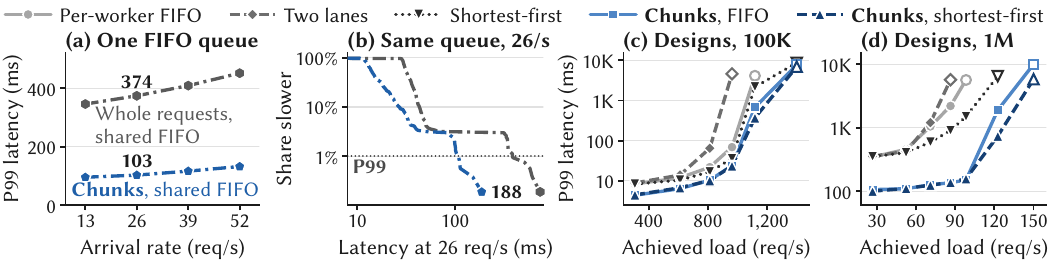}
\caption{Unit of work and queue order. (a)~One four-worker FIFO queue with the same arrival sequence, where only the unit of work changes, for million-character stateless requests. (b)~Every request of (a) at 26 requests/s. (c,~d)~P99 latency against achieved load for five pool designs on four workers, with stateless requests of 100K and 1M characters. Whole-request designs are gray, chunked designs blue, the shortest-first chunked design is \sys pool, and open markers denote overloaded runs.}
\label{fig:sched}
\Description{Four panels. The first shows P99 latency against arrival rate for whole requests and chunks in the same FIFO queue, with chunks between about 95 and 130 ms and whole requests between about 345 and 450 ms. The second shows the share of requests slower than a given latency at 26 requests per second; with whole requests about three percent exceed 300 ms, and with chunks none exceeds 200 ms. The last two show P99 latency against load for five pool designs, where the two chunked designs stay lowest.}
\end{figure*}

Against LoPT with the same encoder and chunk count (Figure~\ref{fig:parallel-performance}b), disjoint chunks are 1.6--2.1$\times$ faster on web text and code, because LoPT matches its seams one after another, and they are exact by construction, which LoPT's position matching is not (\S\ref{sec:challenges}). With 64 threads on the 128-core Zen 3 node, a native Rust prototype of the same planner encodes ten-million-character Qwen3 documents in 0.14\,s, against 4.5\,s serially. Gigatoken also splits long inputs, but only at an ASCII space between an alphanumeric character and a letter and only above a 1\,MiB floor, so a million-character request stays whole and a four-million-character one is split into about four pieces whatever the thread count. On the same host and inputs, with eight or sixteen threads, Gigatoken 0.9.0 is 10--37$\times$ faster than the single-core reference from an empty cache and 23--63$\times$ with its cache warmed by earlier calls, ratios that include its encoder's speed, against 38--97$\times$ for \sys's cuts with the fast encoder. The cuts beat Gigatoken's empty cache in all 16 settings and its warmed cache in 12, and the four exceptions are Qwen3 on web text, where the warmed cache is 2--20\% faster. All 800 outputs equal the reference.

Cuts also confine the fast encoder's guard to the chunk that needs it. On million-character English documents that trigger the guard, a whole-call fallback takes 276\,ms, against 27\,ms for documents that do not. Cut into sixteen chunks encoded one after another on one core, the guarded documents take 45\,ms, and sixteen cores bring them to 17\,ms. On the 610 chat transcripts of \S\ref{sec:fast-floor} (330K characters at the median), cutting at the role markers keeps every chunk on the fast path, matches the reference on all 610, and takes 4.2\,ms on one core, against 94\,ms for the guarded whole prompt.

Cutting at role markers is, on clean transcripts, close to the incremental path. Replaying the 610 transcripts turn by turn, a turn with 250--500K characters of history takes 4.0\,ms with cuts and 2.3\,ms incrementally at the median, and the median transcript costs 79 and 63\,ms of CPU, since the incremental path encodes the first turn whole. They part when the history contains text the fast path does not cover. With one decomposed accent in the first message, the stateless encode pays a reference encode of the chunk that contains it on every turn, 11--19\,ms, while the incremental path stays at 1.3--4.0\,ms and pays for the accent once, when it arrives, 363 against 64\,ms over a transcript. Eight of the transcripts contain such a character on their own, and there the cuts fall back on 257 turns and the incremental path on 11. Every turn's IDs match the reference (Table~\ref{tab:exact}).

\subsection{Tail-Aware Scheduling}
\label{sec:queue-results}
We first test whether a short update still waits behind a whole long request, on the standard mix on four workers, where 92\% of requests are appends to million-character contexts, with a burst of four stateless requests of a million characters each every two seconds. The appends queued behind a burst have a P99 of 301\,ms when the burst enters whole and 134\,ms when it enters as chunks, both in FIFO order, a gain in waiting time alone, since their own service time does not change with chunking. The burst's own P99 falls from 348 to 126\,ms, because four workers now share each request. Without the bursts, the P99 over all requests is 397--410\,ms in every design, set by initializations that run whole.

Figure~\ref{fig:sched}a isolates the unit of work under heavier load, changing only whether a stateless million-character request enters one shared FIFO queue on four workers as one task or as chunks. At 26 requests/s the P99 falls from 374 to 103\,ms, three arrival seeds give 353--377\,ms against 102--115\,ms, and the whole distribution moves (Figure~\ref{fig:sched}b). Here chunks shorten both waiting and service time, since every request is long enough for the four workers to share.

Figure~\ref{fig:sched}c and~d compare five pool designs. Three run whole requests, with per-worker FIFO queues, with two lanes separated by a size threshold, or with a shared shortest-first queue. Two serve chunks, in FIFO or shortest-first order. At 86 requests/s of million-character requests, the P99 is 2.2\,s with per-worker FIFO queues, 0.9\,s with the shared shortest-first queue, and 136\,ms with chunks in FIFO order. The best of twelve two-lane configurations stays 2.0--9.5$\times$ behind the chunked designs. Ordering the chunks shortest first helps further near saturation, with a P99 of 738\,ms against 1,891\,ms at 123 requests/s.

We also replay SGLang's tokenizer stage in one process, without a server or a model, on the testbed's 64 SWE-smith sessions, whose first turns carry 100K characters and whose appends are 1.3K characters at the median, with 2.3\% stateless requests. With the reference the P99 reaches 0.52\,s at 20 requests/s (Figure~\ref{fig:stage}a). \sys serial stays under 10\,ms up to 320 requests/s, and \sys pool holds 27\,ms at 640 requests/s, where the serial front end has saturated. Sessions of 250K characters shift every curve to lower rates in the same order (Figure~\ref{fig:stage}b).

\begin{figure}[t]
\centering
\includegraphics[width=\columnwidth]{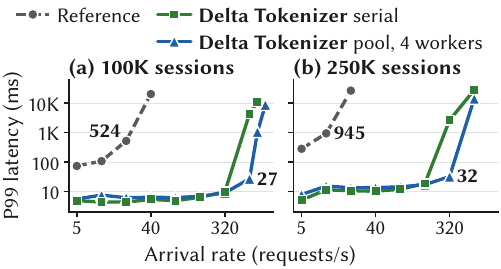}
\caption{SGLang's tokenizer stage replayed in one process, with session turns and 2.3\% stateless requests, for sessions of (a)~100K and (b)~250K characters.}
\label{fig:stage}
\Description{Two log-log panels of P99 latency against arrival rate. With 100K-character sessions the reference tokenizer exceeds half a second at 20 requests/s, Delta Tokenizer serial stays between about 4 and 10 ms up to 320 requests/s and saturates at 640, and the four-worker pool holds 27 ms at 640 requests/s. With 250K-character sessions the same order holds at lower rates.}
\end{figure}

\subsection{End-to-End Serving in SGLang}
\label{sec:e2e}
With the model in the loop and \sys pool inside SGLang's tokenizer stage, warm turns on 100K-character sessions arrive at eight per second and share the front end with bursts of four stateless requests of a million characters each every eight seconds. These requests exceed the model's context, and SGLang declines them only after tokenizing them, because it checks a request's length in tokens. They therefore load the tokenizer without reaching the GPU, and the experiment isolates the front end's effect on served turns. The P99 TTFT of the warm turns is 1,441\,ms with the reference, 166\,ms with \sys serial, and 63\,ms with \sys pool (Figure~\ref{fig:e2e-interference}a). With the reference, 17\% of warm turns take longer than 300\,ms and 8\% longer than a second, while the slowest turn takes 183\,ms with \sys serial and 101\,ms with \sys pool (Figure~\ref{fig:e2e-interference}b). Without long requests, the front end still sets the tail under load. At 24 warm turns per second, the P99 TTFT is 995\,ms with the reference and 103--105\,ms with \sys serial and \sys pool, and the pool adds 1 to 2\,ms to the median relative to \sys serial (supplement, \S S6.7). When the long requests fit the model's context (200K characters) and each one is new, warm turns wait in GPU prefill, which the front end does not touch, and the P99 is 5.2--5.4\,s with all three tokenizers. The next section asks when the front end still sets the tail while the model serves them.

\begin{figure}[t]
\centering
\includegraphics[width=\columnwidth]{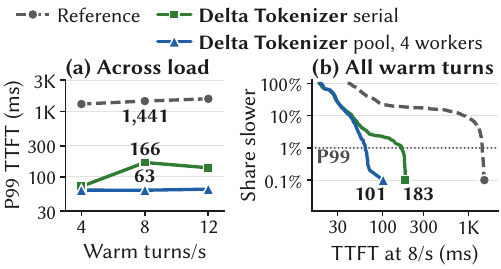}
\caption{Front-end interference on served agent turns in SGLang, while four stateless requests of one million characters arrive every eight seconds and are tokenized before SGLang declines them for their length. (a)~Warm-turn P99 TTFT by warm-turn rate. (b)~Share of warm turns slower than a given TTFT at eight turns per second.}
\label{fig:e2e-interference}
\Description{Two panels. Left: warm-turn P99 TTFT at 4, 8, and 12 warm turns per second for the reference tokenizer, Delta Tokenizer serial, and Delta Tokenizer pool. Right: at eight turns per second, the share of warm turns slower than each TTFT; the reference has a long tail above one second, Delta Tokenizer serial ends near 180 ms, and the pool near 100 ms.}
\end{figure}

\subsection{Served Long Requests}
\label{sec:served}
The 23$\times$ of \S\ref{sec:e2e} comes from long requests that SGLang tokenizes and then declines. Here SGLang accepts and serves them. Next to twelve sessions with 100K-character histories and eight warm turns per second, documents of about 400K characters, about 100K tokens and within the model's context of 131,072 tokens, arrive in clusters, and each is answered. In the repeated form, a cluster of four requests for the same document arrives every eight seconds, so after its first prefill the model finds the document in its prefix cache and computes one new token per request. In the new-document form, two documents never seen before arrive every 40 seconds. Each configuration runs with the reference, \sys serial, and \sys pool, with two or three arrival seeds of 60 seconds each, on the Zen 5 server with Qwen3-8B unless stated (Figure~\ref{fig:served-documents}).

\begin{figure*}[t]
\centering
\includegraphics[width=\textwidth]{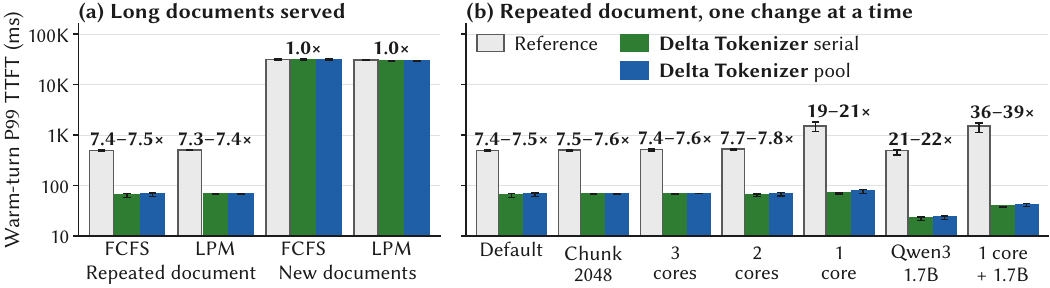}
\caption{Warm-turn P99 TTFT in SGLang next to served documents of about 100K tokens. Bars pool a configuration's arrival seeds, two or three of 60\,s each, error bars give each seed's P99, and the labels give the reference's P99 over that of \sys serial and \sys pool. Qwen3-8B unless noted, twelve sessions, SGLang 0.5.19. (a)~A repeated document, resent every 8\,s and cached by the model, and two new documents every 40\,s, under first-come-first-served (FCFS) and longest-prefix-match (LPM) order. (b)~One change at a time on the repeated document with FCFS order: prefill chunks of 2,048 tokens, the server pinned to three, two, or one cores, and Qwen3-1.7B in place of Qwen3-8B, alone and together with one core.}
\label{fig:served-documents}
\Description{Two panels of warm-turn P99 TTFT on a log axis for the reference, Delta Tokenizer serial, and Delta Tokenizer pool. With repeated documents the reference is near half a second and Delta Tokenizer near 70 ms; with new documents all three are near 30 seconds. The second panel shows the repeated configuration under several levers; one core raises the reference to about 1.6 seconds, and the smaller model lowers Delta Tokenizer to about 24 ms.}
\end{figure*}

With repeated documents the front end sets the tail. The P99 TTFT of warm turns is 498\,ms with the reference and 66.5 and 67.5\,ms with \sys serial and \sys pool, 7.4--7.5$\times$ lower (Figure~\ref{fig:served-documents}a). The reference spends about 110\,ms of serial tokenization on SGLang's event loop for each document, so the four documents of a cluster hold the warm turns behind them for about 0.4\,s, while those turns wait at most 1.5\,ms in the GPU scheduler's queue at the P99. With \sys a warm turn waits a few milliseconds before reaching the GPU, and its remaining tail is its own forward pass, 44--48\,ms at the P99. This holds across SGLang's settings. With FCFS or LPM order, prefill chunks of 2,048 or 8,192 tokens, and the server pinned to three cores or not, the reference's P99 stays at 498--515\,ms and \sys's at 66.5--70.5\,ms, 7.3--7.6$\times$ lower, and a rerun of one configuration with the same arrivals moves by less than 3\%. At twice the turn rate both tails rise and the gap narrows to 6.3--6.5$\times$, as more turns queue behind each document's prefill (supplement, \S S6.7).

Two changes widen the gap to about 20$\times$. The first takes cores away. With the whole server pinned to one core, SGLang's scheduler, which polls in a busy loop, shares that core with the event loop that tokenizes, so each reference encode takes about twice as long and the reference's P99 rises to 1,511\,ms. \sys's rises to 71.7\,ms serial and 77.6\,ms with the pool, 21$\times$ and 19$\times$ lower. With three cores the scheduler and the event loop each have a core of their own, and the readings stay within the rerun noise. With two cores they still do, 524\,ms for the reference and 66.8 and 68.1\,ms for \sys, so the step lies between two cores and one. With one core the pool's four workers share it too and its document encodes slow down, so there the serial adapter is the better configuration.

The second change gives each turn less GPU work. With Qwen3-1.7B, which uses the same tokenizer files, \sys's P99 falls to 22.5 and 23.7\,ms because the forward pass of a warm turn is shorter, while the reference's stays at 493\,ms, 22$\times$ and 21$\times$ higher. On two cores the smaller model gives the same picture, and only the pool, whose workers then share the cores with the server, is about 15\% slower. The two changes add up but do not multiply. With one core and Qwen3-1.7B together, the reference's P99 is 1,508\,ms and \sys's 38.6 and 41.7\,ms, 36--39$\times$ apart rather than the 60$\times$ the two ratios would give, because on one core \sys's own encodes also take about twice as long, and that doubling is a visible share of a 23\,ms tail.

When every document is new, the GPU sets the tail. Two cold prefills of about 106K tokens hold the GPU for about 30 seconds of every 40, and the warm-turn P99 is 29.4--32.0\,s with all three front ends, 0.9--3.6\% higher with the reference. The front end shows only while the backlog drains after each pair of documents. About 230 queued warm turns drain in 6.4--7.7\,s with the reference and 3.9--5.0\,s with \sys, so the reference's median is 1.4--1.9\,s higher. LPM order lowers the median by about 2\,s but the P99 by only 1.0--1.7\,s, and with Qwen3-1.7B the median wait behind a pair of documents by about 4\,s but the P99 by 0.5\,s. The P99 moves less because each session has one request in flight, so a turn delayed by the documents delays the session's next turn. A faster prefill does not move this boundary. With Qwen3-1.7B each cold prefill takes 5.0--6.6\,s instead of 15--18\,s, the warm-turn P99 is 9.9--11.0\,s with all three front ends, and the reference is 2--5\% higher because its backlog drains in 2.1--3.2\,s against 0.5--0.6\,s with \sys.

The same picture holds on a second machine. On a GH200 node, with a 72-core Grace host CPU of which the server used 16 cores, and the fast encoder rebuilt for that host and rechecked before the runs (supplement, \S S7), the repeated document gives a P99 of 630\,ms with the reference and 47.5 and 47.0\,ms with \sys, 13$\times$ apart, since the reference's encodes are 20--30\% slower on that CPU and a warm turn's forward pass is faster on that GPU. With one core the three figures are 2,812, 77.7, and 99.9\,ms. The direction is the same on both machines, and the multiple depends on the platform.

These runs cover two machines, repeated and new documents, twelve sessions, and two or three seeds per configuration, mostly of 60 seconds. In all 129 runs every served prompt-token count matched the reference, for 1,548 final session turns and 3,021 documents, as did all 2,268 shadow re-encodes.

\subsection{Scalability Under a Latency Objective}
\label{sec:capacity}
The load comes from several client processes, since one cannot generate enough. For each pool size we report the highest tested arrival rate of million-character stateless requests whose P99 stays at or below 50\,ms in all three arrival seeds, and the lowest tested rate that fails in every seed. With 16, 32, and 64 workers these brackets are 244 to 290 requests/s (the default 64K target), 530 to 620 (one chunk per worker of about 31K characters), and 750 to 875 (a 32K target).

\begin{table}[t]
\centering
\setlength{\tabcolsep}{3pt}%
\begin{tabularx}{\columnwidth}{@{}>{\raggedright\arraybackslash}Xrr@{}}
\toprule
Check & Compared & Differ \\
\midrule
Fast encoder, whole encodes, 11 tokenizers & 41.8B & 0 \\
Web corpus split at \sys's cuts, 4 tokenizers & 409.6M & 0 \\
Special-token cuts, chat prompts & 610 & 0 \\
Turn-by-turn replay, 610 chat transcripts & 323{,}680 & 0 \\
Adversarial battery, 6 tokenizers & 220{,}000 & 0 \\
\rowcolor{ctrlrow}\quad \emph{control, expected to differ}, without the normalization check, Qwen3 & 100{,}000 & 2{,}115 \\
Platform and runtime repeats (supplement, \S S1) & 1{,}106{,}787 & 0 \\
SGLang, reported prompt tokens & every run & 0 \\
SGLang, shadow re-encodes & 2{,}195 & 0 \\
SGLang with served long documents, shadow re-encodes & 2{,}268 & 0 \\
SGLang, greedy completions & 550 & 0 \\
\bottomrule
\end{tabularx}
\caption{Exactness checks against the reference. Each comparison covers a complete token-ID sequence, except the prompt-token counts that SGLang reports and the greedy completions, which compare the model's output. The shaded row is a control, not a check of \sys. It removes the normalization check, and its nonzero count shows that the battery detects what that check prevents.}
\label{tab:exact}
\end{table}

Capacity grows close to linearly up to 32 workers. At 64 workers the target chunk size matters. At 750 requests/s the 32K target keeps the P99 at 30--44\,ms in the three seeds, 64K-character chunks give 63--78\,ms, and one chunk per worker, twice as many chunks as at 32K, leaves it at 311--340\,ms (supplement, \S S6.6). The best target shrinks as workers are added, but not to one chunk per worker, since the dispatcher's CPU time grows with the number of chunks. At 100K characters, where requests stay whole, the shared queue meets the same objective up to about 6,000 requests/s on 64 workers, while per-worker FIFO queues stop scaling near 3,000. Trace-driven agent traffic asks for much less. Replayed from each session's first call, while contexts are still short, it leaves four workers at most 15\% busy, and at the contexts the same traces reach in steady state the same arrivals would ask the reference tokenizer for 16--20 times more CPU. The case for the front end is latency first, and capacity once many sessions or long requests share it.

\subsection{Exactness}
\label{sec:correctness}
Table~\ref{tab:exact} lists the checks, and apart from the control none finds a difference. Propositions~\ref{prop:splice} and~\ref{prop:chunks} say why a repaired turn and a chunked request return the reference IDs for an admitted configuration, and the checks test at three levels that the implementation and its admission keep to them.

The first level is the encoder that both paths use. The largest check compares the fast encoder's whole encodes with the reference on 3.8 billion web documents, 12.3 trillion characters, under each of the eleven admitted tokenizers, 41.8 billion encodes in all. The second level is the chunked path. 102.4 million of these documents, between 512 and 400,000 characters, are split at \sys's cuts under four tokenizers and compared with whole encodes by the fast encoder and by the reference, and the adversarial battery targets normalization seams and special-token strings under six tokenizers. Its control removes the normalization check and finds 2,115 differences in 100,000 Qwen3 strings, which shows that the battery reaches the cases the check exists for. The third level is the repair path and serving. The turn-by-turn replay of 610 chat transcripts compares every turn of three front ends with the reference, and the SGLang checks cover the served requests themselves, including repaired turns and the served long documents of \S\ref{sec:served}. Supplement \S S7 states what each row counts.

%% file: sections/08-discussion-related.tex
\section{Related Work}
\label{sec:related}
\paragraph{Fast serial and GPU tokenizers.}
The Hugging Face \code{tokenizers} library is the reference implementation for most open models and encodes one request on one core~\cite{hftokenizers}. tiktoken~\cite{tiktoken} is a faster exact encoder for OpenAI's vocabularies and applies no normalizer, and Gigatoken~\cite{gigatoken2026} speeds up pre-tokenization with SIMD and a cache of pre-tokenized pieces. \sys uses a build of Gigatoken corrected to the reference's Unicode data (\S\ref{sec:fast-floor}). These encoders make each full encode cheaper, while \sys removes most full encodes and divides the rest. GPU tokenizers move the encode onto the accelerator. BlockBPE explores GPU BPE without regex pre-tokenization~\cite{blockbpe2025}, GPUTOK implements GPU byte-level BPE~\cite{gputok2026}, and RAPIDS provides GPU WordPiece tokenization for BERT-style workloads~\cite{cudfsubword}. TokTier~\cite{toktier2026}, the preprint of July 2026 in which we first described the session repair of this paper together with a GPU route for fresh long inputs, rewrites the reference regex as a run decomposition so that pre-tokenization runs exactly on the GPU, merges on the GPU, and replays one captured CUDA graph per capacity bucket with the live counts kept on the device, byte-identical to its frozen references on the six configurations it certifies for the GPU. GraphTok~\cite{graphtok2026}, posted in October 2026, describes itself as concurrent with and independent of TokTier. The three contributions it lists, exact parallel tokenization of each family's frontend on the GPU, graph-resident execution that separates a captured graph's capacity from a request's occupancy, and a complete GPU encoding path with host-visible or device-resident output, are the mechanisms TokTier published ten weeks earlier (TokTier, \S4.1 to \S4.3). What GraphTok adds is an overflow path for very long pre-tokens inside the captured graph, the SentencePiece family of Gemma, a dispatcher that merges concurrent requests into one replay, and a Dynamo backend with a study of how many GPUs one front end can feed. Its comparison with TokTier disables TokTier's session reuse and times fresh encodes, and it states that TokTier's released package lacks a CUDA-graph path with array delivery and has a GPU path only in Python; the 0.2.9 release of August 2026 contains that path, selectable through its engine interface, and ships the GPU runtime in its Rust crate as well. GraphTok also starts from the premise that every request must be tokenized in full before the prefix cache can be consulted, and its vLLM experiments, which follow the agentic pattern of a long shared prefix and a short new turn, tokenize the whole prompt on every call. For the turns of a session \sys does without that premise. A turn re-encodes only its appended text and a short window before it, at any boundary the check accepts rather than only at a chat-turn or special-token boundary, and the long requests with no state to reuse it divides across the cores of the front end. By GraphTok's own measurements, fast CPU encoders come within a few percent of it in the TTFT of a single request, and its case rests on the capacity of a front end that encodes every request in full for a fleet of GPUs. \sys keeps the encode on the CPU of the front end, next to request rendering, and reduces the work rather than moving it, and once the front end is this cheap the model's prefill sets the tail whenever every long request is new (\S\ref{sec:served}).

\paragraph{Parallel tokenization of one request.}
LoPT~\cite{lopt2026} cuts a long input into overlapping chunks without regard to the tokenizer, encodes them in parallel, and joins neighboring chunks where their tokens in the overlap coincide by character position, retrying with longer chunks up to five times before it encodes serially. As published, the position match can accept a seam at which the two chunks divide the bytes of a character differently (supplement, \S S5). \sys's chunks are disjoint and end where the tokenizer's own stages cannot move a boundary, so they concatenate without matching or retries. Gigatoken's own split is also disjoint. It cuts only at an ASCII space between an alphanumeric character and a letter, and only in inputs above a 1\,MiB floor. It keeps added-token strings whole across a cut and leaves normalization to the chunks without stating a condition, and unit tests rather than an argument cover its seams. \sys accepts a wider set of transitions for each admitted tokenizer, states the special-token and normalization conditions with a proof, and splits at any target size (\S\ref{sec:parallel-results}).

\paragraph{Incremental and streaming BPE\@.}
Incremental and streaming BPE algorithms bound how far the merge stage looks when text is appended and give it a formal treatment~\cite{incrementalbpe2026,streamingbpe2026,berglund2023bpe}. Fixed BPE dictionaries admit a finite, dictionary-dependent lookahead for appended input~\cite{berglund2023bpe}, but this does not by itself give safe boundaries for normalization, special-token matching, or regex pre-tokenization. \sys's checks cover these stages before the merges, which decide the pieces a serving front end must keep exact, and it checks each boundary at run time instead of a fixed overlap.

\paragraph{Reuse in serving front ends.}
Serving front ends reuse an encoding only when the whole text is cached or the reused prefix ends at a registered special-token string~\cite{dynamotokenizers2026,llmtokenizer2026}, while \sys repairs an append at any boundary its check accepts, for every admitted tokenizer. fastokens~\cite{fastokens}, an optional tokenizer backend of SGLang, reuses earlier work across calls without a check at the seam, and the version we evaluated returned different IDs for 78 of 279 replayed requests (\S\ref{sec:incremental-results}). SGLang can also run several tokenizer workers or batch tokenization requests dynamically~\cite{sglang2024}, which overlaps requests rather than shortening one (\S\ref{sec:workload}). Model-side prefix caching~\cite{vllm2023,sglang2024,mooncake2025} reuses computation after tokenization. Indexed by token IDs, it complements tokenizer-side reuse and cannot replace it.

\newpage
\paragraph{Scheduling for the tail.}
Size-interval assignment~\cite{harcholbalter1999sita}, shortest-remaining-time scheduling~\cite{schrage1968,harcholbalter2003srpt}, and queueing analysis of variable service times~\cite{kleinrock1975,harcholbalter2013} motivate dividing and ordering work by size for the tail~\cite{deanbarroso2013}. Shinjuku preempts long requests at microsecond scale~\cite{shinjuku2019}, but a tokenizer worker cannot switch requests inside an encode, so \sys makes the long work divisible where the output is known not to change. Sarathi-Serve interleaves chunked prefill with decoding~\cite{sarathi2024}, and \sys divides work one stage earlier, where a chunk must be safe and not merely sized.

\paragraph{Agent workloads and CPU costs of serving.}
CPU-induced inference slowdowns have been characterized~\cite{cpuslowdown2026}, as have coding-agent workloads~\cite{tracelab2026,cachewise2026} and production LLM request patterns~\cite{burstgpt2025}. In its Claude Code subset, TraceLab reports the imbalance of \S\ref{sec:workload} between long cached prefixes and short appends, 126K cached-prefix tokens against 857 appended at the median~\cite{tracelab2026}.

%% file: sections/09-conclusion.tex
\section{Conclusion}
\label{sec:conclusion}
\sys treats a growing agent request as an incremental update to an existing token sequence, and an exact boundary check lets it encode only the new text and a short stretch before it. Cut points that practical tokenizers already provide let it split long stateless requests into chunks that run between short updates. In SGLang, this lowers the median time to first token by 59\% at 100K-token chat contexts and keeps agent turns fast while long requests share the front end, whether the engine declines those requests or serves them from the model's cache, with every compared output identical to the reference.

\paragraph{Limitations.}
\label{sec:limits}
Initializations still run whole on their home worker, so workloads dominated by new long sessions keep a long tail, and when every long request is new the GPU's prefill sets the tail whatever the front end does (\S\ref{sec:served}). The special-token cuts of \S\ref{sec:cuts} are not yet wired into the SGLang adapter, and the pool helps only when it has cores of its own. Each pool worker holds about 320\,MB of memory. Evicting a session frees its entry and its text, but the prefix blocks it shared through the store stay, so with many long sessions whose contents differ the store grows with the number of sessions, by about 0.56 to 0.75\,MB per session at 100K tokens. Worker startup and recovery are not evaluated.

\paragraph{Use of AI tools.}
We used AI assistants (Anthropic's Claude; OpenAI's Codex and ChatGPT) to write and review code and scripts, analyze measurements, prepare figures, and draft and edit the text. The authors take full responsibility for the content.

%% file: sections/supplement-body.tex
This supplement accompanies the main paper. Section~S1 lists the tokenizers that \sys admits, the public models that ship them, and the checks and results for each. Sections~S2 to~S4 state the tokenizer contract and the correctness arguments for incremental repair and for safe cuts in full. Section~S5 shows that LoPT's seam matching, as published, is not exact, and compares its speed with that of safe cuts. Section~S6 defines the measurements and adds results that the main paper has no room for, and Section~S7 states what each exactness check counts.

Platforms are named as in the main paper. The Zen 5 server has two 16-core AMD EPYC 9115 processors (up to 4.1\,GHz, SMT disabled) and an RTX PRO 6000 GPU, and the experiments use one of the two processors. The Zen 5 desktops are two machines with a 16-core Ryzen 9 9950X (up to 5.7\,GHz), one of them with an RTX 5090. The 128-core Zen 3 node has two 64-core EPYC 7763 processors (up to 3.5\,GHz) and was rented for the capacity study and the Rust prototype. Earlier capacity runs and a first run of the prototype used shared cluster nodes with two 64-core EPYC 7713 processors (up to 3.7\,GHz), and Sections~\ref{sec:ev-capacity} and~\ref{sec:ev-rust} report them for comparison.

\section{Supported tokenizers and models}
\label{sec:tokenizers}
\sys decides whether it may cut a tokenizer's input into chunks from the tokenizer file and the configuration the loader builds from it, not from the model's name. It first computes the SHA-256 digest of the tokenizer file and looks it up in a table of eleven frozen byte-level BPE tokenizers, each stored with its own repair parameters. If the file is neither one of the eleven nor a recognized equivalent of one (Section~\ref{sec:public-models}), the reference encodes every input. Otherwise \sys reads the configuration that the loader actually builds, which may add special-token strings declared outside the tokenizer file. This configuration must either leave the text unnormalized or normalize it to NFC and match every special-token string on the raw text, before normalization. If it does neither, \sys never cuts the input.

\subsection{The admitted tokenizers}
Table~\ref{tab:tokenizers} lists the eleven tokenizers. Each tokenizer file was copied once from the public model repository in the table, at the revision shown, and the experiments load these frozen copies. Three of the tokenizers normalize text to NFC, two carry an empty list of normalizers that leaves text unchanged, and six have no normalizer. They recognize between 21 and 1,283 special-token strings, the strings such as chat-template markers that are matched before the text is split. These counts come from the loaded configuration, not from the tokenizer file alone. Qwen3.5-0.8B shows the difference, since its tokenizer file declares 26 special-token strings and its separate tokenizer configuration file declares seven more, so the guard that keeps cuts away from these strings works with 33.

The last columns give the two repair parameters that differ between tokenizers (Section~\ref{sec:splice}). The margin is the stretch at the left end of the repair window that the search for a match leaves out. It equals the length of the longest special-token string, but is at least 8 characters. The bound is the longest text that a single token decodes to, and an accepted match must cover more normalized characters than that. The other repair settings are the same for all eleven tokenizers. The first window spans 512 characters of old text, the window doubles at most five times, and a match needs at least two tokens.

\begin{table}[htbp]
\centering
\footnotesize
\setlength{\tabcolsep}{4pt}
\begin{tabular}{@{}l>{\raggedright\arraybackslash}p{2.3in}lrrrr@{}}
\toprule
Tokenizer & Repository; revision and file digest & Normalizer & \shortstack[r]{Special\\strings} & Margin & Bound & \shortstack[r]{Also\\in} \\
\midrule
Qwen3-8B & Qwen/\allowbreak Qwen3-8B\newline \texttt{b968826d9c46}\quad\texttt{aeb13307a71a} & NFC & 26 & 20 & 128 & 79 \\
Llama-3.1-8B & meta-llama/\allowbreak Llama-3.1-8B\newline \texttt{d04e592bb4f6}\quad\texttt{76e48799b099} & none & 256 & 30 & 128 & 3 \\
gpt-oss-120b & openai/\allowbreak gpt-oss-120b\newline \texttt{b5c939de8f75}\quad\texttt{0614fe83cada} & none & 21 & 19 & 128 & 3 \\
DeepSeek-V3 & deepseek-ai/\allowbreak DeepSeek-V3\newline \texttt{e815299b0bcb}\quad\texttt{621ac2e32d0d} & empty list & 818 & 23 & 128 & 2 \\
Qwen3.5-0.8B & Qwen/\allowbreak Qwen3.5-0.8B\newline \texttt{2fc06364715b}\quad\texttt{5f9e4d4901a9} & NFC & 33 & 21 & 128 & 19 \\
MiniMax-M3 & MiniMaxAI/\allowbreak MiniMax-M3\newline \texttt{50942730318c}\quad\texttt{bb1f1626cf01} & NFC & 61 & 31 & 128 & 1 \\
DeepSeek-V4-Flash & deepseek-ai/\allowbreak DeepSeek-V4-Flash\newline \texttt{60d8d70770c6}\quad\texttt{8f9f37ca37fd} & empty list & 1,283 & 36 & 128 & 7 \\
GLM-5.2 & zai-org/\allowbreak GLM-5.2\newline \texttt{b4734de4facf}\quad\texttt{19e773648cb4} & none & 36 & 26 & 512 & 7 \\
OLMo-3-7B & allenai/\allowbreak Olmo-3-1025-7B\newline \texttt{a81bae42db39}\quad\texttt{73fd5254624f} & none & 22 & 19 & 128 & 7 \\
Ministral-3-8B & mistralai/\allowbreak Ministral-3-8B-Instruct-2512-BF16\newline \texttt{f6fae9795746}\quad\texttt{d5f6046775b1} & none & 1,000 & 18 & 76 & 2 \\
Nemotron-3-Nano-4B & nvidia/\allowbreak NVIDIA-Nemotron-3-Nano-4B-BF16\newline \texttt{dfaf35de3e30}\quad\texttt{623c34567aeb} & none & 1,000 & 18 & 76 & 16 \\
\bottomrule
\end{tabular}
\caption{The eleven admitted tokenizers. Each tokenizer file was frozen from the named repository at the revision given by the first hexadecimal string. The second string begins the SHA-256 digest of the file, which admission compares. Special strings are counted in the loaded configuration. Margin and Bound are the repair parameters in characters. Also in counts the other public repositories whose tokenizer file has the same digest (Table~\ref{tab:same-file}).}
\label{tab:tokenizers}
\end{table}

\subsection{Public models with the same tokenizer}
\label{sec:public-models}
Many public models ship one of these files unchanged. Table~\ref{tab:same-file} lists 146 public repositories whose tokenizer file has the same SHA-256 digest as one of the eleven. Most are other sizes, later versions, or quantized releases of the same model lines, and three are NVIDIA models that ship the Qwen3-8B file. For such a repository \sys runs the frozen tokenizer, and the checks of that tokenizer apply unchanged because the bytes are the same. In SGLang, the adapter also compares the configuration that the server loads with the frozen one, section by section, and leaves encoding to SGLang's own tokenizer if any section differs. The list was compiled in August 2026 by downloading and hashing the files of public repositories. \sys hashes the file again whenever it loads a model, so a repository that later changes its file no longer matches.

Another 48 public repositories ship a file whose bytes differ from a frozen file while the tokenizer it describes is the same. In 46 of them the difference lies in JSON formatting or in an older way of writing the same fields, and for the other two the loader builds exactly the frozen configuration from the repository's files. \sys maps these repositories to the matching frozen tokenizer and runs the frozen bytes. A repository whose tokenizer differs in anything else, for instance only in its special-token strings, is not mapped to a frozen tokenizer and is encoded by the reference.

\begin{table}[htbp]
\centering
\footnotesize
\setlength{\tabcolsep}{4pt}
\hyphenpenalty=10000
\exhyphenpenalty=10000
\begin{tabular}{@{}lr>{\raggedright\arraybackslash}p{4.75in}@{}}
\toprule
Tokenizer & Count & Public repositories with the same tokenizer file \\
\midrule
Qwen3-8B & 79 & Qwen/Qwen3-0.6B (+FP8, GPTQ-Int8, MLX-4bit, MLX-6bit, MLX-8bit, MLX-bf16); Qwen/Qwen3-1.7B (+FP8, GPTQ-Int8, MLX-4bit, MLX-6bit, MLX-8bit, MLX-bf16); Qwen/Qwen3-4B (+FP8, AWQ, MLX-4bit, MLX-6bit, MLX-8bit, MLX-bf16); Qwen/Qwen3-4B-Instruct-2507 (+FP8); Qwen/Qwen3-4B-SafeRL; Qwen/Qwen3-4B-Thinking-2507 (+FP8); Qwen/Qwen3-8B (FP8, AWQ, MLX-4bit, MLX-6bit, MLX-8bit, MLX-bf16); Qwen/Qwen3-14B (+FP8, AWQ, MLX-4bit, MLX-6bit, MLX-8bit, MLX-bf16); Qwen/Qwen3-30B-A3B (+FP8, GPTQ-Int4, MLX-4bit, MLX-6bit, MLX-8bit, MLX-bf16); Qwen/Qwen3-30B-A3B-Instruct-2507 (+FP8); Qwen/Qwen3-30B-A3B-Thinking-2507-FP8; Qwen/Qwen3-32B (+FP8, AWQ, MLX-4bit, MLX-6bit, MLX-8bit, MLX-bf16); Qwen/Qwen3-235B-A22B (+FP8, GPTQ-Int4, MLX-6bit, MLX-8bit, MLX-bf16); Qwen/Qwen3-235B-A22B-Instruct-2507 (+FP8); Qwen/Qwen3-235B-A22B-Thinking-2507-FP8; Qwen/Qwen3-Coder-30B-A3B-Instruct-FP8; Qwen/Qwen3-Coder-480B-A35B-Instruct-FP8; Qwen/Qwen3-Next-80B-A3B-Instruct (+FP8); Qwen/Qwen3-Next-80B-A3B-Thinking (+FP8); Qwen/Qwen3-Reranker-0.6B; Qwen/Qwen3-Reranker-4B; Qwen/Qwen3-Reranker-8B; Qwen/Qwen3-VL-Reranker-2B; Qwen/Qwen3-VL-Reranker-8B; nvidia/Nemotron-Terminal-8B; nvidia/Nemotron-Terminal-14B; nvidia/Nemotron-Terminal-32B \\
Llama-3.1-8B & 3 & meta-llama/Llama-3.1-70B; meta-llama/Llama-3.1-405B (+FP8) \\
gpt-oss-120b & 3 & openai/gpt-oss-20b; openai/gpt-oss-safeguard-120b; openai/gpt-oss-safeguard-20b \\
DeepSeek-V3 & 2 & deepseek-ai/DeepSeek-V3-0324; deepseek-ai/DeepSeek-V3-Base \\
Qwen3.5-0.8B & 19 & Qwen/Qwen3.5-2B; Qwen/Qwen3.5-4B; Qwen/Qwen3.5-9B; Qwen/Qwen3.5-27B (+FP8, GPTQ-Int4); Qwen/Qwen3.5-35B-A3B (+FP8, GPTQ-Int4); Qwen/Qwen3.5-122B-A10B (+FP8, GPTQ-Int4); Qwen/Qwen3.5-397B-A17B (+FP8, GPTQ-Int4); Qwen/Qwen3.6-27B (+FP8); Qwen/Qwen3.6-35B-A3B (+FP8) \\
MiniMax-M3 & 1 & MiniMaxAI/MiniMax-M3-MXFP8 \\
DeepSeek-V4-Flash & 7 & deepseek-ai/DeepSeek-V4-Flash-0731; deepseek-ai/DeepSeek-V4-Flash-Base; deepseek-ai/DeepSeek-V4-Flash-DSpark; deepseek-ai/DeepSeek-V4-Pro; deepseek-ai/DeepSeek-V4-Pro-0813; deepseek-ai/DeepSeek-V4-Pro-Base; deepseek-ai/DeepSeek-V4-Pro-DSpark \\
GLM-5.2 & 7 & zai-org/GLM-4.7-Flash; zai-org/GLM-5 (+FP8); zai-org/GLM-5.1 (+FP8); zai-org/GLM-5.2-FP8; zai-org/GLM-5.3-Flash \\
OLMo-3-7B & 7 & allenai/Olmo-3-7B-Think; allenai/Olmo-3-7B-Think-DPO; allenai/Olmo-3-32B-Think; allenai/Olmo-3-32B-Think-DPO; allenai/Olmo-3-1125-32B; allenai/Olmo-3.1-7B-RL-Zero-Math; allenai/Olmo-3.1-32B-Think \\
Ministral-3-8B & 2 & mistralai/Ministral-3-3B-Instruct-2512-BF16; mistralai/Ministral-3-14B-Instruct-2512-BF16 \\
Nemotron-3-Nano-4B & 16 & nvidia/NVIDIA-Nemotron-3-Nano-4B-FP8; nvidia/NVIDIA-Nemotron-3-Nano-30B-A3B-Base-BF16; nvidia/NVIDIA-Nemotron-3-Super-120B-A12B (BF16, FP8, NVFP4); nvidia/NVIDIA-Nemotron-3-Super-120B-A12B-Base-BF16; nvidia/NVIDIA-Nemotron-3-Ultra-550B-A55B (BF16, NVFP4); nvidia/NVIDIA-Nemotron-3-Ultra-550B-A55B-Base-BF16; nvidia/NVIDIA-Nemotron-3-Ultra-550B-A55B-GenRM; nvidia/NVIDIA-Nemotron-Labs-3-Elastic-30B-A3B (BF16, FP8, NVFP4); nvidia/NVIDIA-Nemotron-Labs-3-Puzzle-75B-A9B (BF16, FP8, NVFP4) \\
\midrule
Total & 146 & \\
\bottomrule
\end{tabular}
\caption{Public model repositories whose tokenizer file has the same SHA-256 digest as the frozen file of a tokenizer. A parenthesized list names suffixes that further repositories append to the name before it, and a leading plus sign means that the name itself is included. Qwen/Qwen3-4B-Instruct-2507 (+FP8) thus stands for two repositories, the model and its FP8 release, and nvidia/NVIDIA-Nemotron-3-Ultra-550B-A55B (BF16, NVFP4) for the two suffixed ones.}
\label{tab:same-file}
\end{table}

\subsection{Per-tokenizer results}
\label{sec:per-tokenizer}
The main paper reports most measurements for Qwen3-8B, Llama-3.1-8B, gpt-oss-120b, and DeepSeek-V3. This part breaks the checks and measurements down by tokenizer and adds the checks that cover the other seven. As in the main paper, a check compares a complete token-ID sequence with the Hugging Face reference, and a single differing ID counts as a difference.

\paragraph{Exactness.}
Table~\ref{tab:tok-exact} splits the exactness checks of the main paper by tokenizer. Under each of the eleven admitted tokenizers, including the five that the table has no column for, the fast encoder's whole encodes of the full web corpus, 3,800,016,491 documents and 12.3 trillion characters, match the reference. Six tokenizers also have checks of the chunked path, and none of these found a difference either. Among them, the adversarial battery and the tests of the Rust prototype are the only checks that cover Qwen3.5-0.8B and MiniMax-M3.
The battery's control reruns the same strings without the normalization check, and it differs from the reference only under Qwen3-8B. Qwen3.5-0.8B and MiniMax-M3 also normalize to NFC, but their splitting patterns keep combining marks inside a word, so the unchecked cuts did not change their output on these strings. The normalization check still rejected 9,891 and 9,825 candidate cuts for them.

\begin{table}[htbp]
\centering
\footnotesize
\setlength{\tabcolsep}{4pt}
\begin{tabular}{@{}lrrrrrr@{}}
\toprule
Check & Qwen3-8B & \shortstack[r]{Llama-\\3.1-8B} & \shortstack[r]{gpt-oss-\\120b} & \shortstack[r]{DeepSeek-\\V3} & \shortstack[r]{Qwen3.5-\\0.8B} & \shortstack[r]{MiniMax-\\M3} \\
\midrule
Fast encoder, whole encodes, documents & 3.8B & 3.8B & 3.8B & 3.8B & 3.8B & 3.8B \\
Web corpus split at \sys's cuts, documents & 102.4M & 102.4M & 102.4M & 102.4M & -- & -- \\
Chunking sweep, Zen 5 server & 16,050 & 6,560 & 6,560 & -- & -- & -- \\
Chunking sweep, two desktops & $2\times5{,}248$ & $2\times5{,}248$ & -- & -- & -- & -- \\
Serial encoder, two desktops & 2,645 & 2,645 & 2,645 & 2,645 & -- & -- \\
Special-token cuts, chat prompts & 610 & -- & -- & -- & -- & -- \\
Turn-by-turn replay, 610 chat transcripts & 323,680 & -- & -- & -- & -- & -- \\
Rust prototype, cuts and IDs & 35,912 & 36,000 & 24,000 & 24,000 & 24,000 & 24,000 \\
Adversarial battery & 100,000 & 20,000 & 20,000 & 20,000 & 30,000 & 30,000 \\
\quad control without the normalization check, differ & 2,115 & 0 & 0 & 0 & 0 & 0 \\
Scheduling testbed, 638 runs & 405,268 & -- & -- & -- & -- & -- \\
Scheduling testbed, desktop & 43,408 & -- & -- & -- & -- & -- \\
Capacity study, 128-core Zen 3 node & 429,457 & -- & -- & -- & -- & -- \\
SGLang, shadow re-encodes & 2,195 & -- & 59 & -- & -- & -- \\
SGLang, served long documents, shadows & 2,268 & -- & -- & -- & -- & -- \\
SGLang, greedy completions & 550 & -- & -- & -- & -- & -- \\
SGLang, reported prompt tokens & every run & -- & every run & -- & -- & -- \\
\bottomrule
\end{tabular}
\caption{Checks against the reference, by tokenizer. Each comparison covers a complete token-ID sequence, except the prompt-token counts that SGLang reports and the greedy completions, which compare the model's output. Apart from the control, none finds a difference. The control row counts the strings that differ once the normalization check is removed. The sweep rows count every output assembled at \sys's cuts, whichever encoder encoded the chunks (\sys's fast encoder, the reference, or fastokens), and the whole encodes of \sys's fast encoder. The serial-encoder row counts \sys's fast encoder called in three ways, through the native runtime with and without building the index for content lookup (Section~\ref{sec:ev-session}) and through the Python tokenizer interface. The gpt-oss-120b shadow re-encodes are not part of the main paper's count of 2,195. The row for served long documents includes runs with Qwen3-1.7B, which uses the same tokenizer files as Qwen3-8B, and twelve runs on a GH200 node with an ARM build of the fast encoder (Section~\ref{sec:validation}).}
\label{tab:tok-exact}
\end{table}

\begin{table}[htbp]
\centering
\small
\begin{tabular}{@{}lrr@{}}
\toprule
Check & Compared & Differ \\
\midrule
Chunking sweep, Zen 5 server & 29{,}170 & 0 \\
Chunking sweep and serial encoder, two Zen 5 desktops & 31{,}572 & 0 \\
Rust prototype, cuts and IDs & 167{,}912 & 0 \\
Scheduling testbed, server and desktop & 448{,}676 & 0 \\
Capacity study, 128-core Zen 3 node, 215 runs & 429{,}457 & 0 \\
\midrule
Total & 1{,}106{,}787 & 0 \\
\bottomrule
\end{tabular}
\caption{The rows that Table~3 of the main paper combines into its row of platform and runtime repeats. Each comparison covers a complete token-ID sequence, and none differs.}
\label{tab:platform-exact}
\end{table}

On the web corpus, each document, of 512 to 400,000 characters, was split into $\max(2,\lceil n/8{,}000\rceil)$ chunks, two up to 16,000 characters and about one per 8,000 characters beyond. Candidate cuts that fell inside a special-token string were rejected 250 times under Qwen3-8B and 162 times under Llama-3.1-8B, while gpt-oss-120b and DeepSeek-V3 never met this case. The normalization check rejected four candidates, all under Qwen3-8B. No acceptable cut existed in 36 documents under Qwen3-8B and in 35 under each of the others, and these documents were encoded whole. The guards are thus exercised on real text, not only in the battery.

\paragraph{Configuration checks for all eleven tokenizers.}
Table~\ref{tab:tok-config} gives two further checks for all eleven tokenizers. The first tests, for each tokenizer, the transition rule that Section~\ref{sec:cutproof} argues for. It builds 57,728 short strings, each with one of the eight class changes of the transition set between two characters and placed between one of 16 left contexts and one of 11 right contexts, and asks the reference pre-tokenizer whether a new piece starts at the change. Every string in which no new piece starts falls under one of the two exclusions of the rule, a line break after punctuation or an apostrophe after a letter.

The second check applies the repair's acceptance rule of Section~\ref{sec:splice} to the seams between overlapping chunks. It generates random documents of up to 5,000 characters that mix ASCII text, code, CJK and Hangul characters, emoji, accented letters, runs of whitespace, repeated fragments, and special-token strings. It cuts each document into chunks of 1 to 512 characters that overlap by 1 to 512 characters and joins the chunks at the seams that the rule accepts. A document with a seam that is still not accepted after five retries is encoded whole. For the seven tokenizers other than the main paper's four, every one of 3,000 documents per tokenizer matches the reference, and 72 to 126 of them per tokenizer are encoded whole.

Five admitted tokenizers, DeepSeek-V4-Flash, GLM-5.2, OLMo-3-7B, Ministral-3-8B, and Nemotron-3-Nano-4B, have the fast encoder's whole-encode comparison on 3.8 billion documents and the two checks above, but no check of the chunked path itself. Three of them share parts of their files with a tokenizer whose chunked path was checked. DeepSeek-V4-Flash keeps all 818 special-token strings of DeepSeek-V3 at the same IDs and adds 465 more, and every other part of its file equals that of DeepSeek-V3. GLM-5.2 and OLMo-3-7B split text with the regular expression of Llama-3.1-8B and have no normalizer, but their vocabularies are their own. The remaining two are related only to each other. Nemotron-3-Nano-4B has the merges, pattern, and vocabulary of Ministral-3-8B except for 14 special-token strings, and no other admitted tokenizer uses this pattern.

\begin{table}[htbp]
\centering
\footnotesize
\setlength{\tabcolsep}{4pt}
\begin{tabular}{@{}lllrrrrr@{}}
\toprule
& & & \multicolumn{2}{c}{Transition probe, unsplit} & \multicolumn{3}{c}{Seam test} \\
\cmidrule(lr){4-5}\cmidrule(l){6-8}
Tokenizer & Pattern & Normalizer & at exclusions & elsewhere & documents & wrong & whole \\
\midrule
Qwen3-8B & A & NFC & 1,760 & 0 & -- & -- & -- \\
Llama-3.1-8B & B & none & 1,760 & 0 & -- & -- & -- \\
gpt-oss-120b & C & none & 2,208 & 0 & -- & -- & -- \\
DeepSeek-V3 & D & empty list & 1,760 & 0 & -- & -- & -- \\
Qwen3.5-0.8B & E & NFC & 1,760 & 0 & 3,000 & 0 & 94 \\
MiniMax-M3 & C & NFC & 2,208 & 0 & 3,000 & 0 & 91 \\
DeepSeek-V4-Flash & D & empty list & 1,760 & 0 & 3,000 & 0 & 93 \\
GLM-5.2 & B & none & 1,760 & 0 & 3,000 & 0 & 126 \\
OLMo-3-7B & B & none & 1,760 & 0 & 3,000 & 0 & 89 \\
Ministral-3-8B & F & none & 1,760 & 0 & 3,000 & 0 & 81 \\
Nemotron-3-Nano-4B & F & none & 1,760 & 0 & 3,000 & 0 & 72 \\
\bottomrule
\end{tabular}
\caption{Configuration checks for all eleven tokenizers. Tokenizers with the same pattern letter split text with the same regular expressions. The probe has 57,728 strings per tokenizer, and a string is unsplit when the reference does not start a new piece at the class change it probes. The seam test was run for the seven tokenizers other than the main paper's four; whole counts documents encoded whole after the retries.}
\label{tab:tok-config}
\end{table}

\paragraph{Single-core speed.}
Table~\ref{tab:tok-floor} gives the single-core speed of the fast encoder for the four tokenizers of the main paper on English web documents of one million characters. The fast encoder is 30 to 34 times faster than the reference under the three tokenizers that do not normalize text and about 16 times faster under Qwen3-8B, the one of the four with NFC. Desktop B is slightly slower than desktop A, and the speedups on the two agree within 3\%.

\begin{table}[htbp]
\centering
\small
\begin{tabular}{@{}llrrrrrr@{}}
\toprule
& & \multicolumn{3}{c}{Desktop A} & \multicolumn{3}{c}{Desktop B} \\
\cmidrule(lr){3-5}\cmidrule(l){6-8}
Tokenizer & Normalizer & Reference & Fast encoder & Speedup & Reference & Fast encoder & Speedup \\
\midrule
Qwen3-8B & NFC & 3.3 & 54.4 & 16.6$\times$ & 3.1 & 51.2 & 16.4$\times$ \\
Llama-3.1-8B & none & 4.1 & 126.3 & 30.7$\times$ & 3.9 & 118.6 & 30.3$\times$ \\
gpt-oss-120b & none & 4.1 & 131.6 & 32.3$\times$ & 3.8 & 122.6 & 32.4$\times$ \\
DeepSeek-V3 & empty list & 3.5 & 113.6 & 32.9$\times$ & 3.2 & 109.9 & 33.8$\times$ \\
\bottomrule
\end{tabular}
\caption{Single-core throughput in MB/s of UTF-8 text on English web documents of one million characters, measured on two identical Zen 5 desktops. Each value is the median over nine of ten documents, leaving out the first document that the process encodes. Speedups divide the unrounded medians.}
\label{tab:tok-floor}
\end{table}

\paragraph{Chunk-parallel full encodes.}
Table~\ref{tab:tok-chunks} shows full encodes split into chunks on sixteen cores of the Zen 5 server for the three tokenizers of the chunking sweep. With sixteen chunks, Llama-3.1-8B and gpt-oss-120b run 2.9 to 4.5 times faster than their whole encode on one core, and Qwen3-8B runs 3.7 to 6.4 times faster. Two effects favor Qwen3-8B. Its encode costs about three times as much per character, 37 against about 13 milliseconds per million characters in the cost model of Section~\ref{sec:measure} fitted to web text, so the fixed cost of each chunk weighs less. A character that makes the fast encoder pass work to the reference also costs one chunk instead of the whole document. With 32 chunks Qwen3-8B reaches 5.0 to 8.3 times. The sweep ran Llama-3.1-8B and gpt-oss-120b with at most sixteen chunks.

\begin{table}[htbp]
\centering
\footnotesize
\setlength{\tabcolsep}{4pt}
\begin{tabular}{@{}lrrrrrrr@{}}
\toprule
& \multicolumn{2}{c}{One core} & \multicolumn{3}{c}{Sixteen cores} & \multicolumn{2}{c}{Best speedup over} \\
\cmidrule(lr){2-3}\cmidrule(lr){4-6}\cmidrule(l){7-8}
Input & Reference & Whole & 1 chunk & 16 chunks & 32 chunks & whole & 1 chunk \\
\midrule
\multicolumn{8}{@{}l}{\emph{Qwen3-8B}} \\
\quad web, 1M & 374.37 & 27.21 & 30.81 & 7.37 & 5.48 & 5.0$\times$ & 5.6$\times$ \\
\quad web, 4M & 1,617.90 & 150.06 & 170.28 & 23.62 & 18.31 & 8.2$\times$ & 9.3$\times$ \\
\quad code, 1M & 435.35 & 30.96 & 33.10 & 7.56 & 5.56 & 5.6$\times$ & 6.0$\times$ \\
\quad code, 4M & 1,866.23 & 149.47 & 154.59 & 23.22 & 19.30 & 7.7$\times$ & 8.0$\times$ \\
\quad agent, 1M & 437.32 & 26.84 & 28.75 & 5.74 & 4.29 & 6.3$\times$ & 6.7$\times$ \\
\quad agent, 3M & 1,331.20 & 86.56 & 104.58 & 15.80 & 10.91 & 7.9$\times$ & 9.6$\times$ \\
\quad multi, 1M & 529.54 & 62.53 & 65.23 & 15.04 & 11.37 & 5.5$\times$ & 5.7$\times$ \\
\quad multi, 4M & 2,805.52 & 440.28 & 452.57 & 68.72 & 52.95 & 8.3$\times$ & 8.5$\times$ \\
\multicolumn{8}{@{}l}{\emph{Llama-3.1-8B}} \\
\quad web, 1M & 302.52 & 9.81 & 11.20 & 2.55 & -- & 3.9$\times$ & 4.4$\times$ \\
\quad web, 4M & 1,458.07 & 35.46 & 42.19 & 11.80 & -- & 3.0$\times$ & 3.6$\times$ \\
\quad code, 1M & 343.53 & 16.08 & 17.62 & 4.04 & -- & 4.0$\times$ & 4.4$\times$ \\
\quad code, 4M & 1,448.57 & 62.31 & 67.80 & 15.75 & -- & 4.0$\times$ & 4.3$\times$ \\
\quad agent, 1M & 333.67 & 11.50 & 13.05 & 3.21 & -- & 3.6$\times$ & 4.1$\times$ \\
\quad agent, 3M & 1,036.46 & 32.87 & 38.50 & 9.36 & -- & 3.5$\times$ & 4.1$\times$ \\
\quad multi, 1M & 446.76 & 39.48 & 42.33 & 12.86 & -- & 3.1$\times$ & 3.3$\times$ \\
\quad multi, 4M & 2,374.46 & 231.69 & 250.62 & 51.71 & -- & 4.5$\times$ & 4.8$\times$ \\
\multicolumn{8}{@{}l}{\emph{gpt-oss-120b}} \\
\quad web, 1M & 307.66 & 10.10 & 11.46 & 3.28 & -- & 3.1$\times$ & 3.5$\times$ \\
\quad web, 4M & 1,287.98 & 37.26 & 44.04 & 12.68 & -- & 2.9$\times$ & 3.5$\times$ \\
\quad code, 1M & 341.32 & 14.43 & 15.66 & 4.13 & -- & 3.5$\times$ & 3.8$\times$ \\
\quad code, 4M & 1,534.16 & 54.66 & 62.32 & 14.64 & -- & 3.7$\times$ & 4.3$\times$ \\
\quad agent, 1M & 326.50 & 10.37 & 11.97 & 3.38 & -- & 3.1$\times$ & 3.5$\times$ \\
\quad agent, 3M & 1,006.47 & 30.89 & 36.99 & 9.22 & -- & 3.3$\times$ & 4.0$\times$ \\
\quad multi, 1M & 460.01 & 39.85 & 42.84 & 12.35 & -- & 3.2$\times$ & 3.5$\times$ \\
\quad multi, 4M & 2,415.55 & 219.66 & 229.54 & 49.26 & -- & 4.5$\times$ & 4.7$\times$ \\
\bottomrule
\end{tabular}
\caption{Median encode time in milliseconds on sixteen cores of the Zen 5 server, over 20 documents per setting, each timed at its first encode in the process. Reference and Whole encode the document on one core with the reference and with the fast encoder. The 1-chunk column uses the worker pool with a single chunk, which adds the cost of handing the document to a worker. The speedups use the best chunk count, 32 for Qwen3-8B and 16 for the others. Inputs are English web text, source code, text from public coding-agent transcripts, and multilingual text; sizes are in characters.}
\label{tab:tok-chunks}
\end{table}

\paragraph{Documents that the fast encoder hands to the reference.}
Table~\ref{tab:tok-guard} counts the documents that the fast encoder, encoding whole on one core, passed to the reference in the same sweep. Under Qwen3-8B this happens on every kind of input and more often for longer documents, up to 7 of 10 web documents of ten million characters, mostly because of the normalization guard. Under Llama-3.1-8B it happens only on multilingual text that contains a special-token string, and under gpt-oss-120b it never happens. A document passed to the reference takes six to twelve times as long as the other documents of its setting.

\begin{table}[htbp]
\centering
\footnotesize
\setlength{\tabcolsep}{4pt}
\begin{tabular}{@{}llrrrrrr@{}}
\toprule
& & & \multicolumn{3}{c}{Passed to the reference} & \multicolumn{2}{c}{Median time} \\
\cmidrule(lr){4-6}\cmidrule(l){7-8}
Tokenizer & Input & Documents & all & normalization & special string & passed & others \\
\midrule
Qwen3-8B & web, 64K & 20 & 1 & 1 & 0 & 17.39 & 1.64 \\
Qwen3-8B & web, 256K & 20 & 1 & 1 & 0 & 77.36 & 6.60 \\
Qwen3-8B & web, 1M & 20 & 4 & 4 & 0 & 286.32 & 25.41 \\
Qwen3-8B & web, 4M & 20 & 7 & 7 & 0 & 1202.29 & 129.26 \\
Qwen3-8B & web, 10M & 10 & 7 & 7 & 0 & 2881.65 & 397.84 \\
Qwen3-8B & code, 4M & 20 & 1 & 1 & 0 & 1328.83 & 148.42 \\
Qwen3-8B & code, 10M & 10 & 2 & 2 & 0 & 3243.87 & 378.99 \\
Qwen3-8B & agent, 3M & 20 & 1 & 1 & 0 & 718.60 & 85.64 \\
Qwen3-8B & multi, 256K & 20 & 1 & 1 & 0 & 152.38 & 18.13 \\
Qwen3-8B & multi, 1M & 20 & 1 & 1 & 0 & 445.24 & 61.11 \\
Qwen3-8B & multi, 4M & 20 & 7 & 2 & 5 & 1971.83 & 347.13 \\
Llama-3.1-8B & multi, 1M & 20 & 1 & 0 & 1 & 306.29 & 37.18 \\
Llama-3.1-8B & multi, 4M & 20 & 5 & 0 & 5 & 1567.26 & 215.93 \\
\bottomrule
\end{tabular}
\caption{Documents that the fast encoder, encoding whole on one core of the Zen 5 server, passed to the reference, by the guard that triggered, with median times in milliseconds. Settings without such documents are omitted, including every gpt-oss-120b setting.}
\label{tab:tok-guard}
\end{table}

\paragraph{Serving in SGLang.}
\sys served two models in SGLang, Qwen3-8B and gpt-oss-120b. Table~\ref{tab:tok-sglang} gives the time to first token for one session that grows by 400 characters per turn, over 30 warm turns per server process, leaving out the first warm turn of each process. The gain appears once the context reaches tens of thousands of tokens and grows with it, to 61\% for Qwen3-8B at 100K tokens and 45\% for gpt-oss-120b at about 90K tokens. The RTX 5090 desktop runs Qwen3-8B at its native context of 40,960 tokens, so it covers 8K and 32K tokens, and it shows the same pattern.

Exactness in SGLang rests on different evidence for the two models. For Qwen3-8B, the adapter encoded 2,195 served requests a second time with the reference and compared the IDs, and 550 greedy completions were compared with those produced with the reference tokenizer on the same host. Neither comparison found a difference. On gpt-oss-120b greedy decoding is not reproducible on this serving stack, since the same prompt sent twice produced a different completion in one of three repeats, so the check there uses token IDs only. Its 59 synchronous shadow re-encodes equal the reference, and so do the prompt-token counts that SGLang reports in every run.

\begin{table}[htbp]
\centering
\footnotesize
\setlength{\tabcolsep}{3pt}
\begin{tabular}{@{}llrrrrr@{}}
\toprule
& & \multicolumn{3}{c}{Median} & \multicolumn{2}{c}{99th percentile} \\
\cmidrule(lr){3-5}\cmidrule(l){6-7}
Model and host & Context & Reference & \sys & change & Reference & \sys \\
\midrule
Qwen3-8B, Zen 5 server (RTX PRO 6000) & 8K tokens & 29.13 & 29.18 & $+0.2$\% & 33.55 & 33.49 \\
 & 16K tokens & 35.48 & 30.84 & $-13.1$\% & 38.85 & 33.53 \\
 & 32K tokens & 55.67 & 37.52 & $-32.6$\% & 63.58 & 42.98 \\
 & 64K tokens & 94.85 & 48.02 & $-49.4$\% & 104.25 & 55.69 \\
 & 100K tokens & 136.83 & 53.72 & $-60.7$\% & 158.49 & 74.97 \\
\addlinespace
Qwen3-8B, Zen 5 desktop (RTX 5090) & 8K tokens & 27.30 & 27.22 & $-0.3$\% & 32.55 & 32.44 \\
 & 32K tokens & 50.20 & 35.99 & $-28.3$\% & 80.19 & 44.15 \\
\addlinespace
gpt-oss-120b, Zen 5 server (RTX PRO 6000) & 25,641 tokens & 61.37 & 46.61 & $-24.1$\% & 68.60 & 54.69 \\
 & 64,103 tokens & 116.18 & 66.95 & $-42.4$\% & 123.16 & 74.27 \\
 & 89,744 tokens & 152.85 & 83.62 & $-45.3$\% & 263.33 & 192.22 \\
\bottomrule
\end{tabular}
\caption{Time to first token in milliseconds in SGLang for one session that grows by 400 characters per turn, with the reference tokenizer and with \sys. Qwen3-8B ran with three server processes per tokenizer on the Zen 5 server and two on the Zen 5 desktop, and gpt-oss-120b with two on the Zen 5 server. The gpt-oss-120b contexts are 100K, 250K, and 350K characters.}
\label{tab:tok-sglang}
\end{table}

\section{The tokenizer contract}
\label{sec:contract}
Let $S$ be the exact text passed to the tokenizer after request rendering. Let $G(S)$ be its ordered sequence of units after special-token matching, normalization of the text between special-token strings, and pre-tokenization. A unit is either a special-token string (an added token in Hugging Face's terms) or a normalized piece. Let $E(u)$ be the deterministic sequence of IDs produced for unit $u$. With sequence-level post-processing disabled, the reference output is $F(S)=E^*(G(S))$, where $E^*$ applies $E$ to each unit and concatenates its outputs.

Writing $F$ this way shows what the tokenizer has to satisfy. BPE must never merge across two pieces, its merge rules must be fixed, and BPE dropout must be off, so that $E(u)$ depends on $u$ alone. The flags that control how special-token strings are matched, and the normalizer, are part of the configuration that admission checks (Section~\ref{sec:tokenizers}), and IDs computed under one configuration are never reused under another. Sequence-level post-processing runs once, on the complete assembled ID sequence, as the reference runs it once on the whole text. Calls with paired inputs, truncation modes, or other API options that this contract does not cover go to the reference tokenizer.

For each token the repair keeps an entry with three fields, the token ID, its span in the source text in request coordinates, and an ordinal that numbers the tokens sharing one span. The ordinal is needed because several byte-level tokens can report the same character span when BPE divides the bytes of one character among them (Section~\ref{sec:lopt}). Spans serve to line up the cached and the new sequence and need not partition the text. A match compares all three fields, and the splice is made at token indices, not at character positions.

\section{Incremental splice argument}
\label{sec:splice}
Let $A$ be the previously encoded text and $B$ an append. A repair encodes the window $W=(A\cat B)[s{:}]$ for some start $s$ inside $A$. The cached encoding of $A$ is exact. The repair must find a region where the cached encoding of $A$, the encoding of $W$, and the encoding of the whole of $A\cat B$ all agree. Equal IDs in the cached and window encodings do not show this by themselves.

A match is a run of consecutive entries with equal IDs, equal spans, and equal ordinals in the cached and the window encoding. The search skips a fixed margin at the left end of $W$ (Table~\ref{tab:tokenizers}). For the admitted tokenizers, \sys accepts a run only if it has at least two tokens, covers more normalized characters than the tokenizer's bound (Table~\ref{tab:tokenizers}), and contains a stable pre-tokenizer transition (Section~\ref{sec:cutproof}). These conditions are sufficient for the admitted tokenizers, not for BPE in general. The length condition on its own does not make a splice exact, and Section~\ref{sec:lopt} shows LoPT's version of it accepting a wrong seam. The argument below rests on the transition, and the length condition only adds margin.

\begin{proposition}[Splice through an agreed region]
Suppose a matching run contains a boundary $p$ for which the checks of an admitted tokenizer establish two facts. First, the cached units of $A$ before $p$ are unchanged by the append. Second, the units of $W$ after $p$ are the units of $A\cat B$ after $p$. If the cached and the window entries agree from $p$ through a later entry $r$ of the run, let $i$ and $j$ be the positions just after $r$ in the cached sequence $T_{\mathrm{old}}$ and in the window sequence $T_{\mathrm{win}}$. Then $T_{\mathrm{old}}[0{:}i]\cat T_{\mathrm{win}}[j{:}]=F(A\cat B)$.
\end{proposition}
\begin{proof}
Before $p$, the cached units are the whole-input units by the first fact. After $p$, encoding each unit of the window gives the whole-input IDs by the second fact, because $E$ is deterministic. The entries from $p$ through $r$ have the same IDs in both encodings, so taking them from the cache instead of the window changes nothing. The cached prefix through $r$ followed by the window suffix after $r$ is therefore the whole-input output. The argument works with token indices and does not need the character spans to be disjoint.
\end{proof}

The proposition assumes a boundary $p$ and does not say how to find one. For the admitted tokenizers, \sys finds it with the stable class transitions of Section~\ref{sec:cutproof}, the left margin, and three checks on the window before the IDs from the fast encoder are used. The window must be valid UTF-8, the active normalizer must leave it unchanged, and the bytes decoded from its token IDs must reproduce it exactly. A window that the normalizer would change is never spliced, even where its entries would still match, because normalization moves the entries relative to the raw text. If the window fails a check, or the fast encoder's guard fires, the reference encodes the text after the session's sealed boundary, the retained tail of $A$ together with $B$, and the sealed prefix keeps its IDs. Special-token strings inside the window do not trigger this fallback, because the window's pipeline matches them before the fast path, as a full encode by the reference would.

The first window holds the last 512 characters of $A$. If no run qualifies, \sys doubles the window, at most five times in the evaluated configuration. When the retries are used up, or the window already covers the whole retained tail, the reference encodes the retained tail and $B$ in the same way. The fast encoder's guard can also ask for this at once, without further retries. The session seals its IDs only up to a stable transition, so the pieces before the sealed boundary do not depend on the text after it, and the sealed prefix followed by this encode is the whole-input output. Exactness therefore depends on neither the starting size nor the retry limit. They only decide how much local work \sys tries before it hands the retained tail to the reference.

The serving adapter reuses a session record only when its text $A$ is an exact prefix of the new request. A session identifier that stays the same is not enough, because a chat template can rewrite earlier turns. Edits in the middle of the history are not appends, so \sys encodes such requests in full, and the measurements do not include them. The tokenizer's session records and the model's KV cache are evicted independently, so a request whose session record is gone is re-encoded even when the model still caches its prefix.

\section{Independent cuts for full encoding}
\label{sec:cutproof}
\subsection{Pre-tokenizer transitions}
\label{sec:transitions}
The admitted tokenizers use six distinct splitting patterns (Table~\ref{tab:tok-config}). Each matches runs of letters, groups of digits, runs of punctuation, and whitespace, and they differ in the details. \sys assigns character classes from a table that matches the reference tokenizer's classes. Let $\mathrm{L}$ denote a letter, $\mathrm{N}$ a number, $\mathrm{W}$ whitespace, $\mathrm{M}$ a mark, and $\mathrm{O}$ the remaining class. The candidate transition set is
\[
\begin{split}
\mathcal S=\{&\mathrm{L}{\to}\mathrm{N},\ \mathrm{L}{\to}\mathrm{O},\ \mathrm{L}{\to}\mathrm{W},\\
            &\mathrm{N}{\to}\mathrm{L},\ \mathrm{N}{\to}\mathrm{O},\ \mathrm{N}{\to}\mathrm{W},\\
            &\mathrm{O}{\to}\mathrm{N},\ \mathrm{O}{\to}\mathrm{W}\}.
\end{split}
\]
The $\mathrm{O}{\to}\mathrm{W}$ case is excluded when the right character is CR or LF, and $\mathrm{L}{\to}\mathrm{O}$ is excluded when that character is an apostrophe. The set has no transition out of whitespace, because a pattern can attach the last space of a whitespace run to the following word, so how the run is split depends on the character after it.

For each of the six patterns, an accepted transition in normalized text that contains no special-token string always ends one piece and starts the next, whatever text precedes it. A cut at the transition also makes it the end of the left chunk. The run of letters, digits, or punctuation before the transition then has to close at the end of the chunk exactly as it closes at the class change, and for the six patterns it does. Independent chunks need this end-of-input property. The repair of Section~\ref{sec:splice} does not, because its window always extends past the transition.

\subsection{Special-token strings}
A cut must not fall inside any occurrence of any special-token string, including occurrences that the whole-text scan does not select. The scan takes the leftmost, longest match, so of two overlapping occurrences only one is selected. A cut inside the other one could still change which of the two a chunk's scan selects, so the guard checks every occurrence. \sys admits a tokenizer only if none of its special tokens strips neighboring whitespace, matches only as a single word, or is matched after normalization, and it reads these flags from the loaded configuration (Section~\ref{sec:tokenizers}).

If no occurrence crosses the cut, no partial match is open at the cut, so each chunk contains the special-token strings that the whole-text scan finds in that part of the text. The cut then lies in ordinary text between two special-token strings, and only that stretch is split further for normalization and pre-tokenization.

\subsection{Normalization anchors}
For an NFC configuration, a raw class change is not enough. One counterexample is a raw letter U+0958 followed by a slash. NFC decomposes the letter to U+0915 and the mark U+093C. The raw letter-to-other transition becomes mark-to-other in the normalized text and does not satisfy the allowed set. The control row of Table~3 of the main paper, which removes the normalization check, exposes this class of error.

For a proposed cut $b$, \sys looks for two anchors $a<b<c$ within 256 code points on each side of $b$, positions at which the normalizer is known to split the text (the predicate is given below). The ends of the text also count as anchors. The cut is rejected if either anchor is missing or if any occurrence of a special-token string touches $S[a{:}c]$. Otherwise the normalizer $N$ must satisfy
\[
 N(S[a{:}c])=N(S[a{:}b])\cat N(S[b{:}c]),
\]
and the last character of the normalized left side and the first character of the normalized right side must form an accepted transition. The candidate $b$ itself is still chosen by an accepted transition in the raw text, as in Section~\ref{sec:transitions}.

A code point can serve as an anchor if it is not a mark, normalizing it alone leaves it unchanged, placing the probe marks U+0301 or U+0335 before or after it does not reorder it, and it does not compose with a preceding Hangul probe U+1100 or U+AC00. ASCII characters take a fast path. At run time the predicate is evaluated with the installed reference normalizer, and an offline test over every code point compares it with the Unicode tables of the reference normalizer's library. The argument relies on those tables and on the reordering and composition cases of NFC that they allow, so a different normalizer or Unicode version needs the same test.

Because the normalizer splits the text at the anchors $a$ and $c$, the equation above, which holds between them, shows that it also splits the text at $b$. The transition is then rechecked on the normalized characters at $b$, since a class change that exists in the raw text may have disappeared after normalization, as with U+0958 above. If any of these conditions fails, \sys makes no cut at $b$.

\subsection{Chunk argument}
\begin{proposition}[Concatenation at accepted cuts]
For a frozen, admitted configuration, let $S_1,\ldots,S_k$ partition $S$ at positions satisfying the special-token, normalization, and pre-tokenizer conditions above. Then
\[
 F(S)=F(S_1)\cat\cdots\cat F(S_k).
\]
\end{proposition}
\begin{proof}
The special-token condition keeps the special-token scan the same on each chunk. The normalization condition keeps the text that reaches the pre-tokenizer the same, including the stretch that a cut divides. The class-transition condition, with its end-of-input case, keeps the ordered pre-tokenizer pieces the same. So $G(S)$ is the concatenation of the chunk unit sequences $G(S_i)$. Since $E$ encodes each unit independently and deterministically, $E^*$ preserves this concatenation. Post-processing is then applied once to the assembled sequence, as the reference applies it once to the whole text.
\end{proof}

A bounded backward search places cuts near evenly spaced targets, and Section~\ref{sec:ev-sched} gives the scan budget and the minimum chunk size. A missed target is absorbed by a larger neighboring chunk, and if every target is missed the text is encoded whole. The proposition holds for any partition at accepted cuts, so it does not need a maximum gap between safe cuts, equal chunk sizes, or a particular completion order.

\section{LoPT's seam matching as published}
\label{sec:lopt}
LoPT~\cite{lopt2026} is the most recent parallel tokenizer that cuts a long input without regard to the tokenizer and repairs the seams afterwards. The main paper states that its published algorithm is not exact. This section gives the evidence, takes one failure apart, and compares the speed of the two approaches.

\subsection{Reimplementation}
The code of LoPT is not public, so we reimplemented the algorithm from its paper~\cite{lopt2026}. The input is cut into 32 chunks of equal length, and each chunk reaches into the next one by an overlap. The chunks are encoded in parallel, and the outputs of neighboring chunks are joined where their tokens in the overlap coincide. Two tokens coincide when their character spans in the original text are equal, a choice the paper makes to avoid ``expensive token ID comparisons''. A seam is accepted when at least two consecutive tokens coincide. Otherwise the chunk length is doubled and the input is encoded again, at most five times, after which the input is encoded serially. The paper does not state the overlap length, and we use 512 characters, the default of the earlier ParallelTokenizer. Theorem~3.1 of the paper states that the joined output equals serial tokenization whenever every matched run of tokens spans more characters than the longest token of the vocabulary.

The reimplementation reproduces closely the numbers of LoPT that do not depend on hardware, its accuracy and its retry counts. On LongBench~v2~\cite{longbenchv2} (the contexts of its 503 questions) and L-Eval~\cite{leval2024} (the 537 documents of its Hugging Face release), whose texts are English, it returns the serial IDs for every document under all six tokenizers of LoPT's evaluation. The BPE tokenizers retry on one of the 503 LongBench~v2 documents, where LoPT reports none, and on repetitive inputs the BERT tokenizers exhaust all five retries and fall back to serial encoding, as in LoPT's Table~3.

\subsection{Wrong IDs on CLongEval}
The published rule fails on CLongEval~\cite{clongeval2024}, the Chinese long-text benchmark of LoPT's own evaluation. Table~\ref{tab:lopt-clongeval} counts, for each tokenizer and subset, the samples whose joined IDs differ from serial encoding with the same Hugging Face tokenizer. Under Qwen3, Llama-3.1, and gpt-oss, 30 of the 21,789 encodes are wrong, at most 0.31\% of a subset, and under DeepSeek-V3 and the two BERT tokenizers none is. CLongEval pairs several questions with one context, so the 30 wrong encodes come from 22 distinct pairs of tokenizer and text. Every wrong output was accepted at the first attempt, without a retry. When the matching compares token IDs as well as positions, every sample of the three affected tokenizers returns the serial IDs. LoPT's paper reports an accuracy of 1.0 on this benchmark for all six tokenizers, so its own implementation may compare more than the positions that the paper describes.

\begin{table}[h]
\centering
\begin{tabular}{lrrrr}
\toprule
Tokenizer & Small & Medium & Large & With ID check \\
\midrule
Qwen3 & 8 & 1 & 0 & 0 \\
Llama-3.1 & 8 & 2 & 0 & 0 \\
gpt-oss & 4 & 5 & 2 & 0 \\
DeepSeek-V3 & 0 & 0 & 0 & -- \\
BERT, cased & 0 & 0 & 0 & -- \\
BERT, uncased & 0 & 0 & 0 & -- \\
\midrule
Samples per tokenizer & 2,605 & 2,653 & 2,005 & -- \\
\bottomrule
\end{tabular}
\caption{CLongEval samples whose IDs from LoPT's published rule differ from serial encoding. The last column reruns all three subsets for the three tokenizers with differences, with token IDs compared as well as positions; a dash marks a tokenizer that was not rerun because its published rule produced no difference.}
\label{tab:lopt-clongeval}
\end{table}

\subsection{How position matching goes wrong}
The errors come from characters whose bytes byte-level BPE divides among tokens. Each of the Chinese characters in this example takes three bytes in UTF-8, and depending on its neighbors BPE may keep them in one token or give the last byte to the next token. The tokenizer reports offsets in characters, so every token that holds some of a character's bytes reports that character's span. Two chunks that divide the bytes of one character differently therefore produce tokens with equal spans and different IDs, and matching by position alone treats them as one token.

Table~\ref{tab:lopt-example} shows the seam of one failure, in a 52,232-character document of the medium subset under Qwen3. The character U+547D, whose UTF-8 bytes are 229, 145, and 189, is the last character of the left chunk, so the left chunk encodes it as one three-byte token. Serial encoding and the right chunk see the next character, U+6570, and attach the byte 189 to it, which leaves a token for the first two bytes of U+547D. The three-byte token and the two-byte token report the same span, so the matched run extends to the end of the left chunk. The join then takes the three-byte token from the left chunk and continues with the right chunk's token that starts with byte 189 again. The joined sequence has the serial length and differs from it at one position, and decoding it yields a stray replacement character between the two characters.

\begin{table}[h]
\centering
\begin{tabular}{rrrllc}
\toprule
Index & Serial ID & LoPT ID & Serial bytes & LoPT bytes & Character span \\
\midrule
11,133 & 53930 & 53930 & U+4E4B & U+4E4B & 13,574--13,575 \\
11,134 & 17401 & 50509 & 229 145 & 229 145 189 & 13,575--13,576 \\
11,135 & 29941 & 29941 & 189, U+6570 & 189, U+6570 & 13,575--13,577 \\
\bottomrule
\end{tabular}
\caption{The seam of one wrong output under Qwen3. Index 11,134 carries different tokens with the same span, so position matching accepts the seam, and the joined output contains byte 189 twice.}
\label{tab:lopt-example}
\end{table}

The matched runs in these failures are long. In 19 of the 22 failing pairs of tokenizer and text, every seam of the document matched more than 128 characters, and no token of the three vocabularies is longer than 128 characters. These documents meet the condition of Theorem~3.1 and still receive wrong IDs. The argument behind the theorem identifies a token with its span, and byte-level tokens inside one multi-byte character break that identification. A longer overlap moves the end of the left chunk but does not remove the case, since any chunk can end right after a character whose bytes serial encoding divides.

\subsection{Other inputs}
Comparing token IDs does not make the rule exact for WordPiece. BERT's tokenizer replaces a word longer than 100 characters with one unknown token whose span is the whole word, so the span of a token is not bounded by the longest entry of the vocabulary. When a cut falls inside such a word, each chunk sees a shorter word and splits it into pieces. Random testing found periodic words, such as a long run of \code{yF}, for which the pieces of the two chunks agree in position and ID over 26 characters, more than the 18 characters of BERT's longest token, while serial encoding emits a single unknown token for the whole word. In a 368-character example the joined output has 246 tokens and serial encoding 121. Raising the threshold of the theorem to the 100-character word limit rejects these seams.

The experiments of LoPT's paper accept a seam after two matching tokens, which is weaker than the condition of the theorem. In 60,000 random cases over six tokenizers, with chunk and overlap lengths between 1 and 512 characters, 25 joined outputs differ from serial encoding under this threshold. In 24 of them the matched run is shorter than the longest token, typically a special-token string or another long token cut into pieces that happen to line up. The remaining case matches 261 characters across a combining sequence whose bytes the two chunks group differently, the mechanism of the CLongEval failures. These inputs lie outside the settings of LoPT's evaluation.

\subsection{Speed with the same encoder}
Table~\ref{tab:lopt-speed} compares LoPT with \sys's safe cuts on web text and code, using the Hugging Face tokenizer as the encoder of every chunk in both, on eight cores of the Zen 5 server. LoPT needed no retry on any of these inputs, and both methods returned the serial IDs for all 360 of them. The safe cuts are 1.57 to 2.10 times faster. LoPT matches its seams one after another, which takes 19 to 36\% of its time, while safe cuts are joined by concatenation.

\begin{table}[h]
\centering
\begin{tabular}{lllrrrrrr}
\toprule
 & & & \multicolumn{3}{c}{8 chunks} & \multicolumn{3}{c}{16 chunks} \\
\cmidrule(lr){4-6}\cmidrule(lr){7-9}
Tokenizer & Input & Size & Serial & LoPT & Safe cuts & Serial & LoPT & Safe cuts \\
\midrule
Qwen3 & Web text & 256K & 103 & 31.7 & 18.6 & 101 & 32.9 & 19.4 \\
 &  & 1M & 420 & 121 & 73.7 & 417 & 123 & 77.9 \\
 &  & 4M & 1,784 & 499 & 286 & 1,789 & 508 & 305 \\
 & Code & 256K & 121 & 39.0 & 22.2 & 120 & 40.2 & 19.7 \\
 &  & 1M & 483 & 154 & 87.2 & 487 & 150 & 88.2 \\
 &  & 4M & 2,107 & 617 & 321 & 2,087 & 621 & 339 \\
\addlinespace
Llama-3.1 & Web text & 256K & 70.8 & 26.4 & 12.9 & 70.6 & 27.6 & 14.0 \\
 &  & 1M & 319 & 104 & 56.4 & 320 & 104 & 57.2 \\
 &  & 4M & 1,509 & 456 & 238 & 1,523 & 446 & 232 \\
 & Code & 256K & 82.2 & 32.8 & 15.7 & 82.4 & 35.3 & 17.0 \\
 &  & 1M & 362 & 134 & 66.0 & 362 & 131 & 66.1 \\
 &  & 4M & 1,703 & 560 & 267 & 1,642 & 554 & 272 \\
\bottomrule
\end{tabular}
\caption{Median encode time in milliseconds on eight cores, with the Hugging Face tokenizer as the encoder in every column. Serial encoding was measured again in each configuration. Each 8-chunk cell has 20 inputs, and each 16-chunk cell 10.}
\label{tab:lopt-speed}
\end{table}

\section{Measurement definitions and evaluation details}
\label{sec:measure}
This section defines the measurements and expands results and mechanisms that the main paper states in a sentence or two. It covers the rules by which the worker pool orders and cuts work, how the SGLang adapter finds a request's session, the full comparison of pool designs, the replay of SGLang's tokenizer stage, capacity on the 128-core Zen 3 node, SGLang under load, the fastokens divergence, fallback localization, and the Rust prototype. Latencies are in milliseconds. Unless a caption says otherwise, they are rounded from the recorded values to two decimals below 1\,ms, to one decimal below 100\,ms, and to whole milliseconds from 100\,ms up, and percentage changes and ratios are computed before rounding. In the tables, a dagger (\dag) marks a backlogged run, as defined for each harness below, and a dash marks a configuration that was not run at that point.

\subsection{Workload and latency units}
Character lengths count Unicode code points, not UTF-8 bytes. Throughput values in MB/s use UTF-8 byte counts. Model-context experiments state lengths in tokens.

Front-end replays execute real tokenizer calls with controlled synthetic arrivals. Their latency is completion time minus scheduled arrival, so it includes the time a request waits for its session's previous request. TTFT is the time to the first streamed token, and where arrivals follow a schedule it is measured from the scheduled arrival and includes the same wait. Closed-loop throughput cannot stand in for open-loop latency. In a closed loop, a faster front end sends each session's next turn sooner, so more turns overlap in GPU prefill and the median TTFT can rise while throughput goes up (Section~\ref{sec:ev-load}), whereas open-loop runs hold the arrival schedule fixed.

The request-form comparison uses Qwen3-8B in SGLang v0.5.19, a 100K-token context, 400-character appends, and 60 warm turns pooled from two server processes per tokenizer. Clean raw text, raw text with one decomposed accent, and rendered chat are measured separately (Table~\ref{tab:ev-forms}). The first warm request in each process is excluded. The confidence intervals resample these turns, so they do not include variation between separately started servers.

\begin{table}[h]
\centering
\begin{tabular}{lrrr}
\toprule
100K-token request form & Reference & Fast stateless & \sys \\
\midrule
Clean raw text & 140.64 & 54.21 & 53.17 \\
Raw text, one decomposed accent & 137.94 & 145.66 & 53.64 \\
Chat completion & 144.80 & 150.07 & 59.61 \\
\bottomrule
\end{tabular}
\caption{Median TTFT in milliseconds. These are the unrounded values behind the request-form comparison of the main paper.}
\label{tab:ev-forms}
\end{table}

\subsection{Scheduler and chunking rules as implemented}
\label{sec:ev-sched}

The pool is a set of worker processes. Each worker loads its own tokenizer, encodes with a single thread, and can be pinned to a core. Every worker has a local queue for the tasks of the sessions it owns, and one shared queue holds stateless tasks and chunks. Each session has a home worker, the live worker selected by a CRC-32 hash of the session key. An idle worker compares the head of its local queue with the head of the shared queue and takes the task with the smaller key, or the local task on a tie. Under FIFO order the key is the arrival time. Under shortest-first order it is the estimated service time plus the arrival time, both in milliseconds, which sorts tasks exactly as the estimate minus a credit that grows by one millisecond for every millisecond of waiting (the aging weight, 1.0 by default). Every queued task gains credit at the same rate, so the key never changes while a task waits and each queue can be a heap. A running encode is never interrupted.

The estimate for a task of $N$ characters is $a+b_\kappa N(1+5s)$, where $s$ is the share of non-ASCII characters and $a=0.2$\,ms. The per-character cost $b_\kappa$ depends on the kind of task. It starts at $3\times10^{-5}$\,ms for stateless tasks and chunks and at $10^{-5}$\,ms for session tasks, whose work is mostly an append, and after each task it moves toward the value implied by the measured service time, as an exponentially weighted average with weight 0.05. The non-ASCII share is zero for ASCII text, which Python recognizes in constant time. Other texts are counted in full up to 2,048 characters and through four 512-character windows spread over longer ones. The estimate only orders the queues, so it can be rough. At the end of the microbenchmark of Table~\ref{tab:ev-poolcost}, calibration had moved the per-character cost to $2.75$ and $2.76\times10^{-5}$\,ms for stateless tasks and to $2.78$ and $2.97\times10^{-5}$\,ms for chunks, with four and eight workers, and the mean relative error of the estimates was 22.3\% and 15.5\%.

A stateless request of $N$ characters is cut only if $N$ reaches a threshold, 131,072 characters by default. It then becomes $k=\min(\lceil N/s\rceil,k_{\max},W)$ chunks, where $s$ is the target chunk size, 65,536 characters by default, $k_{\max}=16$, and $W$ is the number of live workers. An optional rule also caps $k$ by the number of idle workers at arrival; no experiment uses it. The planner places tentative cuts at $\lfloor iN/k\rfloor$ for $i=1,\ldots,k-1$ and scans backward from each for an accepted cut, first over 48 characters and then, if that finds none, over up to 4,096. It does not make a cut that would leave a chunk shorter than $\max(16,\min(1024,\lfloor N/4k\rfloor))$ characters, and it never cuts a text shorter than 256 characters. The text of a cut that is not made stays with the preceding chunk, so a request can get fewer chunks than requested but never different IDs. Each chunk carries its request's identifier and its index. A worker returns the chunk's IDs as a packed array of 32-bit integers, and the collector stores them under the index and releases the request once every index has arrived.

The first error from any task fails its request, and later answers for that request are dropped. A worker that exits is retired. The request it was serving fails, and the session tasks queued on it move to the new home of their keys. A request with no answer 60\,s after its first dispatch also fails, and the pool starts only if every worker has loaded within 300\,s. When the pool fails a request, the adapter encodes it in its own process on the route the request would have taken without the pool, so the IDs never depend on the pool. Across the 103 pool runs of the tokenizer-stage replay (Section~\ref{sec:ev-replay}), 389,240 requests were submitted, and none failed, timed out, or fell back to that path. In 12 of these runs one stateless document was encoded whole, because the planner found no accepted cut within its scan window.

Table~\ref{tab:ev-params} lists the parameters and the values the experiments used. The scheduling testbed of the main paper implements the same designs in a separate harness and differs from the pool in three ways. It rounds $N/c$ to the nearest integer instead of rounding up. It orders shortest-first queues by a fixed estimate, 0.2\,ms plus $3.3\times10^{-5}$\,ms per character for stateless work and chunks, and 0.6\,ms plus $5\times10^{-4}$\,ms per appended character for session appends. Its per-worker FIFO design places stateless requests round-robin, so session affinity applies only to session work.

\begin{table}[htbp]
\centering
\small
\begin{tabular}{p{4.3cm}p{3.0cm}p{7.6cm}}
\toprule
Parameter & Pool default & In the experiments \\
\midrule
Queue order & shortest-first & shortest-first; FIFO in the four-worker FIFO comparison \\
Aging weight & 1.0\,ms per ms waited & 1.0 \\
Cutting threshold & 131,072 characters & 65,536 in the replay, the SGLang pool runs, and the microbenchmark; 131,072 in the testbed \\
Target chunk size $c$ & 65,536 characters & 65,536; also 32,768 and one chunk per worker on the 128-core Zen 3 node \\
Chunk count & $\min(\lceil N/c\rceil,k_{\max},W)$ & testbed rounds $N/c$ to the nearest integer \\
Chunk cap $k_{\max}$ & 16 & 16; equal to the number of workers on the 128-core Zen 3 node \\
Estimate, fixed part & 0.2\,ms & 0.2\,ms; 0.6\,ms for appends in the testbed \\
Estimate, per character & $3\times10^{-5}$\,ms stateless, $10^{-5}$\,ms session & calibrated as in the pool; fixed $3.3\times10^{-5}$ and $5\times10^{-4}$\,ms (per appended character) in the testbed \\
Non-ASCII weight & 5 & 5; not used in the testbed \\
Calibration weight & 0.05 & 0.05; no calibration in the testbed \\
Answer timeout & 60\,s & 60\,s \\
Pool start timeout & 300\,s & 300\,s \\
Encoding threads per worker & 1 & 1 \\
\bottomrule
\end{tabular}
\caption{Scheduler and chunking parameters. The pool default is what the adapter uses when nothing is configured.}
\label{tab:ev-params}
\end{table}

Table~\ref{tab:ev-poolcost} gives the cost of single calls. Sending the same session append through the pool adds 0.83\,ms at 100K characters and 5.73\,ms at one million, the cost of moving the whole transcript to the worker and its IDs back. Cutting a stateless million-character request into four chunks on four workers brings it from 30.80 to 12.43\,ms, and eight chunks on eight workers bring it to 8.39\,ms. At 100K characters, two chunks (2.06 and 2.07\,ms) are no faster than the serial encode (1.98\,ms). A single task of 60,000 characters takes 1.71 and 1.77\,ms on the pool.

\begin{table}[htbp]
\centering
\small
\begin{tabular}{lrr}
\toprule
Call & 100K characters & 1M characters \\
\midrule
Reference tokenizer, serial & 22.79 & 325.31 \\
\sys serial, stateless & 1.98 & 30.80 \\
\sys serial, session append of 29 characters & 0.68 & 5.34 \\
Pool of 4, same session append & 1.51 & 11.07 \\
Pool of 8, same session append & 1.46 & 11.29 \\
Pool of 4, stateless, one task & 1.71 & -- \\
Pool of 8, stateless, one task & 1.77 & -- \\
Pool of 4, stateless, chunked & 2.06 (2 chunks) & 12.43 (4 chunks) \\
Pool of 8, stateless, chunked & 2.07 (2 chunks) & 8.39 (8 chunks) \\
Pool of 4, four requests at once (wall time) & 5.95 & -- \\
Pool of 8, eight requests at once (wall time) & 7.08 & -- \\
\bottomrule
\end{tabular}
\caption{Cost of a single call with the Qwen3 tokenizer, in milliseconds to two decimals, measured on an otherwise idle host, each the median of five calls, or of three for the reference. The cutting threshold is 65,536 characters here, so 100K-character requests are cut. The one-task rows use the first 60,000 characters of the same document, which stay below the threshold and run as a single task. The session rows time the first call that appends 29 characters to a session holding the text, in a fresh session for each repetition; a second call with the same text, which finds it stored and only returns the IDs, costs about the same (0.75 and 6.44\,ms serial). Every answer of the pool equals the reference IDs.}
\label{tab:ev-poolcost}
\end{table}

\subsection{How a request finds its session in SGLang}
\label{sec:ev-session}

The adapter looks for a session key in each request, in this order: SGLang's session identifier (available from v0.5.17), the cache salt that separates a client's prefix cache from other clients', and the part of the request identifier before a separator, when one is configured. A request that uses SGLang's server-side continuation sends only its new turn, which is not a transcript that a stored session could extend, so it gets no key. In the configuration used by the experiments with session identifiers, the key reaches the encoder through a context variable that the adapter binds around SGLang's tokenization of each request, and a request without a key is served by content lookup. The adapter can also use content lookup for every request, which is its default, or encode every request without session state.

Content lookup needs only the text. Texts shorter than 4,096 bytes are encoded directly, because for them the bookkeeping costs more than the encode it would save. For every stored text, an index keeps 128-bit keyed BLAKE2b digests of prefixes, one at the end of the text and one at each byte offset that doubles from 4,096. A single pass over the request computes its running digest at every indexed length and proposes the longest stored text whose end digest matches. A byte-by-byte comparison with the stored text must confirm the match before any IDs are reused, so a colliding digest can waste a comparison but cannot change the output. When the confirmed text equals the request, its stored IDs are returned, and when it is a proper prefix, only the remainder goes through the incremental repair of the main paper. Otherwise the request is encoded in full and stored as a new entry. Three limits apply. Content lookup keeps at most 1,024 such entries and drops the oldest first. The in-process cache keeps up to 128\,MiB of session text resident and evicts the least recently used session first. The native store holds up to 4,096 sessions. A keyed session uses its key in place of the index. A stored text that equals the request is reused as it is, one that is a prefix of the request is extended, and any other is replaced.

\paragraph{Session state size.} A session's state is its text, which the store keeps for the prefix check, plus what \sys adds: a four-byte ID per token, source positions for the tail that a repair may touch, a few 16-byte prefix digests, and prefix blocks shared through the store, about 0.56\,MB for a long session that shares no prefix with another. Table~2 of the main paper gives the added state by context. With the text, a 100K-token session takes about 1.48\,MB. The 100K-token row was checked in a served store, where 20 sessions of about 100K tokens grew the native heap by 1,449,152 bytes each, against 1,447,367 by this accounting applied to those sessions' actual lengths. The 128\,MiB budget counts text only, and the shared blocks are not reclaimed when a session is evicted.

With the pool, a keyed request waits in the queue of its home worker. A request without a key is routed by content. The adapter hashes its first 4,096 characters with CRC-32 and uses the hash as its key, and because these characters are the same in every turn of a conversation, all turns reach the same worker, whose store then finds the stored prefix. A keyless request that is long enough to be cut is treated as stateless and cut into chunks, unless the adapter was configured to use content lookup for every request.

In the chat-interface runs behind the request-form comparison of the main paper, the adapter was in its keyed configuration but no key reached it, and all 99 encodes of each of the two server processes took the content route. With raw-text requests at 100K tokens, content lookup and the session identifier give the same median TTFT, 53.59\,ms, while the adapter's median encode time is 2.19 to 2.46\,ms with content lookup and 1.73 to 1.98\,ms with the identifier across three server processes. Two SGLang options change how keys arrive. The dynamic-batch tokenizer runs encodes in an executor that does not carry context variables, so every request takes the content route; with eight paced sessions the median TTFT is then 199.3\,ms, against 200.2\,ms with keys. With four tokenizer processes, each process keeps its own store, so the turns of one session can reach different stores. Keyed encodes slower than 2\,ms, which had to encode in full rather than append, rise from 0.68\% to 3.76\%, the adapter's encode P99 rises from 1.30 to 4.47\,ms, and the median TTFT is 201.1\,ms.

\subsection{Pool designs in full}
\label{sec:ev-designs}

The pool-design comparison runs in the scheduling testbed on four workers, with the standard mix of the main paper and every request encoded in full, without session reuse. Each request re-sends its session's transcript, which grows from turn to turn. In the runs labeled 100K, request sizes therefore have a median of about 131K characters and reach 390K, and about half of the requests are long enough to be cut, into two to four chunks. In the 1M runs the median is about 1.03 million characters, and 98\% of the requests are cut, always into four chunks. Every design replays the same arrivals, so the columns of Table~\ref{tab:ev-designs} line up. Because per-worker FIFO places stateless requests round-robin in the testbed, this comparison exercises no session affinity. The two-lane design uses the best of the twelve configurations in Table~\ref{tab:ev-lanes}. A run counts as backlogged if requests are still unfinished at the end, the achieved rate is below 85\% of the offered rate, or the P99 exceeds 3\,s, and the harness then skips the higher loads for that design.

At 86.5 requests/s of 1M-character requests, the P99 is 2,203\,ms with per-worker FIFO queues, 904\,ms with the shared shortest-first queue, and 136\,ms with chunks in FIFO order. Near saturation, ordering the chunks shortest-first helps further, with 738\,ms against 1,891\,ms at 122.7 requests/s. The P99 of the tuned two lanes is 2.2 and about 6 times that of chunks in FIFO order at the two loads used for tuning with 100K contexts, and 3.9 and 9.5 times with 1M contexts. At the lightest loads the ratios are 2.0 and 3.4. One worker alone has a P99 of 10.3\,ms at 80.3 requests/s and is backlogged at 353 with 100K contexts; with 1M contexts it has a P99 of 302\,ms at 6.5 requests/s and is backlogged at 31.5.

\begin{table}[htbp]
\centering
\small
\setlength{\tabcolsep}{4.5pt}
\begin{tabular}{lrrrrrrr}
\toprule
\multicolumn{8}{l}{\emph{100K-character contexts}} \\
Achieved load (requests/s) & 307 & 611 & 812 & 962 & 1,117 & 1,399 & \\
\midrule
Per-worker FIFO & 8.4 & 12.0 & 22.0 & 69.3 & 4,108\dag & -- & \\
Two lanes, tuned & 8.9 & 14.0 & 65.0 & 4,584\dag & -- & -- & \\
Shared shortest-first & 8.4 & 10.5 & 17.0 & 36.1 & 2,162 & 8,633\dag & \\
Chunks, FIFO & 4.4 & 6.4 & 10.5 & 24.0 & 675 & 7,650\dag & \\
Chunks, shortest-first & 4.5 & 6.8 & 10.3 & 23.6 & 371 & 6,876\dag & \\
\midrule
\multicolumn{8}{l}{\emph{1M-character contexts}} \\
Achieved load (requests/s) & 26.8 & 52.5 & 71.1 & 86.5 & 98.4 & 122.7 & 150.2 \\
\midrule
Per-worker FIFO & 345 & 456 & 1,048 & 2,203 & 5,664\dag & -- & -- \\
Two lanes, tuned & 355 & 425 & 1,204 & 5,716\dag & -- & -- & -- \\
Shared shortest-first & 350 & 405 & 593 & 904 & 1,500 & 6,494\dag & -- \\
Chunks, FIFO & 106 & 109 & 126 & 136 & 150 & 1,891 & 10,015\dag \\
Chunks, shortest-first & 99.6 & 110 & 124 & 137 & 157 & 738 & 5,814\dag \\
\bottomrule
\end{tabular}
\caption{P99 latency of the five pool designs on four workers, with every request encoded in full. The two-lane design sends requests of at least 256K characters to one worker and the rest to three at 100K, and requests of at least 128K characters to three workers and the rest to one at 1M. A dagger marks a backlogged run. The column at 1,117 requests/s sits at the edge of saturation, where a repeated run moves an entry by about 20\% but leaves the order of the designs unchanged.}
\label{tab:ev-designs}
\end{table}

The two-lane design sends a request of at least $\tau$ characters to the large lane and every other request to the small lane, each lane with its own FIFO queue and its own workers. Table~\ref{tab:ev-lanes} lists all twelve configurations at each context size, four thresholds times three splits of the four workers, each at the two loads used for tuning. The best configuration at 100K keeps 14.0\,ms and 55.2\,ms, and the best at 1M keeps 436 and 1,065\,ms. When these two configurations were run again for Table~\ref{tab:ev-designs}, they gave 14.0 and 65.0\,ms and 425 and 1,204\,ms at the same loads.

\begin{table}[htbp]
\centering
\small
\begin{tabular}{lccc}
\toprule
Threshold $\tau$ & Small:large = 1:3 & 2:2 & 3:1 \\
\midrule
\multicolumn{4}{l}{\emph{100K-character contexts, loads of 611 and 812 requests/s}} \\
8K & 15.5 / 510 & \dag & \dag \\
32K & 15.5 / 506 & \dag & \dag \\
128K & 34.3 / \dag & 22.8 / 358 & \dag \\
256K & \dag & 738 / \dag & \textbf{14.0 / 55.2} \\
\midrule
\multicolumn{4}{l}{\emph{1M-character contexts, loads of 52.5 and 71.1 requests/s}} \\
32K & 424 / 1,171 & 2,958 / \dag & \dag \\
128K & \textbf{436 / 1,065} & \dag & \dag \\
512K & 443 / 1,151 & 2,675 / \dag & \dag \\
1M & \dag & 528 / \dag & 415 / 2,167 \\
\bottomrule
\end{tabular}
\caption{P99 latency of the twelve two-lane configurations at each context size, given as P99 at the lower load / P99 at the higher load. Columns give how many of the four workers serve the small and the large lane, and K is 1,024 characters. A dagger alone means the configuration was backlogged at the lower load and not run at the higher one. The configuration in bold is the one used in Table~\ref{tab:ev-designs}.}
\label{tab:ev-lanes}
\end{table}

Table~\ref{tab:ev-mixes} gives the two mixes with session reuse that the main paper summarizes, on four workers with 1M-character contexts. Session work stays on its home worker in every design, and the two-lane design is the untuned default, a 32K threshold with three small-lane workers that also serve the sessions. In the first mix, 92.3\% of the requests in the measured window are appends (1,632 of 1,768 at 88.4 requests/s, with 103 session starts and 33 stateless documents). All five designs then have a P99 between 397 and 410\,ms, set by session starts, whose P99 is 440 to 461\,ms, since a start encodes a whole million-character transcript on its home worker and no queue order can shorten it. At 225.4 requests/s the two shortest-first designs still hold (1,637 and 1,697\,ms), while the other three are backlogged.

In the second mix, clusters of four stateless million-character documents arrive as a Poisson process at 0.5 clusters per second on top of the session load. At the lightest load, chunking cuts the documents' P99 from 348 to 126\,ms and the P99 of the appends queued behind them from 301 to 134\,ms, while session starts stay at 416 and 418\,ms. Two lanes keep the appends fast (49.0\,ms) by giving the documents a worker of their own, where the documents wait 656\,ms. Across the five loads the chunked designs keep the documents between 118 and 139\,ms, and at the four higher loads the append P99 of every design lies between 375 and 481\,ms.

\begin{table}[htbp]
\centering
\small
\setlength{\tabcolsep}{4.5pt}
\begin{tabular}{lrrrrr}
\toprule
\multicolumn{6}{l}{\emph{Appends to 1M-character contexts, P99 of all requests}} \\
Achieved load (requests/s) & 88.4 & 168.2 & 225.4 & 264.2 & \\
\midrule
Per-worker FIFO & 404 & 497 & 3,261\dag & -- & \\
Two lanes & 410 & 681 & 8,734\dag & -- & \\
Shared shortest-first & 397 & 480 & 1,637 & 6,355\dag & \\
Chunks, FIFO & 400 & 494 & 3,576\dag & -- & \\
Chunks, shortest-first & 400 & 476 & 1,697 & 6,253\dag & \\
\midrule
\multicolumn{6}{l}{\emph{With clusters of million-character documents, P99 of documents / P99 of appends}} \\
Achieved load (requests/s) & 42.6 & 84.4 & 107.0 & 127.6 & 145.0 \\
\midrule
Per-worker FIFO & 348 / 301 & 511 / 431 & 367 / 430 & 516 / 449 & 370 / 481 \\
Two lanes & 656 / 49.0 & 1,277 / 381 & 1,104 / 417 & 992 / 446 & 654 / 477 \\
Shared shortest-first & 349 / 291 & 395 / 376 & 373 / 415 & 423 / 439 & 377 / 449 \\
Chunks, FIFO & 126 / 134 & 139 / 377 & 135 / 414 & 138 / 441 & 127 / 454 \\
Chunks, shortest-first & 118 / 141 & 129 / 375 & 128 / 405 & 123 / 424 & 129 / 438 \\
\bottomrule
\end{tabular}
\caption{The two mixes with session reuse on four workers. The two-lane design uses a 32K threshold with three small-lane workers. A dagger marks a backlogged run.}
\label{tab:ev-mixes}
\end{table}

Chunking also lowers the variability of service. With a request's service summed over its tasks, the second moment of service falls from 4,052 to 1,303\,ms$^2$ at four workers, one million characters, and 26 requests/s. This is a descriptive statistic, and the P99 values above are measured directly.

A simple cost model describes the best chunk count of an isolated encode. With $T(N,k)\approx a+bN/k+ck$, where $a$ is fixed overhead, $b$ the per-character encoding cost, and $c$ the per-chunk overhead, the time is lowest near $k\approx\sqrt{bN/c}$. The scheduler does not use this model, and its chunk count follows the target size.

\subsection{Replay of SGLang's tokenizer stage}
\label{sec:ev-replay}

The replay reproduces the request path of SGLang's tokenizer manager in one process with one asyncio event loop, without a server or a model. The reference tokenizer and \sys serial run on the loop's thread, as SGLang runs its tokenizer. The in-process configuration adds chunk-parallel encoding with four threads for stateless requests of at least 65,536 characters, with the chunk count taken from the planner's default table, which gives one chunk below 256,000 characters, four below 1,000,000, sixteen below 4,000,000, and 32 above, and at most one chunk per 65,536 characters. The pool configurations hand every request to a pool of two, four or eight workers that orders its queues shortest-first, keeps sessions on their home workers, and cuts stateless requests of at least 65,536 characters into chunks of about 65,536 characters. The load is the standard mix of the scheduling testbed, with 64 sessions whose turns each re-send the whole transcript with a session key, 2.3\% stateless documents of 32K to 100K characters, Poisson arrivals, and at most one request in flight per session. Each configuration and rate runs in a fresh process with the Qwen3 tokenizer, for a 20-second window after a 4-second warm-up, on the 16-core Zen 5 server with no other job on the processor that runs the replay. After each run the chunked requests and a 5\% sample are re-encoded with the reference and compared, and every request is compared across the configurations run at the same rate. No comparison differed. A run counts as backlogged if requests are unfinished at the end or the P99 exceeds 10\,s.

Table~\ref{tab:ev-replay} gives the median and P99 at each offered rate; the achieved rates are within $-1.0\%$ and $+8.0\%$ of the offered ones. With the reference tokenizer every turn re-encodes the whole transcript, so its P99 reaches 524\,ms at 20 requests/s with 100K-character sessions and 945\,ms at 10 requests/s with 250K, and it is backlogged at 5 requests/s with 1M. \sys serial encodes only each turn's append and a short suffix before it. With 100K sessions its median stays between 1.6 and 2.2\,ms and its P99 between 4.4 and 9.9\,ms up to 320 requests/s, and it saturates at 640. The pools add about 1\,ms to the median with 100K sessions, about 2\,ms with 250K and about 6.5\,ms with 1M, the cost of moving each transcript and its IDs between processes, and they carry the load further, with a P99 of 27.2\,ms for the four-worker pool at 640 requests/s with 100K sessions. The in-process configuration never cuts a request in these three loads, because every stateless document is shorter than 256,000 characters, so it tracks \sys serial. With 1M-character sessions, session starts, the first request of each session, have a P99 of 383 to 409\,ms in every \sys configuration at 5 requests/s. On the serial path a start also holds up the turns queued behind it on the event loop, so \sys serial has a P99 of 55.3\,ms at 5 requests/s against a median of 6.6\,ms, while pools of four and eight workers, where a start occupies only its home worker, keep the P99 at 19.1 and 17.9\,ms.

\begin{table}[htbp]
\centering
\footnotesize
\setlength{\tabcolsep}{3.5pt}
\begin{tabular}{rllllll}
\toprule
Rate & Reference & \sys serial & In-process chunks & Pool of 2 & Pool of 4 & Pool of 8 \\
\midrule
\multicolumn{7}{l}{\emph{100K-character sessions}} \\
5 & 36.2 / 73.9 & 1.7 / 4.9 & 1.6 / 4.7 & 2.9 / 4.7 & 2.7 / 5.7 & 2.8 / 5.5 \\
10 & 39.3 / 108 & 1.6 / 4.5 & 1.7 / 4.7 & 2.6 / 6.0 & 2.7 / 7.6 & 2.8 / 6.9 \\
20 & 80.8 / 524 & 1.6 / 4.4 & 1.6 / 4.5 & 2.7 / 6.1 & 2.6 / 6.2 & 2.7 / 6.9 \\
40 & 7,161\dag{} / 20,387\dag & 1.7 / 5.4 & 1.7 / 5.5 & 2.8 / 6.1 & 2.8 / 6.5 & 2.8 / 7.0 \\
80 & -- & 1.7 / 4.9 & 1.8 / 5.2 & 2.8 / 6.1 & 2.8 / 6.3 & 2.9 / 6.8 \\
160 & -- & 1.8 / 6.4 & 1.8 / 6.2 & 3.0 / 7.3 & 3.0 / 6.9 & 3.0 / 7.4 \\
320 & -- & 2.2 / 9.9 & 2.2 / 9.8 & 3.4 / 10.5 & 3.3 / 9.1 & 3.2 / 8.4 \\
640 & -- & 1,170 / 4,363 & 1,179 / 4,393 & 8.4 / 58.1 & 5.7 / 27.2 & 5.5 / 24.2 \\
800 & -- & 5,178\dag{} / 11,075\dag & 5,187\dag{} / 11,094\dag & 990 / 4,518 & 81.3 / 1,059 & 37.9 / 839 \\
1,000 & -- & -- & -- & 4,815\dag{} / 11,396\dag & 2,958 / 8,808 & 2,998\dag{} / 10,685\dag \\
\midrule
\multicolumn{7}{l}{\emph{250K-character sessions}} \\
5 & 102 / 283 & 2.6 / 5.2 & 2.7 / 10.1 & 4.5 / 9.7 & 4.5 / 8.0 & 4.6 / 8.4 \\
10 & 468 / 945 & 2.5 / 11.5 & 2.5 / 10.9 & 4.3 / 12.6 & 4.5 / 15.4 & 4.6 / 14.5 \\
20 & 11,794\dag{} / 26,581\dag & 2.3 / 10.7 & 2.4 / 10.6 & 4.4 / 12.0 & 4.2 / 13.6 & 4.3 / 13.3 \\
40 & -- & 2.4 / 10.3 & 2.4 / 10.6 & 4.5 / 13.2 & 4.3 / 13.7 & 4.4 / 14.7 \\
80 & -- & 2.6 / 12.0 & 2.6 / 12.3 & 4.7 / 13.4 & 4.6 / 14.5 & 4.7 / 14.7 \\
160 & -- & 3.4 / 18.9 & 3.3 / 19.3 & 5.5 / 19.0 & 5.3 / 17.2 & 5.3 / 16.9 \\
320 & -- & 740 / 2,618 & 701 / 2,499 & 11.7 / 58.0 & 8.6 / 32.5 & 8.4 / 28.9 \\
640 & -- & 15,438\dag{} / 28,169\dag & 15,451\dag{} / 28,200\dag & 7,978\dag{} / 16,682\dag & 6,172\dag{} / 14,154\dag & 5,982\dag{} / 15,652\dag \\
\midrule
\multicolumn{7}{l}{\emph{1M-character sessions}} \\
5 & 15,783\dag{} / 29,188\dag & 6.6 / 55.3 & 6.6 / 87.6 & 13.4 / 73.4 & 12.8 / 19.1 & 12.9 / 17.9 \\
10 & -- & 6.6 / 202 & 6.5 / 206 & 12.8 / 58.2 & 12.6 / 53.2 & 12.7 / 49.6 \\
20 & -- & 6.4 / 40.0 & 6.6 / 39.0 & 13.2 / 46.1 & 13.0 / 53.4 & 12.9 / 52.0 \\
40 & -- & 7.6 / 399 & 7.7 / 401 & 14.8 / 380 & 14.4 / 166 & 14.1 / 76.5 \\
80 & -- & 911 / 2,544 & 1,055 / 2,912 & 41.6 / 477 & 26.6 / 432 & 24.7 / 350 \\
160 & -- & 14,267\dag{} / 29,354\dag & 14,111\dag{} / 29,458\dag & 6,338\dag{} / 18,763\dag & 3,101\dag{} / 16,001\dag & 2,676\dag{} / 15,513\dag \\
\bottomrule
\end{tabular}
\caption{Median / P99 latency from scheduled arrival in the replay of SGLang's tokenizer stage, by offered rate in requests/s. No configuration was run above the first rate at which it backlogged. A dagger marks a backlogged run.}
\label{tab:ev-replay}
\end{table}

In the burst workload, clusters of four stateless documents of one million characters arrive as a Poisson process at 0.5 clusters per second, on top of session turns at the stated rate with 100K-character sessions (Table~\ref{tab:ev-burst}). \sys serial encodes each document on the event loop, and the turns queued behind a cluster wait, so the turn P99 is 325 to 654\,ms. Encoding each document as 16 chunks with four threads in process brings the turn P99 to 38.5 to 54.7\,ms up to 160 requests/s. The pools cut each document into four or eight chunks and encode them away from the loop, and they hold the turn P99 at 20.5 to 49.2\,ms with four workers and 10.0 to 21.0\,ms with eight, the highest values at 320 requests/s, where another user's job started on the server's other processor during the runs. The reference tokenizer is already backlogged at 20 requests/s. Each run has 24 to 56 documents in its window, and the document P99 rests on these. Counting the documents, the achieved rates are up to 12.4\% above the offered turn rates.

\begin{table}[htbp]
\centering
\small
\setlength{\tabcolsep}{4pt}
\begin{tabular}{rlllllr}
\toprule
Rate & Reference & \sys serial & In-process chunks & Pool of 4 & Pool of 8 & Documents \\
\midrule
20 & 17,160\dag{} / 12,802\dag & 377 / 397 & 54.7 / 111 & 32.4 / 142 & 14.5 / 94.2 & 40 \\
40 & -- & 325 / 478 & 38.5 / 252 & 20.5 / 119 & 10.0 / 72.0 & 56 \\
80 & -- & 654 / 764 & 53.0 / 73.1 & 31.7 / 110 & 15.1 / 68.4 & 48 \\
160 & -- & 542 / 728 & 48.8 / 81.1 & 28.9 / 121 & 16.0 / 72.8 & 48 \\
320 & -- & 530 / 525 & 200 / 269 & 49.2 / 132 & 21.0 / 83.1 & 24 \\
\bottomrule
\end{tabular}
\caption{P99 of session turns / P99 of documents in the burst workload of the tokenizer-stage replay, by offered turn rate in requests/s. The last column counts the documents in the measured window. A dagger marks a backlogged run.}
\label{tab:ev-burst}
\end{table}

The four-worker FIFO comparison of the main paper uses the same harness and the pool of the adapter, with the stateless form of the standard mix at 1M characters and one shared FIFO queue. In one configuration no request is cut. In the other, requests of at least 65,536 characters are cut, a transcript into four chunks and the few stateless documents above that size into two. Table~\ref{tab:ev-fifo} gives the results at four rates, and three arrival seeds at 26 requests/s give P99 values of 353, 359 and 377\,ms for whole requests and 102, 102 and 115\,ms for chunks.

\begin{table}[htbp]
\centering
\small
\begin{tabular}{rrll}
\toprule
Offered rate & Achieved rate & Whole requests & Chunks \\
\midrule
13 & 12.4 & 32.0 / 346 & 13.6 / 95.7 \\
26 & 26.8 & 32.1 / 374 & 14.4 / 103 \\
39 & 41.4 & 36.3 / 409 & 16.5 / 117 \\
52 & 52.4 & 40.6 / 451 & 19.5 / 132 \\
\bottomrule
\end{tabular}
\caption{Median / P99 latency with one shared FIFO queue on four workers and stateless million-character requests, with rates in requests/s. The arrivals are the same in both columns.}
\label{tab:ev-fifo}
\end{table}

\subsection{Capacity on the 128-core Zen 3 node}
\label{sec:ev-capacity}

The capacity runs use the scheduling testbed on 128-core Zen 3 nodes, with the workers pinned to 16, 32 or 64 cores. The capacity brackets of Section~4.7 of the main paper come from the repeats on a rented node described later in this section; the runs described first used shared cluster nodes with two 64-core EPYC 7713 processors. Eight load-generating processes send arrivals to one dispatcher process, which assigns tasks to the workers, because one process cannot generate enough load. The load is the stateless form of the standard mix, with 64 sessions, every request re-sending its session's transcript and encoded in full, and at most one request in flight per session. A run passes if its P99 is at most 50\,ms, and it counts as backlogged under the testbed rule of Section~\ref{sec:ev-designs}. The runs used several hosts of the same type, some shared with other jobs, so repeated points can differ; Table~\ref{tab:ev-capacity} lists every run near each bracket. Runs in which the host stalled during load generation were excluded before analysis and are not shown. In the chunked rows, chunks of the 64K or 32K target come from rounding $N/c$, and one chunk per worker means chunks of $N/W$ characters but at least 16,384, which gives 32 chunks of about 31,250 characters at 32 workers and 61 to 64 chunks at 64 workers.

\begin{table}[htbp]
\centering
\small
\setlength{\tabcolsep}{4pt}
\begin{tabular}{rllll}
\toprule
Workers & Configuration & Passes in every run & Mixed & Fails in every run \\
\midrule
\multicolumn{5}{l}{\emph{1M-character requests}} \\
16 & Chunks, 64K target & 244.4 (44.8, 44.8) & -- & 290.2 (59.5--62.2, 3 runs) \\
16 & Per-worker FIFO & -- & -- & 91.7 (1,421) \\
32 & Chunks, 64K target & 411.5 (48.9, 49.5) & -- & 488.7 (53.7--58.9, 3 runs) \\
32 & Chunks, one per worker & 530 (30.1, 38.9) & 565.8 (45.3, 50.4) & 620 (198, 1,398) \\
32 & Per-worker FIFO & -- & -- & 154.3 (4,878\dag) \\
64 & Chunks, 64K target & 500 (46.0) & -- & 625 (56.6--59.1, 5 runs) \\
64 & Chunks, 32K target & 750 (27.8, 29.1) & -- & 875 (56.1--104, 4 runs) \\
64 & Chunks, one per worker & 500 (20.9) & 625 (39.1, 93.5) & 750 (4,078\dag, 5,095\dag) \\
\midrule
\multicolumn{5}{l}{\emph{100K-character requests}} \\
16 & Chunks, 64K target & 1,887 (10.7) & -- & 2,516 (71.3) \\
16 & Per-worker FIFO & 1,150 (20.4) & -- & 1,400 (507) \\
32 & Chunks, 64K target & 3,850 (17.8) & -- & 4,422 (2,620) \\
32 & Per-worker FIFO & 1,658 (28.7) & -- & 2,200 (63.3) \\
64 & Chunks, 64K target & 4,900 (38.9) & -- & 5,907 (249) \\
64 & Per-worker FIFO & 1,500 (34.2) & -- & 2,200 (53.7) \\
\bottomrule
\end{tabular}
\caption{Offered rates in requests/s around a 50\,ms P99 objective on the 128-core Zen 3 node, with the P99 of each run in parentheses. The columns give the highest rate that passed in every run, a rate that passed in one run and failed in another, and the lowest rate above it that failed in every run. Several values are runs on different hosts. A dash means that no tested rate falls in that column. All rows use the earlier version of the testbed's dispatcher.}
\label{tab:ev-capacity}
\end{table}

The three brackets of Section~4.7 in the main paper were repeated on a 128-core Zen 3 node with two 64-core EPYC 7763 (up to 3.5\,GHz) rented for these runs, with the current dispatcher and three arrival seeds per point. Host-wide CPU use outside our container, measured from \code{/proc/stat} against the container's own accounting, stayed below 0.7 cores during every group. Table~\ref{tab:ev-capacity-dedicated} gives each seed's bracket, the highest tested rate that passed and the lowest tested rate above it that failed, on a grid that includes the shared-node points and one or two rates between them. The union over seeds, the highest rate that passed in every seed and the lowest that failed in every seed, is the bracket reported in Section~4.7, and it coincides with the shared-node brackets above.

\begin{table}[htbp]
\centering
\small
\begin{tabular}{rlllll}
\toprule
Workers & Configuration & Seed 1 & Seed 2 & Seed 3 & All seeds \\
\midrule
16 & Chunks, 64K target & 270 $\rightarrow$ 290.2 & 270 $\rightarrow$ 290.2 & 244.4 $\rightarrow$ 270 & 244.4 $\rightarrow$ 290.2 \\
32 & Chunks, one per worker & 530 $\rightarrow$ 565.8 & 565.8 $\rightarrow$ 620 & 565.8 $\rightarrow$ 620 & 530 $\rightarrow$ 620 \\
64 & Chunks, 32K target & 750 $\rightarrow$ 812 & 812 $\rightarrow$ 875 & 812 $\rightarrow$ 875 & 750 $\rightarrow$ 875 \\
\bottomrule
\end{tabular}
\caption{Offered rates in requests/s of million-character stateless requests on the dedicated 128-core Zen 3 node, by arrival seed. Each entry is the highest tested rate whose P99 stayed at or below 50\,ms, followed by the lowest tested rate above it that did not. The last column is the union over seeds. A run is valid if it exits cleanly, every scheduled arrival is either completed or still pending at the end, and the host did not stall. Every run was valid in this sense, and 24,475 shadow comparisons found no difference.}
\label{tab:ev-capacity-dedicated}
\end{table}

On the shared nodes, interpolating the P99 linearly between the rates of each bracket, the chunked pool with 100K-character requests meets the objective up to about 2,295, 3,857 and 4,953 requests/s on 16, 32 and 64 workers, and per-worker FIFO up to about 1,165, 1,992 and 2,066.

At 64 workers the target chunk size matters (Table~\ref{tab:ev-chunksize}). With the earlier dispatcher and 750 requests/s, 32K-character chunks keep the P99 at 27.8 and 29.1\,ms in two runs, 64K-character chunks give 78.4\,ms, and one chunk per worker, about twice as many chunks as with the 32K target, is backlogged in both runs. The dispatcher handles every task, and its CPU time grows with the number of tasks. At 750 requests/s it handles 11,410 tasks/s with 64K chunks and uses 38\% of one core, 22,916 tasks/s and 62\% with 32K chunks, and 41,696 tasks/s and 93\% with one chunk per worker. A later version of the dispatcher does less work per task, with constant-time queue operations, batched reads of worker messages, and task messages that the load generators serialize in advance. It lowers these shares to 46\% with 32K chunks and 71\% with one chunk per worker, and one chunk per worker then reaches 272 and 289\,ms at 750 requests/s instead of backlogging, while the 32K target stays at 28.8\,ms.

\begin{table}[htbp]
\centering
\small
\setlength{\tabcolsep}{3.5pt}
\begin{tabular}{llrrrrrrrr}
\toprule
Target & Chunks & 250 & 375 & 500 & 625 & 750 & 875 & 1,000 & Dispatcher CPU at 750 \\
\midrule
\multicolumn{10}{l}{\emph{Earlier dispatcher}} \\
64K & 15--20 & 38.5 & 39.2 & 46.0 & 56.6 & 78.4 & 101 & -- & 0.38 \\
32K & 31--39 & 20.5 & 36.5 & 23.4 & 24.5 & 29.1 & 104 & 3,057\dag & 0.62 \\
32K, repeat & 31--39 & 20.4 & 21.7 & 22.0 & 36.7 & 27.8 & 69.9 & 3,252\dag & 0.62 \\
One per worker & 61--64 & 13.7 & 14.4 & 20.9 & 39.1 & 4,078\dag & -- & -- & 0.93 \\
One per worker, repeat & 61--64 & -- & -- & -- & 93.5 & 5,095\dag & -- & -- & 0.93 \\
\midrule
\multicolumn{10}{l}{\emph{Later dispatcher}} \\
64K & 15--20 & -- & -- & -- & -- & -- & 91.4 & 125 & -- \\
32K & 31--39 & -- & 21.3 & 22.7 & 24.0 & 28.8 & 100 & 2,993 & 0.46 \\
One per worker & 61--64 & 13.2 & 15.8 & 16.5 & 26.1 & 289 & 5,440\dag & -- & 0.71 \\
One per worker, repeat & 61--64 & -- & -- & -- & 26.9 & 272 & -- & -- & 0.71 \\
\bottomrule
\end{tabular}
\caption{P99 latency on 64 workers with million-character requests, by target chunk size and offered rate in requests/s. The second column gives the range of chunks per request. The last column is the dispatcher's CPU time divided by the run's duration. The two repeats with one chunk per worker ran on the same host with the same allocation. With the 64K target and the earlier dispatcher, 1,094 requests/s is backlogged (8,182\,ms), and with the 32K target and the later dispatcher, 1,125 requests/s is backlogged (5,743\,ms). A dagger marks a backlogged run.}
\label{tab:ev-chunksize}
\end{table}

The 100K-character runs and the chunk-size comparison were also repeated on the rented node with the current dispatcher and three arrival seeds, on the shared-node grid extended by one or two rates. Table~\ref{tab:ev-capacity-100k} gives the brackets. Every bracket is higher than on the shared nodes, whereas the million-character brackets are the same on both. The shared queue passes at 5,907 requests/s on 64 workers in every seed, and linear interpolation puts its capacity at 6,198, 7,657 and 6,732 requests/s in the three seeds; per-worker FIFO reaches about 3,080 to 3,100 requests/s on 64 workers, no more than the 3,189 to 3,325 it reaches on 32. At 64 workers and 750 requests/s of million-character requests, the 32K target gives a P99 of 30.0, 37.5 and 44.3\,ms in the three seeds, the 64K target 63.0 to 78.0\,ms, and one chunk per worker 311 to 340\,ms; at 875 requests/s one chunk per worker overloads in every seed, while the 32K target reaches 64.9 to 152\,ms and the 64K target 79.3 to 100\,ms. The order of the three targets is the same as on the shared nodes. Over the whole repeat, 215 runs in 24 groups, every run passed the validity rule and 429,457 shadow comparisons found no difference.

\begin{table}[htbp]
\centering
\small
\begin{tabular}{rlllll}
\toprule
Workers & Design & Seed 1 & Seed 2 & Seed 3 & All seeds \\
\midrule
16 & Per-worker FIFO & 1,887 $\rightarrow$ 2,200 & 1,887 $\rightarrow$ 2,200 & 1,887 $\rightarrow$ 2,200 & 1,887 $\rightarrow$ 2,200 \\
16 & Shared queue & 2,516 $\rightarrow$ 2,750 & 2,516 $\rightarrow$ 2,750 & 2,516 $\rightarrow$ 2,750 & 2,516 $\rightarrow$ 2,750 \\
32 & Per-worker FIFO & 3,316 $\rightarrow$ 3,850 & 2,750 $\rightarrow$ 3,316 & 2,750 $\rightarrow$ 3,316 & 2,750 $\rightarrow$ 3,850 \\
32 & Shared queue & 4,900 $\rightarrow$ 5,251 & 4,422 $\rightarrow$ 4,900 & 4,422 $\rightarrow$ 4,900 & 4,422 $\rightarrow$ 5,251 \\
64 & Per-worker FIFO & 2,953 $\rightarrow$ 3,900 & 2,953 $\rightarrow$ 3,900 & 2,953 $\rightarrow$ 3,900 & 2,953 $\rightarrow$ 3,900 \\
64 & Shared queue & 5,907 $\rightarrow$ 6,500 & 7,656 $\rightarrow$ 9,570 & 6,500 $\rightarrow$ 7,656 & 5,907 $\rightarrow$ 9,570 \\
\bottomrule
\end{tabular}
\caption{Offered rates in requests/s of 100K-character stateless requests on the dedicated 128-core Zen 3 node, by arrival seed, with the 64K chunk target, which leaves these requests whole. Each entry is the highest tested rate whose P99 stayed at or below 50\,ms, followed by the lowest tested rate above it that did not; the last column is the union over seeds. In seed 1 on 64 workers the P99 crossed 50\,ms at 6,500 requests/s (64.8\,ms) and fell below it again at 7,656 (45.6\,ms) before the run at 9,570 overloaded; the bracket takes the first crossing.}
\label{tab:ev-capacity-100k}
\end{table}

Traffic that follows real agent traces asks for much less. In the replay of Table~\ref{tab:ev-trace}, the sessions' arrival times and request sizes follow real agent traces, and the sessions are filled with web text. Every session starts at its first call, when contexts are still short. At 50 requests/s the requests have a median of 21,677 characters, a 90th percentile of 52,491, a 99th percentile of 188,024, and a maximum of 641,692; the medians at the higher rates are 21.7K to 22.9K characters. The rate is raised by drawing more sessions. These runs use the scheduling testbed on sixteen cores of the Zen 5 server with the default cutting rule of 131,072 and 65,536 characters.

\begin{table}[htbp]
\centering
\small
\begin{tabular}{lrrrr}
\toprule
\multicolumn{5}{l}{\emph{Four workers, every call encoded in full}} \\
Achieved load (requests/s) & 52.6 & 105.1 & 210.0 & 420.4 \\
\midrule
Per-worker FIFO & 6.0 & 4.3 & 3.6 & 4.5 \\
Shared shortest-first & 5.5 & 4.1 & 3.5 & 3.5 \\
Chunks, FIFO & 4.1 & 3.1 & 3.0 & 3.0 \\
\midrule
\multicolumn{5}{l}{\emph{Four workers, with session reuse}} \\
Achieved load (requests/s) & 105.1 & 210.0 & 420.4 & 841.8 \\
\midrule
Per-worker FIFO & 4.8 & 4.4 & 6.3 & 11.5 \\
Shared shortest-first & 4.7 & 4.4 & 6.3 & 11.6 \\
Chunks, FIFO & 5.1 & 4.4 & 6.4 & 12.4 \\
\midrule
\multicolumn{5}{l}{\emph{Sixteen workers, with session reuse}} \\
Achieved load (requests/s) & 210.0 & 420.4 & 841.8 & 1,262.3 \\
\midrule
Per-worker FIFO & 3.9 & 5.2 & 6.4 & 6.9 \\
Chunks, FIFO & 4.5 & 5.2 & 6.9 & 7.0 \\
\bottomrule
\end{tabular}
\caption{P99 latency in the replay whose sessions follow real agent traces in their arrival times and request sizes, filled with web text. The workers are at most 7.9\% busy in the first block, 14.9\% in the second, and 5.5\% in the third.}
\label{tab:ev-trace}
\end{table}

\subsection{SGLang under load}
\label{sec:ev-load}

Table~\ref{tab:ev-load} lists the runs. They use SGLang v0.5.19 on an RTX PRO 6000 GPU of the Zen 5 server and serve Qwen3-8B, except that the desktop row uses an RTX 5090 in a Zen 5 desktop. The adapter runs in process without the pool, and the sessions take their text from the public agent transcripts described in the main paper. Every turn re-sends the session's whole transcript with its session identifier, and TTFT is measured on warm turns only. In the closed-loop runs, each session sends its next turn as soon as the previous one has been answered, and its transcript grows from 8K to 48K tokens over 30 turns. The paced runs add a think time drawn uniformly between 1 and 3\,s before each turn, and the desktop's paced runs grow from 8K to 16K tokens. The small-append runs start from a 32K-token prefix and append 400 characters per turn in a closed loop. The Poisson run sends 15 turns/s over twelve sessions of 70K characters for 60\,s and measures TTFT from each turn's scheduled arrival.

\begin{table}[htbp]
\centering
\small
\setlength{\tabcolsep}{4pt}
\begin{tabular}{lrrlll}
\toprule
Setting & Sessions & Runs & Reference & \sys serial & Change \\
\midrule
Closed loop & 1 & 3 & 177 / 309 & 158 / 276 & $-$10.5 / $-$11.0\% \\
Closed loop & 4 & 3 & 520 / 878 & 568 / 887 & $+$9.2 / $+$1.0\% \\
Closed loop & 8 & 3 & 870 / 1,660 & 941 / 1,596 & $+$8.1 / $-$3.8\% \\
Paced & 4 & 2 & 198 / 481 & 172 / 383 & $-$13.0 / $-$20.4\% \\
Paced & 8 & 2 & 227 / 762 & 201 / 719 & $-$11.5 / $-$5.6\% \\
Small appends, closed loop & 4 & 2 & 101 / 147 & 94.7 / 106 & $-$6.4 / $-$27.5\% \\
Small appends, closed loop & 8 & 2 & 174 / 293 & 166 / 187 & $-$4.5 / $-$36.1\% \\
Poisson, 15 turns/s & 12 & 1 & 85.9 / 662 & 54.6 / 512 & $-$36.4 / $-$22.6\% \\
Desktop, paced & 4 & 2 & 48.0 / 88.1 & 38.8 / 85.3 & $-$19.2 / $-$3.3\% \\
\bottomrule
\end{tabular}
\caption{Warm-turn TTFT in SGLang, median / P99 in milliseconds, with the reference tokenizer and with \sys serial. The last column gives the change of the median and of the P99. The Poisson row measures from each turn's scheduled arrival.}
\label{tab:ev-load}
\end{table}

In the closed loop with four or eight sessions, \sys raises the median TTFT by 8 to 9\% while completing more turns per second, 4.196 against 4.079 with four sessions and 4.632 against 4.535 with eight. A faster front end sends each session's next turn sooner, so more turns overlap in GPU prefill. With one session, with paced arrivals, and with open-loop arrivals, the median falls. The small-append runs are closed-loop as well, and there both the median (by 6.4\% and 4.5\%) and the P99 (by 27.5\% and 36.1\%) fall. Section~\ref{sec:per-tokenizer} gives the runs with gpt-oss-120b and how their exactness was checked.

When six sessions replay 250K-character transcripts open-loop, limiting SGLang's HTTP and tokenizer process to one core or to eight makes no difference for either tokenizer (Table~\ref{tab:ev-openloop}), because SGLang tokenizes requests one at a time in that process. The wait in the table is the time a turn spends behind its session's previous request. At 8 turns/s it is 116\,ms with the reference on one core and 0.18\,ms with \sys. At 12 turns/s every configuration is beyond what this setup sustains, since each session receives a turn every 500\,ms on average while a turn takes 451 to 512\,ms end to end at the median, so the medians there reflect queueing.

\begin{table}[htbp]
\centering
\small
\begin{tabular}{lrlll}
\toprule
Tokenizer & Cores & 4 turns/s & 8 turns/s & 12 turns/s \\
\midrule
Reference & 1 & 146 / 839 / 0.16 & 293 / 1,614 / 116 & 3,385 / 16,492 / 3,180 \\
Reference & 8 & 147 / 849 / 0.17 & 296 / 1,632 / 114 & 3,284 / 15,845 / 3,067 \\
\sys & 1 & 59.4 / 517 / 0.18 & 85.8 / 1,041 / 0.18 & 1,496 / 9,842 / 1,386 \\
\sys & 8 & 59.3 / 516 / 0.17 & 87.0 / 1,050 / 0.20 & 1,483 / 9,783 / 1,379 \\
fastokens & 1 & 62.7 / 534 / 0.17 & 89.0 / 1,039 / 0.20 & 1,574 / 10,577 / 1,454 \\
\bottomrule
\end{tabular}
\caption{Median TTFT / P99 TTFT / median wait behind the session's previous request, in milliseconds from scheduled arrival, when six sessions replay 250K-character transcripts open-loop in SGLang. The cores column is the budget of SGLang's HTTP and tokenizer process.}
\label{tab:ev-openloop}
\end{table}

With the pool inside SGLang, the main paper's end-to-end experiment sends 400-character appends to twelve warm sessions of 100K characters as Poisson arrivals at 4, 8 or 12 turns/s for 60\,s. Every 8\,s a cluster of four stateless documents of one million characters arrives, which SGLang tokenizes and then declines for exceeding the model's context. Table~\ref{tab:ev-e2e} pools two server processes per configuration. Without the documents, at 16 and 24 turns/s, the pool of four workers gives 29.0 / 73.9 and 34.5 / 105\,ms at the median and P99, against 27.8 / 74.2 and 32.9 / 103\,ms for \sys serial and 73.4 / 199 and 124 / 995\,ms for the reference. A control in which SGLang serves 200K-character documents instead of declining them gives warm-turn P99 values of 5,358, 5,214, and 5,216\,ms for the reference, \sys serial, and the pool, so GPU prefill dominates that setting.

\begin{table}[htbp]
\centering
\small
\setlength{\tabcolsep}{3.5pt}
\begin{tabular}{rllllllll}
\toprule
& \multicolumn{2}{c}{Reference} & \multicolumn{2}{c}{\sys serial} & \multicolumn{2}{c}{Pool of 4} & \multicolumn{2}{c}{Pool of 8} \\
Rate & Turns & Doc. & Turns & Doc. & Turns & Doc. & Turns & Doc. \\
\midrule
4 & 52.0 / 961 / 1,297 & 995 & 25.6 / 49.1 / 71.7 & 128 & 26.7 / 46.6 / 62.9 & 32.4 & 26.8 / 47.4 / 64.3 & 31.5 \\
8 & 58.7 / 1,242 / 1,441 & 1,013 & 26.0 / 53.8 / 166 & 125 & 26.5 / 48.7 / 62.7 & 33.2 & 26.6 / 49.0 / 62.6 & 32.9 \\
12 & 71.5 / 1,324 / 1,561 & 976 & 26.4 / 53.7 / 138 & 125 & 27.1 / 51.0 / 65.1 & 35.0 & 27.0 / 51.1 / 67.1 & 32.5 \\
\bottomrule
\end{tabular}
\caption{SGLang with the model in the loop and clusters of million-character documents that are tokenized and then declined. The turn columns give warm-turn TTFT from scheduled arrival as P50 / P95 / P99, and the document columns give the median time until a document is answered. Rates are warm turns per second.}
\label{tab:ev-e2e}
\end{table}

When every request carries a session key and no long stateless request arrives, the pool has little to do. In a separate set of paced runs with eight sessions, the four-worker pool gives a median TTFT of 197.5\,ms, against 202.0\,ms for \sys serial and 227.1\,ms for the reference. Measured from each turn's schedule, its median stays within 6\,ms of \sys serial in every open-loop run, with 57.2 and 57.9\,ms against 56.0 and 56.3\,ms for twelve 70K-character sessions at 15 turns/s, 61.6 and 51.2\,ms against 59.7 and 49.5\,ms for six 250K-character sessions at 4 turns/s, and 91.9 and 90.9\,ms against 89.4 and 85.4\,ms at 8 turns/s. The pool's cost shows in the adapter's median encode time, 2.12\,ms against 0.74\,ms, and in SGLang's median tokenize window, 2.76\,ms against 1.36\,ms. The two pool servers served 1,978 and 1,872 keyed requests and 6 and 6 requests by content lookup, all through the pool, with no fallback.

A further run of the repeated-document configuration of Section~4.6 in the main paper doubles the warm-turn rate to sixteen per second over a 30-second window, which keeps the twelve sessions and the document inside the model's cache. The P99 TTFT of warm turns is 618\,ms with the reference and 94.5 and 98.3\,ms with \sys serial and \sys pool, 6.3--6.5$\times$ apart against 7.4--7.5$\times$ at eight turns per second, and the medians are 81.2, 31.7, and 33.1\,ms. The reference's tail is still the time its event loop spends tokenizing documents, with a P99 of 456\,ms for the wait before the handler receives a turn that follows a cluster. \sys's tail grows because more turns queue behind each document's prefill, so the P99 of a turn's own forward pass after a cluster rises to 47--57\,ms from 44--48\,ms at the lower rate while GPU utilization rises from about 30\% to 40--54\%.

\subsection{fastokens}
\label{sec:ev-fastokens}

fastokens 0.1.2 is the optional tokenizer backend that SGLang v0.5.19 pins, and it reuses work from earlier calls in the same process. To locate its divergence, we replayed the 279 requests of the sequential TTFT runs of the main paper in serving order through one fastokens process and compared every request with the reference. 78 requests differ, 46 by one extra token and 32 by two. The differences lie in five of the nine sessions, one session per context size and server run, with 30, 16, 15, 9 and 8 differing requests. Once a session differs, every later request of that session differs at the same position.

Every difference has the same shape. Where the reference keeps a newline, a run of spaces, and a second newline together as one token, fastokens emits a lone newline, ID 198, followed by the spaces and the second newline as one token. The runs in the five sessions are 16, 3, 6, 4 and 5 spaces long. Table~\ref{tab:ev-ftk} shows the first difference in one session. Turn 15 of that session has 37,262 characters and ends with \code{prev\_host\_state}, a newline and six spaces. Turn 16 appends 401 characters that begin with a newline and fourteen spaces. The reference encodes turn 16 as 9,097 tokens and fastokens as 9,098, and the first difference is at token 9,021, the reference's whitespace token that spans the boundary between the two turns. Replaying turns 0 to 21 of this session in order in a fresh process reproduces it, with turns 0 to 15 equal and turns 16 to 21 different. Encoded alone in fresh processes, the same turn-16 text matches the reference in all 20 encodes, five each with two, four, eight and sixteen cores.

\begin{table}[htbp]
\centering
\small
\begin{tabular}{rll}
\toprule
Token & Reference & fastokens \\
\midrule
9,020 & \code{\_state} & \code{\_state} \\
9,021 & newline, 6 spaces, newline (35117) & newline (198) \\
9,022 & 13 spaces (1060) & 6 spaces, newline (6926) \\
9,023 & space, \code{for} (369) & 13 spaces (1060) \\
9,024 & space, \code{method} (1714) & space, \code{for} (369) \\
\bottomrule
\end{tabular}
\caption{The first difference in the fastokens replay, at turn 16 of one session, with token IDs in parentheses. Turn 15 ended after the first newline and six spaces, so the reference's token 9,021 spans the boundary between the two turns.}
\label{tab:ev-ftk}
\end{table}

The same effect shows in serving. At the end of each run, the benchmark compares the prompt-token count that SGLang reports for the final request with a local reference count. With raw-text requests on the Zen 5 server, the fastokens count differs in 15 of 24 runs, 5 of 9 sequential runs and 10 of 15 runs under core budgets of two, four and eight cores. On the desktop it differs in 3 of 4 runs, and through the chat interface it matches in all 6. The reference tokenizer and \sys match in every run. A separate check that encodes each of 50 prompts once found all 50 counts equal and identical greedy outputs.

\subsection{Fallback localization}
A fast encoder that meets a character it cannot handle sends the whole call to the reference encoder, so one exceptional character slows the entire request (main paper, \S2.2, \S3.3.3 and \S4.3). Table~\ref{tab:localization} separates the effect of chunking on this fallback from parallelism, on twenty English documents of one million characters, four of which trigger the guard. On those four, encoding sixteen chunks one after another on one core lowers the median from 276\,ms to 45\,ms, a sixfold gain before any parallelism, and sixteen cores lower it further to 17\,ms. On clean documents, sequential chunking changes little and the gain is from the cores.

\begin{table}[h]
\centering
\begin{tabular}{lrr}
\toprule
Execution & Clean & Guarded \\
\midrule
Whole input, one core & 26.5 & 276 \\
16 chunks, one core & 24.7 & 45.0 \\
16 chunks, 16 cores & 5.1 & 17.4 \\
\bottomrule
\end{tabular}
\caption{Median encode time in milliseconds for English documents of one million characters, 16 clean and 4 that trigger the fast encoder's guard. Chunking on one core confines the fallback to one chunk.}
\label{tab:localization}
\end{table}

\subsection{Rust prototype}
\label{sec:ev-rust}

The Rust prototype plans cuts with the same rules and encodes the chunks on native threads, one thread per chunk up to 64, keeping the IDs in native buffers. It ran on 66 cores of the rented 128-core Zen 3 node of the capacity study, with no other experiment on the node and 0.5 to 1.2 cores of CPU use outside the container, and its timings come from the prototype alone, not from the SGLang adapter or its pool. The prototype is a research build that has not been through the checks applied to the main encoder. Each configuration encodes 20 documents at 1M and 4M characters and 10 at 10M, and the table reports the median of each document's first encode. The prototype returns the IDs either as a Python list or as a packed byte buffer. Table~\ref{tab:ev-rust} gives Qwen3 medians for the list return by chunk count, and the best speedup over serial encoding for both returns.

\begin{table}[htbp]
\centering
\small
\setlength{\tabcolsep}{4pt}
\begin{tabular}{llrrrrrrrrrr}
\toprule
Text & Size & Serial & $k=4$ & 8 & 16 & 32 & 64 & Best, list & Best, bytes & Guarded & Fallback \\
\midrule
Web & 1M & 50.5 & 15.6 & 12.9 & 10.3 & 7.8 & 8.1 & 6.47 (32) & 8.55 (32) & 4 of 20 & 1 \\
Web & 4M & 268 & 93.7 & 65.5 & 43.5 & 35.7 & 36.2 & 7.50 (32) & 12.93 (64) & 7 of 20 & 2 \\
Web & 10M & 4,514 & 1,405 & 678 & 333 & 194 & 138 & 32.73 (64) & 47.84 (64) & 7 of 10 & 4 \\
Code & 1M & 41.8 & 17.9 & 11.7 & 9.5 & 8.0 & 7.1 & 5.88 (64) & 8.50 (64) & 0 of 20 & 0 \\
Code & 4M & 255 & 86.7 & 57.8 & 49.1 & 39.7 & 38.3 & 6.66 (64) & 14.33 (64) & 1 of 20 & 2 \\
Code & 10M & 552 & 272 & 182 & 137 & 111 & 107 & 5.17 (64) & 11.31 (64) & 2 of 10 & 2 \\
\bottomrule
\end{tabular}
\caption{Median time in milliseconds of the Rust prototype with the Qwen3 tokenizer on 66 cores of the rented 128-core Zen 3 node, by chunk count $k$ with the list return, and the best speedup over serial encoding, to two decimals, with the chunk count that gives it. Guarded counts the documents whose serial encode falls back to the reference for the whole document. Fallback is the largest number of the 64 chunks of one document that fall back.}
\label{tab:ev-rust}
\end{table}

At ten million characters of web text, seven of the ten documents contain a character that sends the serial encode of the whole document to the reference encoder, so the serial median is 4,514\,ms, while the three clean documents take 794 to 814\,ms. With 64 chunks only the chunks that contain such a character fall back, 14 of the 640 in all, and the median falls to 138\,ms. The three clean documents go from 794--814\,ms to 91--102\,ms. Most of the 33-fold gain at this size therefore comes from confining the fallback to a few chunks, and the rest from parallel encoding.

With one chunk the prototype takes 0.89 to 1.06 times the serial time in 15 of the 16 groups of two tokenizers, two kinds of text and four sizes, and 1.27 times in one, Qwen3 on code at 10M characters. At 64 chunks of web text, planning the cuts takes 0.94, 3.43 and 8.47\,ms at 1M, 4M and 10M characters, and returning the IDs to Python as a list takes 3.55, 20.8 and 67.1\,ms. Both steps are serial, so they take a growing share of the time as threads are added, and the byte-buffer return removes most of the list cost.

Fitting the cost model of Section~\ref{sec:ev-designs}, with the fixed cost $a$ folded into $T_1$ so that $T(k)\approx T_1/k+ck$, at 64 chunks on 256K-character inputs gives a cost per chunk $c$ of 0.028 to 0.035\,ms for the prototype, against about 0.16\,ms for the process pool on a shared cluster node of the same generation, whose encodes of 256K characters take 2.66\,ms with 16 chunks and 10.04\,ms with 64. Going from 16 to 64 chunks at that size changes the prototype's speed by a factor of 0.78 to 1.21, while the pool slows to 0.26 to 0.35 of its speed. The prototype's best chunk count is 32 or 64 in 12 of the 16 groups.

The runs on this node compared 12,320 ID sequences with the reference, the first and second encode of every document in 352 configurations, and none differed. These checks are in addition to the 23,912 ID comparisons and 144,000 adversarial comparisons of planned cuts and IDs from the prototype's runs on the Zen 5 server. The same matrix had run earlier on a shared cluster node with two 64-core EPYC 7713 processors, with the same wheel and another 12,320 identity checks without a difference; there the ten-million-character web documents took 3,967\,ms serially and 124\,ms with 64 chunks, 32.0$\times$, and across the 352 configurations its times were about 13\% lower at the median than on the rented node.

\section{Scope of the exactness checks}
\label{sec:validation}
Table~3 of the main paper totals the exactness checks, and Table~\ref{tab:tok-exact} splits them by tokenizer, and Table~\ref{tab:platform-exact} lists the platform and runtime repeats, with its caption stating what the sweep and serial-encoder rows count. Every counted output passed through \sys's cuts or its fast encoder. The first row, whole encodes by the fast encoder on the full web corpus with one comparison per document and tokenizer, is the only row that reaches all eleven admitted tokenizers. The same experiments also compared whole encodes by fastokens, and by the reference itself, with the reference. These matched as well, but they test other encoders, so they are not counted.

The rows count different things, documents in the web-corpus check, constructed strings in the adversarial battery, assembled outputs in the chunking sweeps, turns in the turn-by-turn replay, and served requests in the SGLang checks, so each is reported on its own row. The 2,195 SGLang shadow re-encodes combine 319 synchronous and 1,876 asynchronous checks. The 2,268 asynchronous shadow re-encodes of the runs with served long documents (Section~4.6 of the main paper) have a row of their own. Of these, 212 come from twelve runs on a GH200 node, where the fast encoder is an ARM build and not one of the builds behind the first row of Table~3. Before those runs it matched the reference on 800,000 whole documents under six tokenizers, on the 220,000 strings of the adversarial battery, and on the turn-by-turn replay, and the shadow re-encodes and prompt-token counts of the runs themselves are counted above.